\documentclass{article} 
\usepackage{iclr2027_conference,times}

\usepackage[utf8]{inputenc}
\usepackage[T1]{fontenc}

\usepackage{hyperref}
\usepackage{url}
\usepackage{amsmath}
\usepackage{amssymb}
\usepackage{amsthm}
\usepackage{booktabs}
\usepackage{graphicx}
\usepackage{adjustbox}
\usepackage{wrapfig}
\usepackage{float}
\usepackage{multirow}
\usepackage{multicol}
\usepackage{color}
\usepackage{colortbl}
\usepackage{tabularx}
\usepackage{array}
\usepackage{caption}
\usepackage{subcaption}
\usepackage{algorithm,algpseudocode}
\usepackage{xspace}
\usepackage{cleveref}
\usepackage{pifont}
\usepackage[inline]{enumitem}

\usepackage{todonotes}

\usepackage{siunitx}
\usepackage{amsfonts}
\usepackage{mathtools}
\usepackage{tikz}
\usetikzlibrary{calc,arrows.meta,positioning,backgrounds}
\usetikzlibrary{calc,arrows.meta,positioning,backgrounds,shadings}

\usepackage{makecell,ragged2e,rotating}
\usepackage[dvipsnames]{xcolor}

\def\eqref#1{equation~\ref{#1}}

\def\1{\bm{1}}

\DeclareMathAlphabet{\mathsfit}{\encodingdefault}{\sfdefault}{m}{sl}
\SetMathAlphabet{\mathsfit}{bold}{\encodingdefault}{\sfdefault}{bx}{n}

\newcommand{\methodname}{SFPRM\xspace}
\newcommand{\methodlong}{Spherical Fixed-Point Reasoning Models\xspace}

\newcommand{\reals}{\mathbb{R}}
\newcommand{\sph}{\mathbb{S}}

\DeclareMathOperator{\Norm}{Norm}
\DeclareMathOperator{\Slerp}{Slerp}
\DeclareMathOperator{\gSlerp}{gSlerp}
\DeclareMathOperator{\nlerp}{nlerp}
\DeclareMathOperator{\Attn}{Attn}
\DeclareMathOperator{\MLP}{MLP}
\DeclareMathOperator{\Exp}{Exp}
\DeclareMathOperator{\Log}{Log}
\DeclareMathOperator{\arccosop}{arccos}
\DeclareMathOperator{\Real}{Re}

\newcommand{\inner}[2]{\langle #1,\, #2\rangle}
\newcommand{\norm}[1]{\lVert #1\rVert}

\theoremstyle{plain}
\newtheorem{proposition}{Proposition}
\newtheorem{lemma}{Lemma}
\theoremstyle{definition}

\theoremstyle{remark}
\newtheorem{remark}{Remark}
\theoremstyle{plain}
\newtheorem{theorem}{Theorem}

\definecolor{fprmcol}{RGB}{198,86,20}    
\definecolor{slfprcol}{RGB}{27,94,196}   
\newcommand{\dfp}[1]{\textcolor{fprmcol}{#1}}
\newcommand{\dsl}[1]{\textcolor{slfprcol}{#1}}

\definecolor{tablehead}{HTML}{26384A}
\definecolor{lightrow}{HTML}{F3F6F8}
\definecolor{fprmorange}{RGB}{198,86,20}
\definecolor{slfprblue}{RGB}{27,94,196}
\definecolor{yesgreen}{RGB}{0,125,80}
\definecolor{approxorange}{RGB}{190,100,0}

\newcommand{\keybox}[1]{%
  \par\noindent
  {\setlength{\fboxsep}{4pt}\setlength{\fboxrule}{0.7pt}%
   \fcolorbox{black!30}{lightrow}{\begin{minipage}{\dimexpr\linewidth-2\fboxsep-2\fboxrule\relax}#1\end{minipage}}}%
  \par
}

\newcommand{\filmrow}[4][]{%
  \rotatebox[origin=c]{90}{\footnotesize puzzle $#2$} &
  \raisebox{-.5\height}{\includegraphics[width=0.465\linewidth,#1]{#3/filmstrip_p#2.pdf}} &
  \raisebox{-.5\height}{\includegraphics[width=0.465\linewidth,#1]{#4/filmstrip_p#2.pdf}} \\[50pt]}

\newcolumntype{Y}{>{\RaggedRight\arraybackslash}X}

\newcommand{\fpyes}{\textcolor{yesgreen}{\textbf{Yes}}}
\newcommand{\fpno}{\textcolor{red!70!black}{\textbf{No}}}
\newcommand{\fpapprox}{%
  \textcolor{approxorange}{\textbf{Appr.}}%
}

\newcommand{\sudokuFdir}{figures/sudoku/frames/fprm_hf_sudoku}
\newcommand{\sudokuSdir}{figures/sudoku/frames/slfpr_ep15k_sudoku}
\newcommand{\mazeFdir}{figures/maze/maze_official_fprm_ckpt/filmstrip/fprm_hf_maze}
\newcommand{\mazeSdir}{figures/maze/maze_official_fprm_ckpt/filmstrip/slfpr_lamvdw_maze}
\newcommand{\arcFdir}{figures/miniarc/miniarc_filmstrip/fprm}
\newcommand{\arcSdir}{figures/miniarc/miniarc_filmstrip/sfprm}
\newcommand{\arceFdir}{figures/miniarc/filmstrips_more/fprm}
\newcommand{\arceSdir}{figures/miniarc/filmstrips_more/sfprm}
\newcommand{\sudokutop}[1]{\filmrow[trim={0 50pt 0 0},clip]{#1}{\sudokuFdir}{\sudokuSdir}}
\newcommand{\sudokumid}[1]{\filmrow[trim={0 50pt 0 9pt},clip]{#1}{\sudokuFdir}{\sudokuSdir}}
\newcommand{\sudokubot}[1]{\filmrow[trim={0 0 0 9pt},clip]{#1}{\sudokuFdir}{\sudokuSdir}}
\newcommand{\mazetop}[1]{\filmrow[trim={0 50pt 0 0},clip]{#1}{\mazeFdir}{\mazeSdir}}
\newcommand{\mazemid}[1]{\filmrow[trim={0 50pt 0 9pt},clip]{#1}{\mazeFdir}{\mazeSdir}}
\newcommand{\mazebot}[1]{\filmrow[trim={0 0 0 9pt},clip]{#1}{\mazeFdir}{\mazeSdir}}
\newcommand{\arctop}[1]{\filmrow[trim={0 50pt 0 0},clip]{#1}{\arcFdir}{\arcSdir}}
\newcommand{\arcmid}[1]{\filmrow[trim={0 50pt 0 9pt},clip]{#1}{\arcFdir}{\arcSdir}}
\newcommand{\arcbot}[1]{\filmrow[trim={0 0 0 9pt},clip]{#1}{\arcFdir}{\arcSdir}}
\newcommand{\arcetop}[1]{\filmrow[trim={0 50pt 0 0},clip]{#1}{\arceFdir}{\arceSdir}}
\newcommand{\arcemid}[1]{\filmrow[trim={0 50pt 0 9pt},clip]{#1}{\arceFdir}{\arceSdir}}
\newcommand{\arcebot}[1]{\filmrow[trim={0 0 0 9pt},clip]{#1}{\arceFdir}{\arceSdir}}
\providecommand{\cmark}{\ding{51}}
\providecommand{\xmark}{\ding{55}}

\algblockdefx[NoGrad]{NoGrad}{EndNoGrad}{\textbf{with} \textsc{no\_grad}\textbf{:}}{}

\providecolor{fprmorange}{RGB}{213,94,0}
\providecolor{slfprblue}{RGB}{0,114,178}
\providecommand{\dfp}[1]{\textcolor{fprmorange}{#1}}
\providecommand{\dsl}[1]{\textcolor{slfprblue}{#1}}
\newcommand{\ca}{\textcolor{red}{\mathbf{\alpha}}}
\newcommand{\cb}{\textcolor{ForestGreen}{\mathbf{\beta}}}
\newcommand{\cc}{\textcolor{blue}{\mathbf{\gamma}}}
\newcommand{\cd}{\textcolor{magenta}{\mathbf{\delta}}}

\title{Spherical Fixed-Point Reasoners: \\Adaptive Loop Halting on the Hypersphere}

\author{%
Arber Zela$^{\ca}$ \quad
Artavazd Maranjyan$^{\ca}$ \quad
Pol Puigdemont$^{\ca}$ \quad
David Berghaus$^{\cb,\cc}$ \quad \\[2pt]
\textbf{Valentyn Boreiko$^{\cd}$ \quad Volkan Cevher$^{\ca}$} \\[3pt]
$^{\ca}$EPFL \quad
$^{\cb}$Lamarr Institute \quad
$^{\cc}$Fraunhofer IAIS \quad
$^{\cd}$Amazon AGI
}

\iclrfinalcopy 
\begin{document}

\maketitle

\begin{abstract}
Looped (weight-tied) transformers enable adaptive computation by allocating more iterations to harder inputs. Fixed-Point Reasoners (FPRM) achieve this with a damped fixed-point solver, but operate in an unbounded Euclidean residual space with coupled rescaling and relative-norm halting. We introduce \emph{Spherical Fixed-Point Reasoners} (SFPRM), which performs the same computation on the unit hypersphere of normalized transformers (nGPT). Input injection, damping, and halting become geodesic operations, yielding boundedness by construction and eliminating coefficient coupling. We prove fixed-point existence, local convergence under an explicit step-size condition, and global convergence under strong injection. 
On Sudoku-Extreme, Maze-Hard and Mini-ARC, SFPRM achieves $96.6\%$, $87.9\%$ and $59.1\%$ test accuracy, respectively, outperforming FPRM and other baselines while allocating test-time compute more effectively. Its converged geometry is also interpretable: token–input angles capture puzzle structure and predict correctness.
\end{abstract}

\section{Introduction}
\label{sec:intro}

Scaling test-time computation has become the main lever for improving the reasoning ability of sequence models, most visibly through chain-of-thought traces \citep{wei2022chain,wang2023self,cobbe2021training,yao2023tree,snell2025scaling}. Looped transformers offer a latent alternative \citep{dehghani2019universal,Bai2019DEQ,geiping2025recurrent}: a single weight-tied block is applied repeatedly, so depth becomes an iteration count chosen at inference time, and compute can be allocated per input according to its difficulty. On algorithmic puzzles such as Sudoku, mazes, and ARC, this recipe lets small recursive reasoners compete with far larger models \citep{Wang2025HRM,JolicoeurMartineau2025TRM,Movahedi2026FPRM}.

A looped transformer raises two questions: \emph{when should it stop}, and \emph{will it stay well-behaved as it runs}? The fixed-point view \citep{Bai2019DEQ}, answers the first by treating the loop as a solver for $\mathbf z^\star=F_\theta(\mathbf z^\star;\mathbf x)$ and halting once the iterate stops moving. Fixed-Point Reasoners (FPRM) \citep{Movahedi2026FPRM} turn this into a practical recipe: a damped iteration whose step size decays on a patience rule, a relative-norm residual as the halting test, and a boundedness theorem that keeps the recurrent state from diverging. The second question is where the Euclidean formulation shows its cost. Because the residual stream is free to grow, a boundedness theorem must couple four residual-rescaling coefficients, and the halting residual needs a normalization factor to be scale-invariant. 

We introduce \emph{\methodlong} (\methodname), which run FPRM's damped solver on nGPT's \citep{Loshchilov2024nGPT} unit sphere (\Cref{fig:arch}). 
The normalized Transformer (nGPT) obtains stable dynamics by changing the \emph{geometry} rather than adding control: every hidden state lives on the unit hypersphere $\sph^{d-1}$ and each sub-layer is a learned angular move followed by renormalization, so the divergence FPRM must guard against cannot occur.
Every control quantity of the loop becomes an angle: the input is re-injected by moving a fraction $\lambda$ of the great-circle arc toward it, the damped update is a spherical linear interpolation ($\Slerp$) whose coefficient $\eta_i$ is an angular step size that FPRM's patience-based decay literally shrinks, and the halting residual is the worst-token angle between an iterate and its proposal. Our contributions are:
\begin{enumerate}[label=(\roman*), leftmargin=2.5em]
  \item \textbf{Simplicity}: because the state cannot leave the sphere, boundedness is structural, and FPRM's coupled residual-rescaling coefficients, its boundedness theorem, and its relative-norm residual disappear rather than being re-proved;
  \item \textbf{guarantees}: a fixed point exists from continuity and input injection alone, with no contraction or Lipschitz assumption; angular damping stabilizes any fixed point whose tangent Jacobian has real parts below one, with an explicit step-size window; and sufficiently strong injection makes it a unique global attractor;
  \item \textbf{performance}: on Sudoku-Extreme and Maze-Hard, \methodname halts up to $2.7\times$ earlier than FPRM at equal or higher accuracy; and
  \item \textbf{interpretability}: the angle between each token state and its input anchor separates token classes during the recursion, converges to a readable encoding of the task's structure, and, on Sudoku, predicts whether the answer is correct.
\end{enumerate}

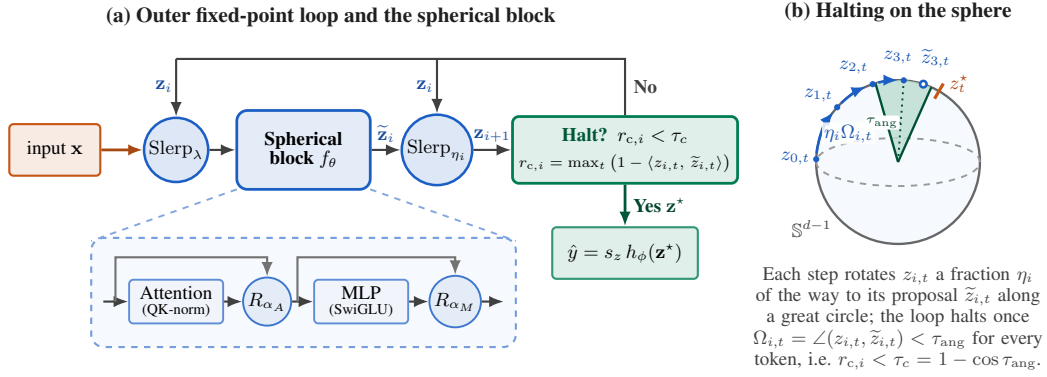
\begin{figure}[t]
\centering
\resizebox{0.99\linewidth}{!}{%
\begin{tikzpicture}[
  font=\normalsize,
  >={Latex[length=2.4mm, width=1.9mm]},
  box/.style={draw=black!60, line width=1.0pt, rounded corners=2pt, fill=black!3,
              minimum height=0.85cm,
              minimum width=2.4cm, inner sep=2pt, align=center, font=\small},
  minibox/.style={draw=slfprblue!70, line width=0.9pt, rounded corners=2pt, fill=slfprblue!6,
              minimum height=0.72cm, minimum width=1.55cm, inner sep=1.5pt,
              align=center, font=\footnotesize},
  bigblock/.style={draw=slfprblue, line width=1.3pt, rounded corners=5pt, fill=slfprblue!12,
                   minimum width=2.2cm, minimum height=1.3cm, align=center, font=\small},
  opnode/.style={circle, draw=slfprblue, line width=1.2pt, fill=slfprblue!15,
                 inner sep=1.3pt, font=\small},
  miniop/.style={circle, draw=slfprblue!85, line width=0.9pt, fill=slfprblue!12,
                 inner sep=1.0pt, font=\footnotesize},
  haltnode/.style={draw=yesgreen, line width=1.3pt, rounded corners=3pt, fill=yesgreen!12,
                   inner sep=3pt, align=center, font=\small,
                   minimum width=1.8cm, minimum height=1.1cm},
  arr/.style={->, line width=1.0pt, black!75},
  state/.style={font=\small, text=slfprblue!80!black},
  note/.style={font=\small, text=black!75, align=center},
  ptitle/.style={font=\normalsize\bfseries, text=black!88, align=center},
]

\node[ptitle] at (2.55,6.70) {(a) Outer fixed-point loop and the spherical block};

\node[opnode]   (inject) at (0,4.50) {$\Slerp_\lambda$};
\node[bigblock] (block)  at (2.10,4.50) {\textbf{Spherical}\\[-1pt]\textbf{block} $f_\theta$};
\node[opnode]   (gs)     at (4.30,4.50) {$\Slerp_{\eta_i}$};
\node[haltnode] (halt)   at (7.40,4.50)
      {\textcolor{yesgreen!70!black}{\textbf{Halt?}}\ \ $r_{\mathrm c,i}<\tau_c$\\[2pt]
       {\scriptsize $r_{\mathrm c,i}=\max_t\big(1-\inner{z_{i,t}}{\widetilde z_{i,t}}\big)$}};

\node[box, fill=fprmorange!12, draw=fprmorange, minimum width=1.5cm] (input) at (-2.0,4.50)
      {\footnotesize input $\mathbf x$};
\draw[arr, fprmorange!85!black, line width=1.3pt] (input) -- (inject);

\draw[arr] (inject) -- (block);
\draw[arr] (block)  -- node[above, state]{$\widetilde{\mathbf z}_i$} (gs);
\draw[arr] (gs)     -- node[above, state]{$\mathbf z_{i+1}$} (halt);

\draw[line width=1.0pt, black!75] (halt.north) -- (7.40,6.0) -- (0,6.0);
\draw[arr] (0,6.0) -- (inject);
\draw[arr] (4.30,6.0) -- (gs);
\node[state] at (-0.20,5.55) {$\mathbf z_i$};
\node[state] at (4.10,5.55) {$\mathbf z_i$};
\node[font=\small, text=black!75, anchor=west] at (7.42,5.55) {\textbf{No}};

\node[box, fill=yesgreen!12, draw=yesgreen!80] (read) at (7.40,2.80)
      {\footnotesize$\hat y=s_z\,h_\phi(\mathbf z^\star)$};
\draw[arr, yesgreen!70!black, line width=1.3pt] (halt) --
      node[right, font=\small, text=yesgreen!55!black, pos=0.5]{\textbf{Yes}\ $\mathbf z^\star$} (read);

\begin{scope}[shift={(2.10,2.15)}]
  \node[draw=slfprblue!60, dashed, line width=1.0pt, rounded corners=6pt, fill=slfprblue!4,
        minimum width=7.0cm, minimum height=1.8cm] at (0,0) {};

  \node[minibox] (attn) at (-2.10,-0.15) {Attention\\[-3pt]{\scriptsize (QK-norm)}};
  \node[miniop]  (ra)   at (-0.60,-0.15) {$R_{\alpha_A}$};
  \node[minibox] (mlp)  at ( 0.95,-0.15) {MLP\\[-3pt]{\scriptsize (SwiGLU)}};
  \node[miniop]  (rm)   at ( 2.50,-0.15) {$R_{\alpha_M}$};

  \draw[arr] (-3.30,-0.15) -- node[above, state, font=\footnotesize, pos=0.2]{} (attn);
  \draw[arr] (attn) -- (ra);
  \draw[arr] (ra)   -- (mlp);
  \draw[arr] (mlp)  -- (rm);
  \draw[arr] (rm)   -- (3.3,-0.15);

  \draw[arr, black!55] (-3.10,-0.15) -- (-3.10,0.58) -- (-0.60,0.58) -- (ra.north);
  \draw[arr, black!55] (-0.10,-0.15) -- (-0.10,0.58) -- ( 2.50,0.58) -- (rm.north);
\end{scope}

\draw[dashed, line width=0.9pt, slfprblue!55] (1.00,3.85) -- (-1.40,3.0);
\draw[dashed, line width=0.9pt, slfprblue!55] (3.20,3.85) -- ( 5.60,3.0);

\begin{scope}[shift={(11.90,4.30)}, scale=0.9]
  \node[ptitle] at (0,2.75) {(b) Halting on the sphere};

  \fill[slfprblue!4] (0,0) circle (1.5);
  \draw[black!60, line width=1.1pt] (0,0) circle (1.5);
  \draw[dashed, line width=0.9pt, black!40] (0,0) ellipse (1.5 and 0.44);
  \node[note, font=\footnotesize] at (-1.58,-1.22) {$\sph^{d-1}$};

  \fill[yesgreen!25, opacity=0.85]
       (0,0) -- (66:1.5) arc[start angle=66, end angle=106, radius=1.5] -- cycle;
  \draw[yesgreen!65!black, line width=1.1pt] (0,0) -- (66:1.5);
  \draw[yesgreen!65!black, line width=1.1pt] (0,0) -- (106:1.5);
  \draw[yesgreen!60!black, dotted, line width=0.9pt] (0,0) -- (86:1.5);
  \draw[yesgreen!65!black, line width=0.9pt, -] (86:0.72) arc[start angle=86, end angle=106, radius=0.72];
  \node[font=\scriptsize, text=yesgreen!55!black, fill=white, inner sep=0.8pt,
        rounded corners=1pt] at (112:0.78) {$\tau_{\mathrm{ang}}$};

  \foreach \a in {178, 138, 108, 86}{ \fill[slfprblue] (\a:1.5) circle (1.6pt); }
  \draw[arr, slfprblue, line width=1.2pt] (174:1.5) arc[start angle=174, end angle=142, radius=1.5];
  \draw[arr, slfprblue, line width=1.2pt] (134:1.5) arc[start angle=134, end angle=113, radius=1.5];
  \draw[arr, slfprblue, line width=1.2pt] (104:1.5) arc[start angle=104, end angle= 90, radius=1.5];
  \node[font=\footnotesize, text=slfprblue!90!black] at (178:1.88) {$z_{0,t}$};
  \node[font=\footnotesize, text=slfprblue!90!black] at (140:1.90) {$z_{1,t}$};
  \node[font=\footnotesize, text=slfprblue!90!black] at (115:1.92) {$z_{2,t}$};
  \node[font=\footnotesize, text=slfprblue!90!black] at ( 90:1.92) {$z_{3,t}$};
  \node[font=\footnotesize, text=slfprblue!90!black, fill=white, fill opacity=0.85,
        text opacity=1, inner sep=1pt, rounded corners=1pt] at (150:1.05) {$\eta_i\Omega_{i,t}$};

  \draw[slfprblue, line width=1.2pt, fill=white] (72:1.5) circle (1.7pt);
  \node[font=\footnotesize, text=slfprblue!90!black] at (70:2.02) {$\widetilde z_{3,t}$};
  \draw[fprmorange, line width=1.5pt] (60:1.34) -- (60:1.66);
  \node[font=\footnotesize, text=fprmorange!85!black] at (52:1.82) {$z_t^\star$};

  \node[note, anchor=north, text width=4.9cm] at (0,-1.78)
    {Each step rotates $z_{i,t}$ a fraction $\eta_i$ of the way to its proposal
     $\widetilde z_{i,t}$ along a great circle; the loop halts once
     $\Omega_{i,t}=\angle(z_{i,t},\widetilde z_{i,t})<\tau_{\mathrm{ang}}$ for every token,
     i.e.\ $r_{\mathrm c,i}<\tau_c=1-\cos\tau_{\mathrm{ang}}$.};
\end{scope}
\end{tikzpicture}}
\caption{\textbf{\methodname overview.} \emph{(a)} Each iteration interpolates the current state
toward the normalized input ($\Slerp_\lambda$), applies the weight-tied spherical nGPT block to
obtain the proposal $\widetilde{\mathbf z}_i$, and rotates a fraction $\eta_i$ of the way toward it
($\Slerp_{\eta_i}$); the loop halts once the worst-token chordal residual $r_{\mathrm c,i}$ falls below
$\tau_c$. The eigen-learning-rate retractions $R_\alpha$ inside the block keep every token state on
unit sphere $\sph^{d-1}$. \emph{(b)} One token's iterates rotate along a great circle by shrinking angles
$\eta_i\Omega_{i,t}$; the halting test compares each proposal with the current iterate, so it never
needs the fixed point $z_t^\star$.}
\label{fig:arch}
\end{figure}

\section{Related Work}
\label{sec:related}

\textbf{Looped and weight-tied transformers.}
Tying the weights of a transformer block \citep{vaswani2017attention} and applying it repeatedly turns depth into an iteration count at inference time \citep{dehghani2019universal,lan2020albert}. Looped models match much deeper non-looped ones on synthetic reasoning tasks and can simulate chain-of-thought in latent space \citep{saunshi2025latent}, learn in-context algorithms more sample-efficiently \citep{yang2024looped,gatmiry2024looped}, emulate programmable computation \citep{giannou2023programmable}, generalize in length when the loop matches the algorithm's recursion \citep{fan2025looped}, scale to billions of parameters \citep{geiping2025recurrent,prairie2026parcaescalinglawsstable}, and supply a scaling axis for masked diffusion models \citep{carballo2026recursive,miele2026fixedpoint}. \methodname is a looped transformer whose loop is run as a fixed-point solve on the sphere.

\textbf{Equilibrium models, adaptive computation, and halting.}
Adaptive Computation Time \citep{Graves2016ACT} and PonderNet \citep{banino2021pondernet} learn a halting distribution through an auxiliary loss, so halting is a learned head rather than a property of the dynamics. Recent methods such as HRM \citep{Wang2025HRM}, URM \citep{gao2025urm}, TRM \citep{JolicoeurMartineau2025TRM}, and PTRM \citep{sghaier2026ptrm} inherit this design. Deep equilibrium models \citep{Bai2019DEQ,bai2020mdeq,winston2020monotone,bai2021stabilizing} replace the unrolled iteration by the fixed point of a weight-tied layer and differentiate implicitly through it. Attractor models \citep{FeinAshley2026Attractor} and equilibrium reasoners \citep{Huang2026EqR} make the loop's attractor explicit and halt on a solver tolerance or a residual-based selection, and FPRM \citep{Movahedi2026FPRM} halts on a scale-invariant residual of the tied block itself, at the cost of coupled residual-rescaling coefficients that keep its Euclidean residual stream bounded; iterating past the point where the answer has settled is itself harmful \citep{bansal2022overthinking,schwarzschild2021algorithm}. \Cref{tab:unify} sorts these reasoners by whether a numerical fixed-point solve is performed and what governs halting; \methodname alone combines a genuine solve with an intrinsic, scale-free halting criterion on a curved state space.

\textbf{Normalized and hyperspherical representations.}
Constraining representations to a sphere is a long-standing stabilization device, for parameters \citep{salimans2016weightnorm,miyato2018spectral}, latent variables \citep{davidson2018hyperspherical}, contrastive representations \citep{wang2020alignment}, and flow language models that rotate tokens on $\sph^{d-1}$ \citep{deschenaux2026hyperspherical}. nGPT \citep{Loshchilov2024nGPT} carries the constraint through an entire transformer, with every hidden state and weight vector on $\sph^{d-1}$ and each sub-layer a normalized retraction toward a learned target. \methodname adopts that block, hence every optimizer and solver step becomes an angular update.

\begin{table*}[t]
\centering
\caption{Recent recurrent reasoners: HRM \citep{Wang2025HRM}, TRM
\citep{JolicoeurMartineau2025TRM}, URM \citep{gao2025urm}, Attractor \citep{FeinAshley2026Attractor},
EqR \citep{Huang2026EqR} and FPRM \citep{Movahedi2026FPRM}. ``Appr.'' means no numerical fixed-point solve is performed.}
\label{tab:unify}

\scriptsize
\setlength{\tabcolsep}{3.5pt}
\renewcommand{\arraystretch}{1.0}

\begin{tabularx}{\textwidth}{
  @{}
  >{\RaggedRight\arraybackslash}p{1.15cm}
  >{\RaggedRight\arraybackslash}p{3.4cm}
  >{\centering\arraybackslash}p{0.9cm}
  >{\hsize=1.14\hsize}Y
  >{\hsize=0.86\hsize}Y
  @{}
}
\toprule
\textbf{Method} &
\textbf{Recurrence} &
\textbf{Fix.\ pt.} &
\textbf{State update and input} &
\textbf{Halting / output} \\
\midrule

\textbf{HRM} &
Fast $L$-loop $+$ slow $H$-loop &
\fpapprox &
Euclidean $(z_L,z_H)$; input drives $L$ &
ACT head; read $z_H$ \\

\textbf{TRM} &
Shared net refines $z$, answer $y$ &
\fpno &
Euclidean; input re-read each step &
ACT head; read $y$ \\

\textbf{URM} &
UT; fixed inner $+$ ACT outer loop &
\fpno &
Euclidean; short conv in MLP &
ACT head, ${\le}16$ steps \\

\textbf{EqR} &
Tied loop; $B$ stochastic restarts &
\fpapprox &
Euclidean; random init, path noise &
ACT; lowest-residual run \\

\textbf{Attractor} &
Backbone $+$ equilibrium loop &
\fpyes &
Backbone $\tilde y_0$ seeds each step &
Solver tol.; read $\tilde y^\star$ \\

\textcolor{fprmorange}{\textbf{FPRM}} &
Tied Transformer loop &
\fpyes &
$(\alpha_1,\beta_1)$ mix; $(\alpha_2,\beta_2)$ injection &
Rel.-$\ell_\infty$ residual \\

\midrule

\textcolor{slfprblue}{\textbf{\methodname}} &
Tied nGPT spherical loop &
\fpyes &
Eigen-LR motion; $\lambda$ geodesic injection &
Chordal residual \\

\bottomrule
\end{tabularx}
\end{table*}

\section{\methodlong}
\label{sec:method}


\subsection{Preliminaries}
\label{sec:background}

\textbf{Notation.}
For a sequence of $T$ tokens with hidden dimension $d$, the recurrent state at loop iteration $i$ is $\mathbf z_i=(z_{i,1},\ldots,z_{i,T})$ with token states $z_{i,t}\in\reals^d$. We use $i$ for loop iterations, $t$ for token positions, and $\ell$ for sub-layers. $\sph^{d-1} := \{ z \in\reals^d : \norm{z} = 1  \}$ denotes the unit hypersphere in $\reals^d$. The embedded input is $\mathbf x=(x_1,\ldots,x_T)$. Tuples carry the product inner product  $\inner{\mathbf u}{\mathbf v}:=\sum_{t=1}^{T}\inner{u_t}{v_t}$ and norm $\norm{\mathbf u}^2:=\sum_{t=1}^{T}\norm{u_t}_2^2$. Lowercase $f_\theta$ always denotes the weight-tied network, uppercase $F_\theta(\cdot\,;\mathbf x)$ the complete input-dependent recurrent map, and $h_\phi$ the readout.

\textbf{Looped transformers and fixed-point reasoning.}
A looped Transformer applies a weight-tied network repeatedly to a state that is re-exposed to the same input $\mathbf x$ at every application: $\mathbf z_{i+1}=F_\theta(\mathbf z_i;\mathbf x)$, for $i=0,1,\ldots$, and predicts with a readout $h_\phi(\mathbf z_i)$. A convergent iteration reaches a fixed point $\mathbf z^\star=F_\theta(\mathbf z^\star;\mathbf x)$, the equilibrium view of \citet{Bai2019DEQ}. The prediction is determined by the fixed point, while the number of iterations needed to reach it is chosen per input. 
FPRM \citep{Movahedi2026FPRM} uses a pre-norm block of depth $D$ ($2D$ attention and MLP sub-layers). Because the residual stream is free to grow indefinitely, FPRM uses tied residual coefficients $(\alpha_1,\beta_1)$ and iteration-level coefficients $(\alpha_2,\beta_2)$:
\begin{equation}
  \mathbf z_i^\ell
  =
  \alpha_1 \mathbf z_i^{\ell-1}
  +
  \beta_1 f_{\theta^\ell}^{\ell}\!\left(\Norm_{\mathrm{pre}}(\mathbf z_i^{\ell-1})\right),
  \quad \ell=1,\ldots,2D,
  \qquad
  \mathbf z_{i+1}^{0}=\alpha_2 \mathbf z_i^{2D}+\beta_2\mathbf x .
  \label{eq:fprm-scaling}
\end{equation}
Via a boundedness theorem that couples the four coefficients:,
\begin{equation}
\textstyle
  \beta_2=1-\alpha_2\alpha_1^{2D},
  \qquad
  \beta_1=\frac{\beta_2(1-\alpha_1)}{1-\alpha_1^{2D}},
  \qquad
  0\leq\alpha_1,\alpha_2<1,
  \label{eq:fprm-coeff}
\end{equation}
ensures that the recurrent carry path (gain $\alpha_2\alpha_1^{2D}$) is balanced against the injected input. Given a proposal $\widetilde{\mathbf z}_i=F_\theta(\mathbf z_i;\mathbf x)$, FPRM damps linearly:
\begin{equation}
  \mathbf z_{i+1}=(1-\eta_i)\mathbf z_i+\eta_i\widetilde{\mathbf z}_i,
  \label{eq:fprm-damping}
\end{equation}
with a per-sample $\eta_i\in(0,1]$ that is decayed on a patience rule whenever the residual stalls, and halts on the relative residual:
\begin{equation}
\textstyle
  r_{\mathrm E,i}
  =
  \frac{\norm{\mathbf z_i-\widetilde{\mathbf z}_i}_\infty}{\norm{\widetilde{\mathbf z}_i}_\infty+\epsilon},
  \qquad
  \norm{\mathbf u}_\infty:=\max_{1\le t\le T,\,1\le j\le d}|(u_t)_j|,
  \label{eq:fprm-halt}
\end{equation}
once $r_{\mathrm E,i}<\tau_{\mathrm E}$ or the budget is exhausted. The coefficient coupling and the denominator of the residual exist to \emph{contain} a Euclidean residual stream whose scale is not fixed. Rather than differentiate implicitly through the equilibrium as in \citet{Bai2019DEQ}, FPRM trains with truncated BPTT over windows of $K$ iterations, deep supervision, and state detachment between windows.

\textbf{nGPT and spherical geometry.}
nGPT \citep{Loshchilov2024nGPT} constrains every hidden state to the unit sphere $\sph^{d-1}$ via $\Norm(u)=u/\norm{u}_2$ for a nonzero token vector and folds each sub-layer output into the state through learned, per-coordinate \emph{eigen learning rates} $\alpha_A,\alpha_M\in\reals^d$:
\begin{equation}
\textstyle
  \begin{aligned}
  \mathbf z&\leftarrow\Norm\!\big(\mathbf z+\alpha_A\odot(\Norm(\Attn(\mathbf z))-\mathbf z)\big),\\
  \mathbf z&\leftarrow\Norm\!\big(\mathbf z+\alpha_M\odot(\Norm(\MLP(\mathbf z))-\mathbf z)\big).
  \end{aligned}
  \label{eq:ngpt-retraction}
\end{equation}
$\Norm$ acts token-wise on tuples, and $\alpha_A,\alpha_M$ are shared across tokens. Each sub-layer is a normalized retraction toward a learned target, so unit norm is preserved by construction rather than enforced by a constraint \citep{absil2008optimization}. For unit vectors $u,v$ write $\angle(u,v):=\arccosop(\inner{u}{v})$. For $\Omega=\angle(u,v)\in(0,\pi)$, the spherical linear interpolation:
\begin{equation}
\textstyle
  \label{eq:slerp}
  \Slerp_{\eta}(u,v)
  =
  \frac{\sin((1-\eta)\Omega)}{\sin\Omega}\,u
  +
  \frac{\sin(\eta\Omega)}{\sin\Omega}\,v,
  \qquad \eta\in[0,1],
\end{equation}
moves a fraction $\eta$ along the shortest great-circle arc from $u$ to $v$ and stays on the sphere. For coincident endpoints we set $\Slerp_{\eta}(u,u):=u$, the limit of \cref{eq:slerp} as $v\to u$; at an antipodal pair ($v=-u$) the shortest arc is not unique, so we use \cref{eq:gslerp} (\Cref{sec:app-slerp}). For small $\Omega$, $\Slerp$ agrees with normalized linear interpolation $\Norm((1-\eta)u+\eta v)$ up to $O(\Omega^3)$ (\Cref{lem:gslerp}).

\subsection{\methodname: the recurrent map and its solver}
\label{sec:slfpr-core}

\methodname is a looped Transformer model whose recurrent state lives on the unit hypersphere and whose
loop is run as a damped fixed-point solver. \Cref{alg:slfpr-simple} places it next to FPRM \citep{Movahedi2026FPRM}, its Euclidean
predecessor. Every difference is a Euclidean
operation replaced by its spherical one: the state is confined to $\mathcal S$ by a weight-tied nGPT
block (line~\ref{alg:line-block}, \cref{eq:ngpt-retraction}), affine injection becomes a geodesic
move a fraction $\lambda$ of the arc toward the input (line~\ref{alg:line-inject},
\cref{eq:slfpr-map}), the relative residual becomes the chordal one (line~\ref{alg:line-residual},
\cref{eq:slfpr-residual}), and linear damping becomes a rotation a fraction $\eta_i$ of the way
toward the proposal (line~\ref{alg:line-damp}, \cref{eq:slfpr-step}). 

\begin{wrapfigure}[21]{r}{0.45\linewidth}
\vspace{-5mm}
\centering
\resizebox{\linewidth}{!}{
\begin{tikzpicture}[>=Latex, font=\small,
  lbl/.style={font=\footnotesize}, note/.style={font=\scriptsize, black!65}]

\begin{scope}[shift={(0.45,0)}]
  \draw[black!20, thin] (0,0) circle ({4*0.21});
  \draw[black!20, thin] (0,0) circle (2.520);
  \node[note, anchor=south west, inner sep=1pt] at ({4*0.21*0.7071},{4*0.21*0.7071}) {$4^\circ$};
  \node[note, anchor=south west, inner sep=1pt] at ({2.520*0.7071},{2.520*0.7071}) {$12^\circ$};
  \draw[black!55, densely dashed, thin] (0,0) circle ({7.5*0.21});
  \fill[fprmcol, fill opacity=0.18] plot[smooth cycle] coordinates {(2.036,0.000) (2.199,0.231) (2.376,0.505) (2.056,0.668) (1.794,0.799) (1.572,0.907) (1.377,1.000) (1.201,1.082) (1.042,1.156) (0.890,1.225) (0.746,1.291) (0.603,1.354) (0.460,1.417) (0.315,1.480) (0.162,1.545) (0.000,1.613) (-0.178,1.687) (-0.376,1.767) (-0.604,1.860) (-0.738,1.656) (-0.851,1.474) (-0.953,1.312) (-1.048,1.164) (-1.139,1.025) (-1.227,0.892) (-1.315,0.759) (-1.405,0.625) (-1.499,0.487) (-1.601,0.340) (-1.714,0.180) (-1.841,0.000) (-1.994,-0.210) (-2.056,-0.438) (-1.805,-0.586) (-1.595,-0.710) (-1.416,-0.817) (-1.255,-0.912) (-1.111,-1.000) (-0.974,-1.082) (-0.844,-1.162) (-0.716,-1.240) (-0.587,-1.320) (-0.455,-1.402) (-0.316,-1.487) (-0.166,-1.582) (-0.000,-1.687) (0.190,-1.809) (0.415,-1.954) (0.585,-1.801) (0.729,-1.639) (0.860,-1.490) (0.981,-1.350) (1.094,-1.215) (1.203,-1.083) (1.310,-0.952) (1.417,-0.819) (1.527,-0.680) (1.640,-0.533) (1.761,-0.374) (1.892,-0.199) (2.036,-0.000)};
  \draw[fprmcol, very thick] plot[smooth cycle] coordinates {(2.036,0.000) (2.199,0.231) (2.376,0.505) (2.056,0.668) (1.794,0.799) (1.572,0.907) (1.377,1.000) (1.201,1.082) (1.042,1.156) (0.890,1.225) (0.746,1.291) (0.603,1.354) (0.460,1.417) (0.315,1.480) (0.162,1.545) (0.000,1.613) (-0.178,1.687) (-0.376,1.767) (-0.604,1.860) (-0.738,1.656) (-0.851,1.474) (-0.953,1.312) (-1.048,1.164) (-1.139,1.025) (-1.227,0.892) (-1.315,0.759) (-1.405,0.625) (-1.499,0.487) (-1.601,0.340) (-1.714,0.180) (-1.841,0.000) (-1.994,-0.210) (-2.056,-0.438) (-1.805,-0.586) (-1.595,-0.710) (-1.416,-0.817) (-1.255,-0.912) (-1.111,-1.000) (-0.974,-1.082) (-0.844,-1.162) (-0.716,-1.240) (-0.587,-1.320) (-0.455,-1.402) (-0.316,-1.487) (-0.166,-1.582) (-0.000,-1.687) (0.190,-1.809) (0.415,-1.954) (0.585,-1.801) (0.729,-1.639) (0.860,-1.490) (0.981,-1.350) (1.094,-1.215) (1.203,-1.083) (1.310,-0.952) (1.417,-0.819) (1.527,-0.680) (1.640,-0.533) (1.761,-0.374) (1.892,-0.199) (2.036,-0.000)};
  \fill[slfprcol, fill opacity=0.14] (0,0) circle ({8.110*0.21});
  \draw[slfprcol, very thick] (0,0) circle ({8.110*0.21});
  \node[lbl, slfprcol!85!black, anchor=east, inner sep=1pt] at ({-8.110*0.21-0.04},{0.30}) {disk};
  \node[lbl, fprmcol!85!black, anchor=south east, inner sep=1pt] at (-0.604,1.860) {box};
  \fill[black] (0,0) circle (1.5pt) node[below left=0pt, lbl] {$z_{i,t}$};
  \draw[black!60, thin] (0,0) -- (1.540,0.328);
  \draw[black!60, thin] (0,0) -- (-0.834,-1.336);
  \fill[black] (1.540,0.328) circle (1.6pt);
  \fill[black] (-0.834,-1.336) circle (1.6pt);
  \node[lbl, anchor=west, inner sep=2pt] at (1.540,0.328)
    {$\widetilde z_{i,t}^{\,(1)}$ {\color{slfprcol}$\checkmark$}{\color{fprmcol}$\checkmark$}};
  \node[lbl, anchor=north east, inner sep=2pt] at (-0.834,-1.336)
    {$\widetilde z_{i,t}^{\,(2)}$ {\color{slfprcol}$\checkmark$}{\color{fprmcol}$\times$}};
  \node[note, anchor=north, align=center] at (0,{-2.520-0.10})
    {\footnotesize radius $=$ geodesic angle $\angle(z_{i,t},\widetilde z_{i,t})$};
\end{scope}

\node[anchor=south west, align=left, lbl, inner sep=2pt] at (-1.05,{2.520+0.22}) {
  {\color{slfprcol}\rule[0.25ex]{1.2em}{1.6pt}}\ SFPRM: $r_{\mathrm c,i}=\max_t\rho_{i,t}<\tau_c$\\[1pt]
  {\color{fprmcol}\rule[0.25ex]{1.2em}{1.6pt}}\ FPRM: $\textstyle
 r_{\mathrm E,i}=\frac{\norm{\mathbf z_i-\widetilde{\mathbf z}_i}_\infty}{(\norm{\widetilde{\mathbf z}_i}_\infty{+}\epsilon)}<\tau_{\mathrm E}$};

\begin{scope}[shift={(-2.75,2.65)}]
  \shade[ball color=black!8, opacity=0.9] (0,0) circle (0.72);
  \draw[black!40] (0,0) circle (0.72);
  \draw[black!30, dashed] (-0.72,0) arc[start angle=180, end angle=360, x radius=0.72, y radius=0.23];
  \draw[black!30] (-0.72,0) arc[start angle=180, end angle=0, x radius=0.72, y radius=0.23];
  \coordinate (zp) at (0.30,0.44);
  \fill[black] (zp) circle (1.2pt) node[note, above left=-1pt] {$z_{i,t}$};
  \draw[slfprcol, thick] (zp) ellipse [x radius=0.19, y radius=0.11, rotate=-20];
  \node[note, anchor=north, inner sep=1pt] at (0,-0.78) {$\mathcal S^{d-1}$};
  \draw[black!35, dashed] ($(zp)+(0.18,0.06)$) -- ($(2.75+0.45,-2.65)+(112:2.520)$);
  \draw[black!35, dashed] ($(zp)+(0.02,-0.14)$) -- ($(2.75+0.45,-2.65)+(150:2.520)$);
\end{scope}
\end{tikzpicture}
}
\caption{Proposals accepted by each halting test around a token state $z_{i,t}$. The chordal
test accepts a disk; FPRM's accepts an axis-aligned $\ell_\infty$ section, so
two proposals at the same angle can be judged differently.}
\label{fig:halting-criteria}
\end{wrapfigure}
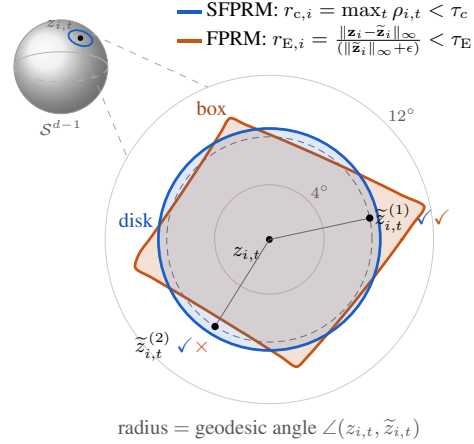

\textbf{State space and spherical block.}
\methodname's recurrent state is a tuple of unit token vectors, $\mathbf z_i=(z_{i,1},\ldots,z_{i,T})\in\mathcal S:=(\sph^{d-1})^T$, where $\norm{z_{i,t}}_2=1$ and the input $\mathbf x=(x_1,\ldots,x_T)\in\mathcal S$ is the embedded sequence normalized token-wise. The weight-tied network $f_\theta:\mathcal S\to\mathcal S$ is a stack of $D$ nGPT layers, each an attention retraction followed by an MLP retraction as in \cref{eq:ngpt-retraction}, with QK-normalization and eigen learning rates shared across iterations. Attention couples tokens, so $f_\theta$ acts on the product $\mathcal S$ rather than token by token; it is the block applied at every iteration in \cref{eq:slfpr-map}. 
Following nGPT's logit scaling, predictions are read out as $s_z\,h_\phi(\mathbf z)$, where, given the vocabulary $\mathcal V$, $h_\phi$ is a linear map $\reals^d\to\reals^{|\mathcal V|}$ and $s_z$ is a learned scalar. nGPT would normalize the rows of $\phi$ like every other weight vector, however we keep them unnormalized so the head can set its own per-class scale while the state stays on the sphere.

\textbf{Geodesic input injection.}
We lift $\Slerp$ from tokens to sequences by applying it position-wise. Given two unit token vectors $\mathbf u$ and $\mathbf v$, we write:
\begin{equation}
\textstyle
  \label{eq:sequence-slerp}
  \Slerp_\eta(\mathbf u,\mathbf v)
  :=\bigl(\Slerp_\eta(u_t,v_t)\bigr)_{t=1}^{T},\qquad \eta\in[0,1].
\end{equation}
At every iteration $i$, the current state is first pulled toward the input along the shortest arc using $\Slerp_\lambda$ of \cref{eq:sequence-slerp}, and the result is then passed through the nGPT block of \cref{eq:ngpt-retraction}:
\begin{equation}
  \label{eq:slfpr-map}
  F_\theta(\mathbf z;\mathbf x)
  :=f_\theta\!\left(\Slerp_{\lambda}(\mathbf z,\mathbf x)\right),
  \qquad 0<\lambda<1 .
\end{equation}
The injection strength $\lambda$ is a model parameter, fixed or learned, with the same reading in either case: it is the fraction of the arc from each token state to its input token that is traversed before the block is applied. Geodesic injection is the intrinsic form of FPRM's additive mixing $\alpha_2\mathbf z+\beta_2\mathbf x$ (\cref{eq:fprm-scaling}): normalizing any positive combination $\alpha u+\beta v$ lands on the same shortest arc from $\widehat u$ to $\widehat v$, at a fraction set by the coefficients, the norms, and the angle (\Cref{lem:fprm-slerp} in \Cref{app:proof-euclidean-injection}). FPRM's coefficients therefore do two jobs at once: they place the state along this arc, and, through the coupling (\cref{eq:fprm-coeff}), they keep its norm from drifting. On the unit sphere the second part is given by construction, and the first is done by $\lambda$ alone. \methodname thus replaces the implicit, state-dependent fraction by a single parameter and drops $(\alpha_1,\beta_1,\alpha_2,\beta_2)$ together with their coupling.

\textbf{Scale-free halting residual.}
Given the same proposal $\widetilde{\mathbf z}_i=F_\theta(\mathbf z_i;\mathbf x)$ of \cref{eq:slfpr-map}, geodesic
distance on $\mathcal S$ is the angle between unit vectors, so \methodname halts on the worst-token
\emph{chordal} residual, which is smooth where the angle itself is not:
\begin{equation}
\textstyle
  \label{eq:slfpr-residual}
  r_{\mathrm c,i}
  =
  \max_{1\le t\le T}\rho_{i,t},
  \qquad
  \rho_{i,t}
  :=1-\inner{z_{i,t}}{\widetilde z_{i,t}}
  =\tfrac12\norm{z_{i,t}-\widetilde z_{i,t}}_2^2,
\end{equation}
and halts once $r_{\mathrm c,i}<\tau_c$, where $\tau_c=1-\cos\tau_{\mathrm{ang}}$ for an angular tolerance $\tau_{\mathrm{ang}}$, so the test fires exactly when every token's proposal lies within $\tau_{\mathrm{ang}}$ of its iterate. The \emph{worst token} $\max_t$ is the spherical analogue of the maximum over tokens and coordinates in $r_{\mathrm E,i}$ (\cref{eq:fprm-halt}): a few unconverged but decisive tokens should keep the iteration running. What FPRM's denominator and $\epsilon$ provide is scale invariance, and on the unit sphere that is automatic, so neither is needed. The two tests remain inequivalent: $\rho_{i,t}$ sees only the geodesic angle and so accepts a disk of radius
$\tau_{\mathrm{ang}}$ around the iterate, whereas $r_{\mathrm E,i}$ sees the largest-magnitude
coordinate and so accepts an axis-aligned $\ell_\infty$ region whose angular extent depends on the
direction of the update, and two proposals at the same angle can get opposite verdicts (\Cref{fig:halting-criteria}). We adopt the
chordal form for its simplicity and frame-independence.

\textbf{Geodesic damping.}
If the residual is above tolerance, the state is advanced by a rotation toward the proposal, in place of FPRM's linear interpolation (\cref{eq:fprm-damping}):
\begin{equation}
  \label{eq:slfpr-step}
  \mathbf z_{i+1}
  :=\Slerp_{\eta_i}(\mathbf z_i,\widetilde{\mathbf z}_i),
  \qquad 0<\eta_i\le1 .
\end{equation}
By \Cref{lem:slerp-geometry}, token $t$ turns by exactly
$\eta_i\,\angle(z_{i,t},\widetilde z_{i,t})$, so $\eta_i$ is an \emph{angular step size}: $\eta_i=1$ applies the raw map and smaller values shrink the rotation without leaving the sphere, so the block is always evaluated on the geometry it was trained on. The update never changes the equilibria: for $\eta_i>0$ the step returns $\mathbf z_i$ if and only if $\widetilde{\mathbf z}_i=\mathbf z_i$, so the damped iteration and $F_\theta$ have exactly the same fixed points wherever the arcs are defined. As in FPRM (\cref{eq:fprm-damping}), $\eta_i$ is per sample, starts at $1$, and is multiplied by $\gamma<1$ whenever the residual has not improved for $P$ iterations. On the sphere this patience rule is literally a shrinking of the per-iteration rotation angle, a derivative-free search for a step inside the stability window of \Cref{prop:converge}.

\begin{algorithm}[t]
\caption{FPRM (left) and \methodname (right) in one skeleton. \textbf{Black}${=}$shared control flow; \dfp{orange}${=}$FPRM-specific; \dsl{blue}${=}$\methodname-specific. Both start from a random state $\mathbf z_0$ and re-inject the input before every application of $f_\theta$; the coloured lines are the only differences.}
\label{alg:slfpr-simple}
\begin{minipage}[t]{0.485\textwidth}
\centering\textbf{FPRM} {\footnotesize\dfp{(Euclidean fixed point)}}
\hrule height 0.3pt \vspace{2pt}
{\footnotesize
\begin{algorithmic}[1]
\Require Euclidean block $f_\theta$ \dfp{(Transformer)}, head $h_\phi$; tol.\ $\tau_{\mathrm E}$;
  decay $\gamma$; patience $P$; cap $N_{\max}$
\State $\mathbf z \gets \mathbf z_0$ \Comment{random init}
\State $\eta{\gets}1$;\ $p{\gets}P$;\ $r_{\mathrm{best}}{\gets}\infty$
\Repeat
  \State $\widetilde{\mathbf z} \gets \dfp{\alpha_2\,\mathbf z + \beta_2\,\mathbf x}$
    \Comment{\dfp{affine input injection}}
  \State $\widetilde{\mathbf z} \gets f_\theta(\widetilde{\mathbf z})$
    \Comment{\dfp{$2D$ sub-layers, $(\alpha_1,\beta_1)$}}
  \State $r \gets \dfp{\norm{\mathbf z-\widetilde{\mathbf z}}_\infty / (\norm{\widetilde{\mathbf z}}_\infty+\epsilon)}$
    \Comment{\dfp{relative-$\ell_\infty$}}
  \State \textbf{if} $r{<}r_{\mathrm{best}}$ \textbf{then} $r_{\mathrm{best}}{\gets}r$;\ $p{\gets}P$
    \textbf{else} $p{\gets}p{-}1$
  \State \textbf{if} $p{\le}0$ \textbf{and} $r{>}\tau_{\mathrm E}$ \textbf{then} $\eta{\gets}\gamma\eta$;\ $p{\gets}P$
  \State $\mathbf z \gets \dfp{(1{-}\eta)\,\mathbf z + \eta\,\widetilde{\mathbf z}}$ \Comment{\dfp{linear damping}}
  \State \emph{(train)} every $K$: backprop $\mathrm{CE}\big(h_\phi(\mathbf z),y\big)$; detach $\mathbf z$
\Until{$r{<}\tau_{\mathrm E}$ \ \textbf{or}\ $N_{\max}$ steps}
\State \Return $h_\phi(\mathbf z)$
\end{algorithmic}}
\end{minipage}%
\hfill\vrule\hfill
\begin{minipage}[t]{0.485\textwidth}
\centering\textbf{\methodname} {\footnotesize\dsl{(spherical fixed point)}}
\hrule height 0.3pt \vspace{2pt}
{\footnotesize
\begin{algorithmic}[1]
\Require spherical block $f_\theta$ \dsl{(nGPT)}, head $h_\phi$; tol.\ $\tau_c$;
  decay $\gamma$; patience $P$; cap $N_{\max}$
\State $\mathbf z \gets \mathbf z_0\,\dsl{\in\mathcal S}$ \Comment{random unit-norm init}
\State $\eta{\gets}1$;\ $p{\gets}P$;\ $r_{\mathrm{best}}{\gets}\infty$
\Repeat
  \State \label{alg:line-inject} $\widetilde{\mathbf z} \gets \dsl{\Slerp_{\lambda}(\mathbf z,\mathbf x)}$
    \Comment{\dsl{geodesic injection}}
  \State \label{alg:line-block} $\widetilde{\mathbf z} \gets f_\theta(\widetilde{\mathbf z})$
    \Comment{\dsl{tied spherical block}}
  \State \label{alg:line-residual} $r \gets \dsl{\max_t\big(1-\inner{z_t}{\widetilde z_t}\big)}$
    \Comment{\dsl{chordal residual}}
  \State \label{alg:patience} \textbf{if} $r{<}r_{\mathrm{best}}$ \textbf{then} $r_{\mathrm{best}}{\gets}r$;\ $p{\gets}P$
    \textbf{else} $p{\gets}p{-}1$
  \State \label{alg:eta}\textbf{if} $p{\le}0$ \textbf{and} $r{>}\tau_c$ \textbf{then} $\eta{\gets}\gamma\eta$;\ $p{\gets}P$
  \State \label{alg:line-damp} $\mathbf z \gets \dsl{\Slerp_{\eta}(\mathbf z,\widetilde{\mathbf z})}$ \Comment{\dsl{geodesic damping}}
  \State \emph{(train)} every $K$: backprop $\mathrm{CE}\big(\dsl{s_z}h_\phi(\mathbf z),y\big)$;
  detach $\mathbf z$
\Until{$r{<}\tau_c$ \ \textbf{or}\ $N_{\max}$ steps}
\State \Return $\dsl{s_z}\,h_\phi(\mathbf z)$ \Comment{\dsl{scaled readout}}
\end{algorithmic}
}
\end{minipage}
\end{algorithm}


\textbf{Training.}
Training follows FPRM: deep supervision with truncated BPTT, i.e., each optimizer step backpropagating through one segment of $K$ geodesic steps with the state detached between segments. The residual, patience counter and step-size schedule are computed from detached states, so they steer the solver without entering the gradient; differentiating $\Norm$ projects state cotangents onto the tangent space, parameters stay Euclidean, and weight vectors are renormalized after every optimizer step. In practice, we use a guarded $\Slerp$ (\cref{eq:gslerp}), which replaces the sine formula by normalized linear interpolation near coincident endpoints and by an explicit arc near antipodes.

\subsubsection{Theoretical guarantees}
\label{sec:guarantees}

We now investigate what the geometry buys analytically. All statements concern the exact recurrence
\crefrange{eq:slfpr-map}{eq:slfpr-step} with shortest-arc interpolation. Full proofs are in Appendix \ref{app:proofs}.

\textbf{Boundedness by construction.}
Because injection and damping return unit token vectors (\Cref{lem:slerp-geometry}), the block does so by construction (\cref{eq:ngpt-retraction}), and so does every branch of the guarded step (\Cref{lem:gslerp}), the state stays on $\mathcal S$ for every $i$ and every $\theta$: recurrent magnitudes can neither
explode nor vanish, which is what FPRM's boundedness theorem secures through the coefficient
constraints in \cref{eq:fprm-coeff}. However, boundedness alone says nothing about whether a
fixed point exists or is reached, which is what we address next.

\textbf{Input injection guarantees a fixed point.}
On a sphere, continuity alone does not guarantee a fixed point since rotating a circle moves every point. Input injection repairs this by moving each token a fraction $\lambda$ of the way to its input, keeping every injected state in the cap $\mathcal C_{\mathbf x}
  :=\left\{
    \mathbf y\in\mathcal S:\
    \angle(y_t,x_t)\le(1-\lambda)\pi
    \text{ for every token }t
  \right\}$,
which, unlike the sphere, can be flattened into a closed ball, where every continuous self-map has a fixed point.

\begin{proposition}[Fixed-point existence]
\label{prop:existence}
Fix $\mathbf x$ and $0<\lambda<1$. Assume $f_\theta$ is continuous on $\mathcal C_{\mathbf x}$ and
  $f_\theta(\mathbf y)_t\ne-x_t
  \text{ for every }\mathbf y\in\mathcal C_{\mathbf x}
  \text{ and token }t$.
Then the recurrent map $F_\theta(\cdot\,;\mathbf x)$ in \cref{eq:slfpr-map} has at least one fixed
point $\mathbf z^\star=F_\theta(\mathbf z^\star;\mathbf x)$ in $\mathcal S$.
\end{proposition}
The full proof (Appendix \ref{app:proof-existence}) applies Brouwer's theorem \citep{Brouwer1911,Ortega1970}
to the reversed composition $H(\mathbf y):=\Slerp_{\lambda}(f_\theta(\mathbf y),\mathbf x)$. Unlike
$F_\theta$, which injects before the block, $H$ ends with the injection, so it maps the cap into
itself, and the antipode assumption keeps it well defined and continuous; its fixed point
$\mathbf y^\star$ gives the fixed point $\mathbf z^\star:=f_\theta(\mathbf y^\star)$ of $F_\theta$. No Lipschitz
constant, contraction modulus, or spectral bound on $f_\theta$ is needed. FPRM proves existence
only via global convergence, hence under contraction \citep[Theorem~2]{Movahedi2026FPRM};
Brouwer alone would need the coefficient coupling that keeps a ball invariant. On the sphere compactness is structural and injection supplies the missing ball-shaped domain. Only $\mathbf y^\star$, however, provably lies in $\mathcal C_{\mathbf x}$, and existence implies neither uniqueness nor attraction, which the next two results address.

\textbf{Angular damping stabilizes a fixed point.}
Let $\mathbf z^\star$ be a fixed point of $F_\theta(\cdot\,;\mathbf x)$ and let $J:=D_{\mathbf z}F_\theta(\mathbf z^\star;\mathbf x)$
be its Jacobian in the recurrent state at $\mathbf z^\star$, with $\theta$ and $\mathbf x$ held
fixed. Since $F_\theta$ maps $\mathcal S$ to itself, $J$ sends tangent perturbations
$\mathbf u=(u_t)_{t=1}^T$ with $\inner{u_t}{z_t^\star}=0$ to tangent perturbations, and we regard
it as an endomorphism of that $T(d-1)$-dimensional tangent space. It contains all coupling between
tokens through attention and need be neither orthogonal nor diagonalizable. 

\begin{theorem}[Local convergence of angular damping]
\label{prop:converge}
Suppose $F_\theta(\cdot\,;\mathbf x)$ is differentiable at a fixed point
$\mathbf z^\star$ and every $\zeta\in\sigma(J)$, where $\sigma(J)$
is the Jacobian's set of eigenvalues, satisfies $\Real(\zeta)<1$. Then for every fixed
$0<\eta\le1$ with
\begin{equation}
\textstyle
  \label{eq:eta-stability-window}
  \eta<\eta_{\max}
  :=\min_{\zeta\in\sigma(J)}
  \frac{2\bigl(1-\Real(\zeta)\bigr)}{|\zeta-1|^2},
\end{equation}
the damped iteration
$\mathbf z_{i+1}=\Slerp_{\eta}\!\left(\mathbf z_i,F_\theta(\mathbf z_i;\mathbf x)\right)$
converges to $\mathbf z^\star$ from every sufficiently close initial state, with
$\norm{\mathbf z_i-\mathbf z^\star}\le Cq^i\norm{\mathbf z_0-\mathbf z^\star}$ for some $C>0$ and
$q\in(0,1)$. Moreover, wherever the shortest arcs are defined, the damped iteration and $F_\theta$
have exactly the same fixed points.
\end{theorem}
This is the spherical counterpart of FPRM's stabilization result
\citep[Theorem~3]{Movahedi2026FPRM}, with an explicit window and with unit norm and fixed points
preserved. The window is certified for a fixed $\eta$; the patience rule
(lines~\ref{alg:patience}--\ref{alg:eta} of \Cref{alg:slfpr-simple}) searches for one
heuristically, without guarantee. The proof (Appendix \ref{app:proof-convergence}) linearizes the damped
step at $\mathbf z^\star$ to $(1-\eta)\operatorname{Id}+\eta J$ (\Cref{lem:slerp-linearization}),
which moves each eigenvalue $\zeta$ of $J$ a fraction $\eta$ from $1$ toward $\zeta$; the window
of \cref{eq:eta-stability-window} is exactly what keeps them all inside the unit disk. A spectral
radius below one then gives geometric convergence in tangent coordinates, which the coordinate map
transfers to the state. Such an $\eta$ exists only if $\Real(\zeta)<1$ for every $\zeta$: damping
tames modes that oscillate or spiral away under the raw map, not expansive ones with
$\Real(\zeta)\ge1$, which training must remove (\Cref{rem:damping}). 

\textbf{Global convergence under strong injection.}
Finally, we identify a regime where the raw recurrence ($\eta=1$) converges from every admissible initialization. $\mathbf x$ and $f_\theta$ are held fixed while $\lambda$ varies.

\begin{theorem}[Global convergence under strong input injection]
\label{thm:global-convergence}
Suppose $f_\theta:\mathcal S\to\mathcal S$ is continuously differentiable in a neighborhood of
$\mathbf x$ and $f_\theta(\mathbf x)_t\ne-x_t$ for every token $t$. Then there exists a
$\lambda_0\in(0,1)$, depending on $f_\theta$ and $\mathbf x$, such that for every
$\lambda\in(\lambda_0,1)$ the map $F_\theta(\cdot\,;\mathbf x)$ has a fixed point $\mathbf z^\star$,
unique among all states at which the input interpolation is defined. For each such $\lambda$ and
every $\mathbf z_0\in\mathcal S$ with $z_{0,t}\ne-x_t$ for all $t$, the iterates
$\mathbf z_{i+1}=F_\theta(\mathbf z_i;\mathbf x)$ are well defined and satisfy
$\norm{\mathbf z_i-\mathbf z^\star}\le Cq^{i-1}$ for all $i\ge1$, with $C>0$ and $q\in(0,1)$
independent of $\mathbf z_0$.
\end{theorem}
The proof (Appendix \ref{app:proof-global-convergence}, with an explicit $\lambda_0$) uses logarithmic coordinates at $\mathbf x$, where injection multiplies by $1-\lambda$: for $\lambda$ near one, the first injection pulls every admissible state into a small ball around $\mathbf x$ on which the map is a Banach contraction, so only that ball matters. FPRM's global result \citep[Theorem~2]{Movahedi2026FPRM} instead rests on a \emph{global} Lipschitz bound and a small residual coefficient, which the sphere lacks: injection does not contract near the input's antipode. Together the three results cover every injection strength: a fixed point exists for all $\lambda$, strong injection makes it a unique global attractor, and damping gives local convergence under weaker injection.

\section{Empirical Evaluation}
\label{sec:experiments}

\begin{wrapfigure}[20]{r}{0.47\linewidth}
  \vspace{-13mm}
  \centering
  \includegraphics[width=\linewidth]{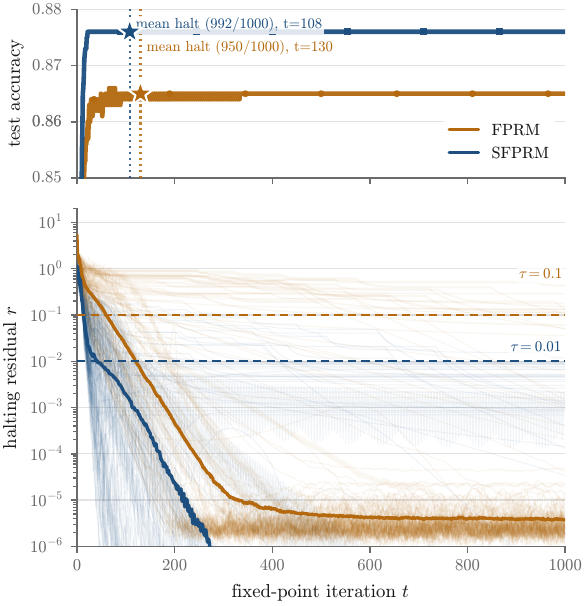}
  \vspace{-6mm}
  \caption{\textbf{Exact accuracy and halting on Maze-Hard} (Sudoku: \cref{fig:halting-sudoku}). \methodname vs.\ FPRM evaluation on $1000$ test puzzles. \emph{Top:} exact accuracy; mean halt and number of examples halting. \emph{Bottom:} per-sequence residual (thin), median (heavy).}
  \label{fig:halting}
\end{wrapfigure}
With the following experiments we answer three research questions:
\begin{enumerate}[label=\textbf{(RQ\arabic*)}, leftmargin=*]
  \item \emph{Accuracy.} How does \methodname compare to FPRM and other recurrent reasoner baselines?
  Does constraining the recurrence on the sphere cost accuracy? (\Cref{sec:exp-puzzles})
  \item \emph{Convergence and halting.} How effectively does \methodname utilize the learned recurrence under the angular halting criterion and is the fixed-point unique? (\Cref{sec:exp-halting})
  \item \emph{Geometry of the computation.} What does the fixed-point iteration write into the state, and does the spherical parameterization make it readable? (\Cref{sec:exp-geometry})
\end{enumerate}

\begin{table}[t]
\centering
\caption{\textbf{Comparison with published recurrent reasoners.} Reported numbers are taken from the
respective papers; Mini-ARC numbers are from \citet{Huang2026EqR} (Table~14). Numbers in
parentheses are our own runs under the same setting as SFPRM. Tf $=$ Transformer,
Mixer $=$ MLP-Mixer, rec.\ $=$ recurrent refinement, FP $=$ fixed-point solve. Within the $5$--$7$M group
the best result is bold and the second best underlined.
}
\vspace{-2mm}
\small
\setlength{\tabcolsep}{4pt}
\resizebox{\linewidth}{!}{
\begin{tabular}{c@{\hspace{4pt}}llcccccc}
\toprule
& Model & Architecture & \# params. & \shortstack{Single Loop\\(No Hier.)} &
\shortstack{Sudoku-Ext.\\Pass@1} & \shortstack{Maze-Hard\\Pass@1} &
\shortstack{Mini-ARC\\Pass@1} & \shortstack{ARC-1\\Pass@2} \\
\midrule
\multirow{3}{*}{\rotatebox[origin=c]{90}{\emph{Larger}}}
& HRM            & rec.\ Tf                   & 27M   & \xmark & 55.0\% & 74.5\% & 44.9\% (51.7\%) & 40.3\% \\
& Attractor Model& Tf $+$ FP loop             & 27M   & \xmark & 91.4\% & 93.1\% & --- & --- \\
& URM            & rec.\ Tf $+$ conv          & 14M   & \xmark & 77.6\% & --- & --- & --- \\
\midrule
\multirow{5}{*}{\rotatebox[origin=c]{90}{\emph{$5$--$7$M}}}
& TRM            & rec.\ Tf                   & 7M    & \xmark & 74.7\% & 85.3\% & 48.4\% (36.9\%) & 44.6\% \\
& EqR            & rec.\ Mixer $+$ restarts      & 5-7M    & \xmark & 93.0\% & --- & \underline{55.3\%} (40.3\%) & --- \\
& Attractor Model& FP Tf           & 7M    & \xmark & 54.3\% & 46.7\% & --- & --- \\
& FPRM           & FP Tf                      & 7M    & \cmark & \underline{94.2\%} (78.4\%) & \underline{87.0\%} (83.9\%) & --- (55.0\%) & \textbf{47.5\%} \\
\rowcolor{lightgray}
& \textbf{\methodname (ours)} & spherical FP Tf & 7M & \cmark & \textbf{96.6\%} & \textbf{87.9\%} & \textbf{59.1\%} & \underline{47.4\%} \\
\bottomrule
\end{tabular}
}
\label{tab:main-results}
\vspace{-2mm}
\end{table}


\textbf{Setup.}
We evaluate \methodname on puzzle benchmarks: Maze-Hard, Sudoku-Extreme \citep{Wang2025HRM}, Mini-ARC \citep{kim2022playgrounds} and ARC-1 \citep{chollet2025arcprize2024technical}, used to assess generalization when trained with a limited number of examples.
\methodname and FPRM share a $6.83$M-parameter Transformer backbone and follow the hyperparameters of \citet{Movahedi2026FPRM} (details in Appendix \ref{app:exp-details}). \methodname trains its weight matrices with Muon \citep{jordan2024muon} and its embeddings, normalization scales, and scalars with AdamW, regularizes with variational and weight dropout, and learns the injection fraction $\lambda$, except on Sudoku ($\lambda$ fixed at $0.1$). It halts once $r_{\mathrm c,i}<\tau_c=0.01$ (\cref{eq:slfpr-residual}), i.e.\ once every token's proposal lies within $8.1^\circ$ of its iterate.

\begin{figure}[b]
\centering
\begin{subfigure}[b]{0.23\linewidth}
  \centering
  \includegraphics[width=\linewidth]{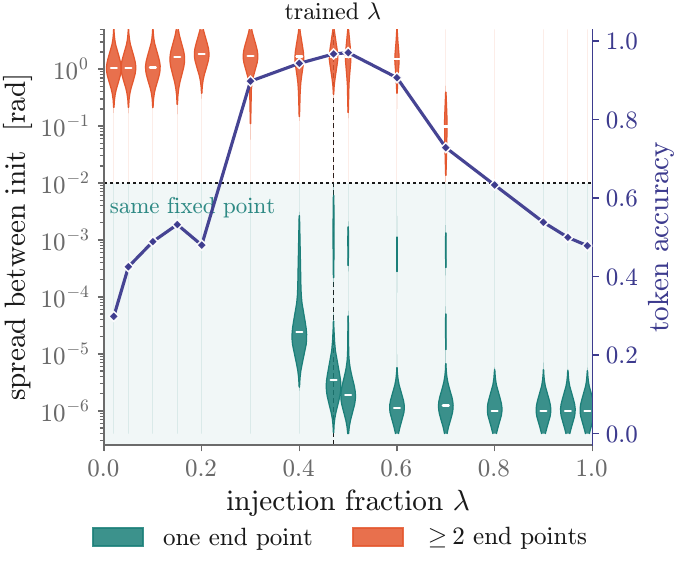}
  \caption{Uniqueness vs.\ $\lambda$.}
  \label{fig:global-conv-a}
\end{subfigure}\hfill
\begin{subfigure}[b]{0.76\linewidth}
  \centering
  \includegraphics[width=\linewidth]{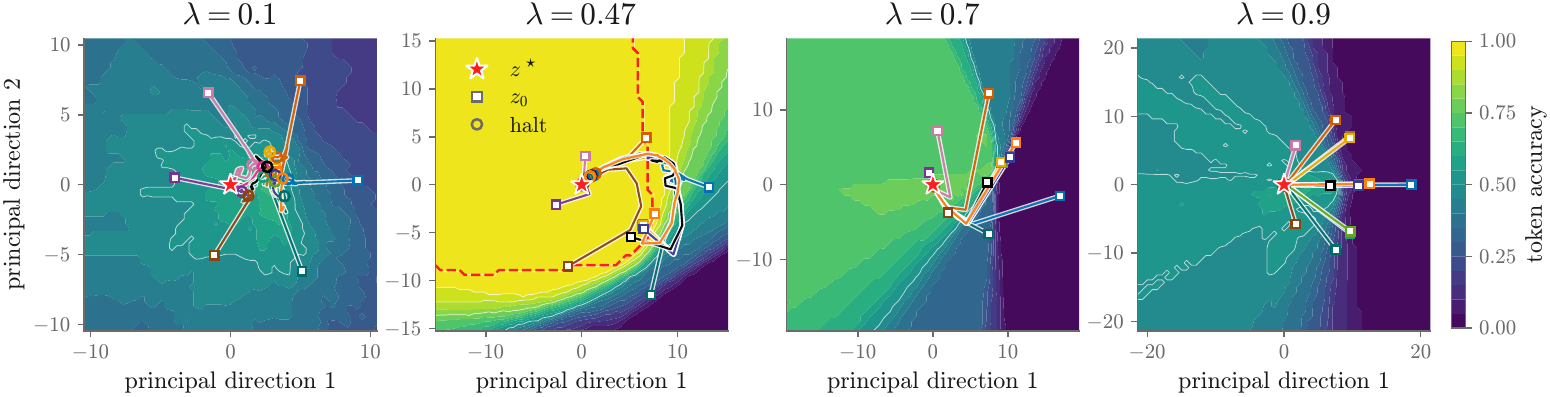}
  \caption{Trajectories in the top-2 principal plane at four values of $\lambda$.}
  \label{fig:global-conv-b}
\end{subfigure}
\caption{\textbf{Global convergence under input injection (\methodname, Sudoku).}
\textbf{(a)} Largest angle between the end states of eight initializations, over $32$ puzzles, against $\lambda$ (teal: all coincide, coral: $\ge2$ end states); indigo, right axis: token accuracy.
\textbf{(b)} Trajectories from the eight $\mathbf z_0$ (squares) to $\mathbf z^\star$ (star), coloured by the token accuracy of the decoded state; circles mark the first iteration with $r_{\mathrm c,i}<\tau_c$, dashed red the exact-match region.}
\label{fig:global-conv}
\end{figure}

\subsection{RQ1: accuracy on puzzle benchmarks}
\label{sec:exp-puzzles}



\Cref{tab:main-results} compares \methodname with published recurrent reasoners. On Sudoku-Extreme, \methodname reaches $96.6\%$ exact accuracy, $2.4$ points above FPRM published results. Since we cannot resample the same $1$k base puzzles augmented $1000\times$ from FPRM, we train it on $5$k puzzles augmented $200\times$ instead: the same number of training examples, but five times as many distinct puzzles. On Maze-Hard, \methodname reaches $87.9\%$ on the $1{,}000$ test mazes, above TRM's and FPRM's published results
and behind only the $27$M-parameter Attractor model \citep{FeinAshley2026Attractor}. 
On Mini-ARC, \methodname reaches $59.1\%$ Pass@1, above both the best published result in its size group (EqR, $55.3\%$ \citep{Huang2026EqR}) and all baselines retrained under our protocol (FPRM $55.0\%$, HRM $51.7\%$, EqR $40.3\%$, TRM $36.9\%$). On ARC-1, it matches FPRM.

\keybox{\textbf{Answer to RQ1.} Confining the state to the sphere improves accuracy: \methodname outperforms FPRM and other baselines on Sudoku-Extreme ($96.6\%$ vs.\ $94.2\%$), Maze-Hard ($87.9\%$ vs.\ $87.0\%$) and Mini-ARC ($59.1\%$ vs.\ $55.0\%$), and matches it on ARC-1 ($47.4\%$ vs.\ $47.5\%$).
}


\subsection{RQ2: convergence and halting}
\label{sec:exp-halting}

\textbf{Loop utilization and halting.}
To see how the two solvers converge and when they stop, we run each trained model with halting
disabled for $N_{\max}$ iterations, decode every iterate $\mathbf z_i$, and record the residual its
solver tests. Halting is counted per sequence, i.e., a sequence halts once its residual stays below the threshold (Appendix \ref{app:exp-rollouts}). \Cref{fig:halting} shows the result on Maze and \Cref{fig:halting-sudoku} on Sudoku. Both residuals fall quickly, and both models saturate their
accuracy within $20$--$35$ iterations, so every later loop is spent after the answer stops changing. \methodname converges faster and stops sooner: it halts on $992/1000$ mazes against FPRM's $950$, and on average $22$ iterations earlier ($108$ vs. $130$).



\textbf{Fixed-point uniqueness.}
\Cref{thm:global-convergence} predicts that for $\lambda>\lambda_0$ the raw recurrence
$\mathbf z_{i+1}=F_\theta(\mathbf z_i;\mathbf x)$ contracts to one fixed point from every
initialization. We test this by overriding $\lambda$ at inference in the trained model (learned $\lambda=0.47$) and running the raw map, without damping, in fp32 for $400$ steps on $32$ evenly spaced test puzzles, from eight initializations each: six random states, the input $\mathbf x$, and its antipode $-\mathbf x$, which the theorem excludes. The largest geodesic angle
between the eight end states measures uniqueness without a reference point
(\Cref{fig:global-conv}a). For $\lambda\ge0.8$ all eight end states coincide on every puzzle. At the
trained $\lambda=0.47$ the raw map reaches an attracting fixed point on $91\%$ of the puzzles, and on
the puzzle of \Cref{fig:global-conv}b all eight paths end inside the exact-match region. At
$\lambda=0.1$ the raw iteration does not converge, while for $\lambda\ge0.7$ it converges within
a few steps to a unique fixed point that no longer solves the puzzle.

\textbf{Varying the fixed-point solver.}
At inference we replace the damped step of \Cref{alg:slfpr-simple} by two standard solvers, keeping the model and the halting rule: Anderson acceleration \citep{anderson1965iterative,walker2011anderson}
and limited-memory Broyden \citep{broyden1965class}, the solver of deep equilibrium models \citep{Bai2019DEQ}. Both cost one evaluation of $F_\theta$ per iteration, like the damped step (details in Appendix \ref{appsec:solvers}). For both models Anderson halts earlier than the default solver, on every Maze example, while staying within one point of its accuracy (\Cref{tab:solvers} and \Cref{fig:solvers}). Broyden halts about as early on Maze but diverges afterwards, so it relies on stopping in time. Without damping ($\eta=1$), $37$--$100\%$ of examples never reach $\tau$, and on Sudoku accuracy drops as well. We leave training with these solvers, rather than swapping them in at inference, for future work.

\begin{figure}[t]
  \centering
  \begin{subfigure}[b]{0.4\linewidth}
    \centering
    \includegraphics[width=\linewidth]{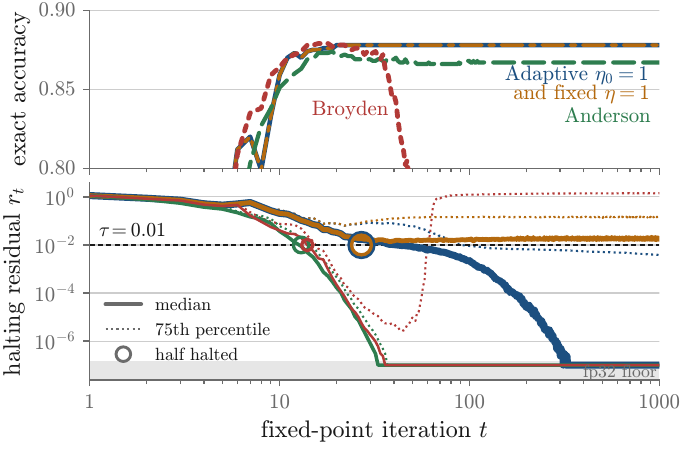}
    \caption{Accuracy per solver.}
    \label{fig:solvers-a}
  \end{subfigure}\hfill
  \begin{subfigure}[b]{0.59\linewidth}
    \centering
    \includegraphics[width=\linewidth]{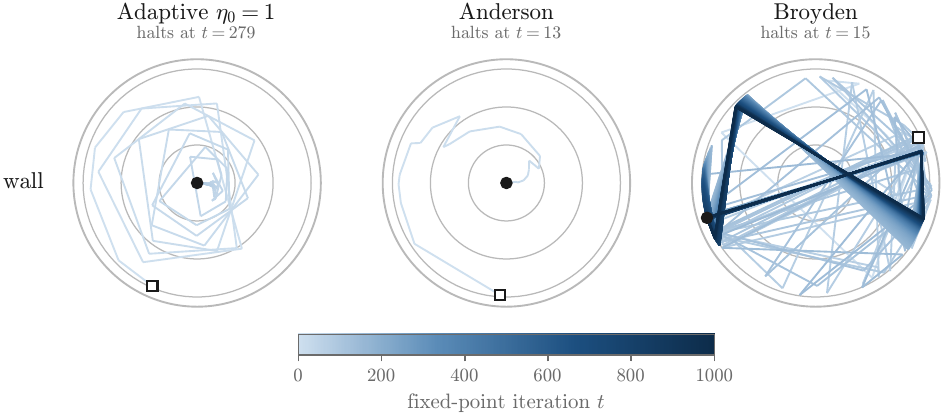}
    \caption{Path-token trajectory per solver.}
    \label{fig:solvers-b}
  \end{subfigure}
\caption{\textbf{Fixed-point solvers} (\methodname, Maze).
\textbf{(a)} Accuracy, halting disabled.
\textbf{(b)} Radius is the geodesic angle (log-scaled) to $\mathbf z^\star$, rings at $0.1/1/10/100^\circ$.}
\vspace{-4mm}
\label{fig:solvers}
\end{figure}

\keybox{\textbf{Answer to RQ2.} \methodname uses its loop iterations more effectively than FPRM: its angular criterion halts earlier on Sudoku-Extreme and Maze-Hard at equal or higher accuracy. Under strong input injection the raw iteration converges to a unique fixed point, as \Cref{thm:global-convergence} predicts.}

\subsection{RQ3: the geometry of the computation}
\label{sec:exp-geometry}

\begin{wrapfigure}[22]{r}{0.54\linewidth}
  \vspace{-4mm}
  \centering
  \includegraphics[width=\linewidth]{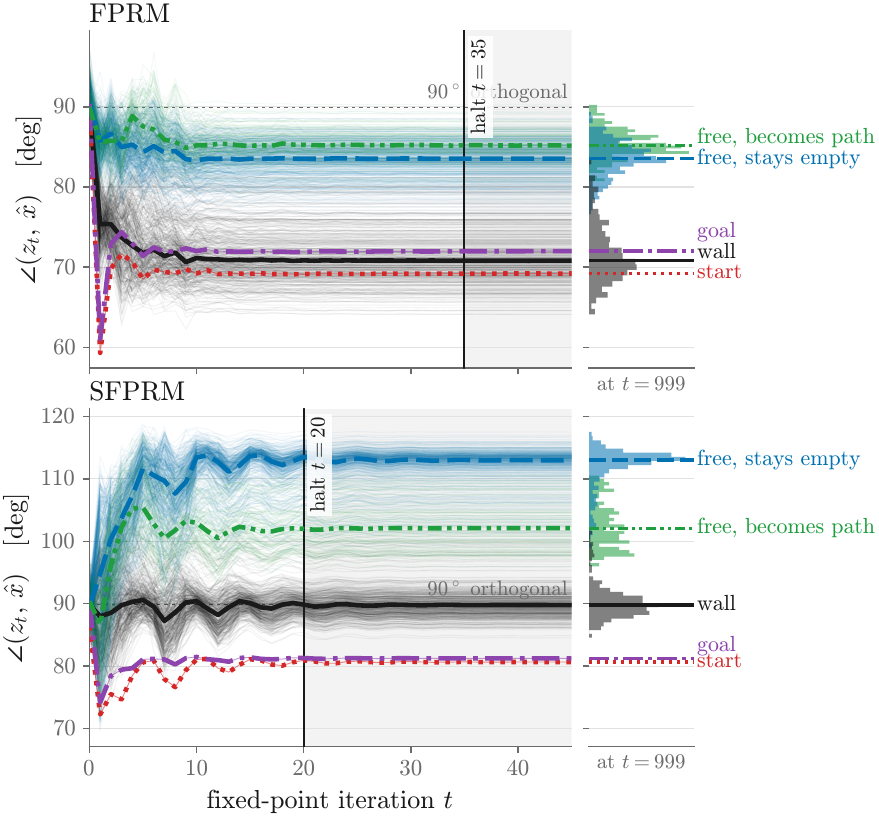}
  \vspace{-5mm}
  \caption{\textbf{Anchor angles by cell class on Maze}: cells (thin
  lines), class means (heavy), the halt iteration (vertical line), and the final distribution (margin).}
  \label{fig:anchor-angles-maze}
\end{wrapfigure}
\textbf{The state reveals the task structure.} On the sphere a token state's only free quantity is its direction, so we probe the computation through the \emph{anchor angle} $\angle(z_{i,t},x_t)$ between each cell's iterate and its unit-normalized input embedding, the quantity that injection acts on (\cref{eq:slfpr-map}).

\Cref{fig:filmstrip} shows four frames of one rollout per task (FPRM and more puzzles in \Cref{fig:filmstrips-sudoku,fig:filmstrips-maze}), and \Cref{fig:anchor-angles-maze,fig:anchor-angles} show the angle over loops for every cell of a test puzzle, clustered by cell class. On Sudoku, we can see that, for \methodname, the clue tokens rotate toward their anchors while the blanks end near orthogonal. By $i{=}39$ the anchor map is a binary image of the clue mask. Each loop iteration writes \emph{the cell class information} into the direction of the state. On the other hand, FPRM encodes the mask mainly in the norm of its state (\Cref{fig:interpretab_angles}), a scale set by training rather than by the geometry. Its anchor map shows only small differences between clues and blanks.
Similarly, on Maze, \methodname settles walls at ${\approx}90^\circ$, empty cells at ${\approx}113^\circ$, and path cells at ${\approx}102^\circ$. Empty and path cells receive the same input token but still end at different angles, so the looping itself decides which cells lie on the path instead of reading it off the input. FPRM separates walls from free cells but not path from empty cells.
On Mini-ARC (\Cref{fig:filmstrips-arc0,fig:filmstrips-arc,fig:filmstrips-arc2}), the cells a task rewrites end furthest from their anchors under \methodname (mean AUC $0.81$), while FPRM's ordering varies across tasks (AUC $0.51$).

\textbf{The final angle predicts correctness.}
On unsolved Sudoku puzzles both \methodname cell classes end farther from their anchors (\Cref{fig:interpretab_angles}): the clues at ${\approx}65^\circ$ instead of ${\approx}52^\circ$, almost disjoint from the solved ones, and the blanks at ${\approx}92^\circ$ instead of ${\approx}88^\circ$. The converged angle thus flags a wrong answer without labels, shared by every cell and puzzle. FPRM's final norms separate the two outcomes far less: its clue and blank norms overlap widely.

\keybox{\textbf{Answer to RQ3.} \methodname's loop writes task structure into the state, and on the sphere it is legible as an angle to the input: the clue mask on Sudoku, the path on Maze, and, on Sudoku, whether the answer is correct. FPRM's angles and state norm carry this structure only faintly or partially.}

\begin{figure}[t]
  \centering
  \begin{subfigure}[t]{0.49\linewidth}
    \centering
    \includegraphics[width=\linewidth]{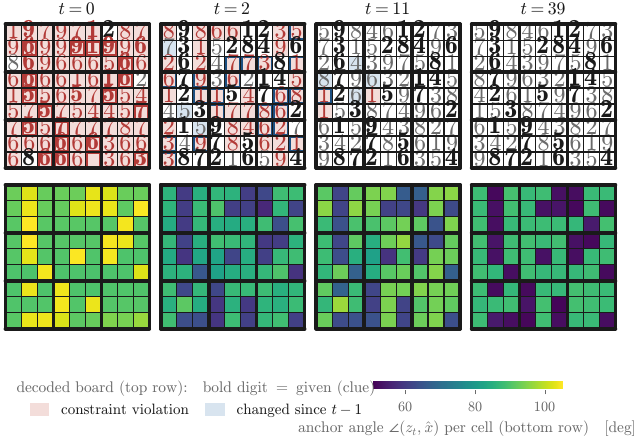}
  \end{subfigure}\hfill
  \begin{subfigure}[t]{0.49\linewidth}
    \centering
    \includegraphics[width=\linewidth]{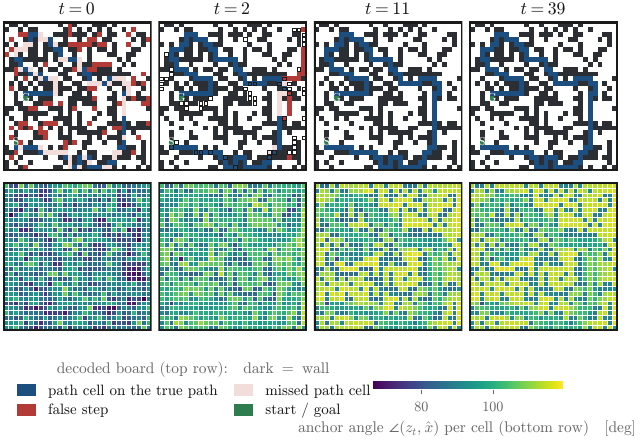}
  \end{subfigure}
  \caption{\textbf{\methodname rollout frames} (FPRM in \Cref{fig:filmstrips-sudoku,fig:filmstrips-maze}). \emph{Top row}: the board decoded from $\mathbf z_i$; \emph{bottom row}: each cell's angle to its own input anchor. \emph{left}: Sudoku-Extreme; \emph{right}: Maze-Hard.}
  \label{fig:filmstrip}
  \vspace{-2mm}
\end{figure}

\begin{figure}[t]
  \centering
  \begin{tabular}{@{}c@{\hspace{5pt}}c@{}}
    \dfp{\textbf{FPRM}} & \dsl{\textbf{\methodname}} \\[1pt]
    \includegraphics[width=.535\linewidth]{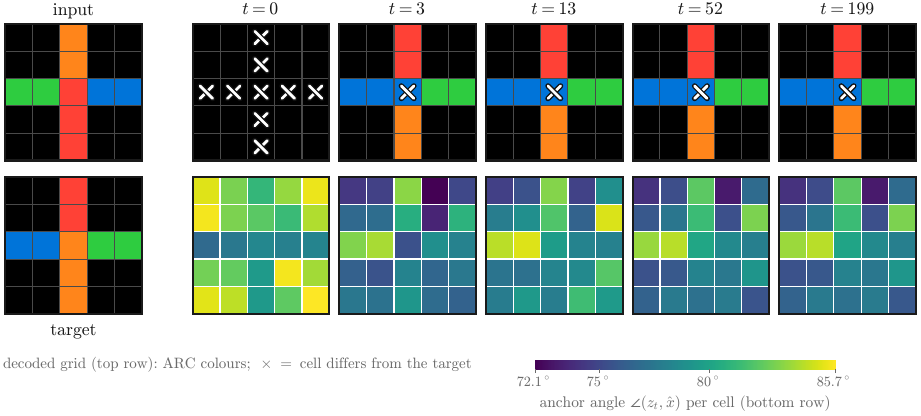} &
    \begin{minipage}[b]{.4334\linewidth}
      \adjincludegraphics[trim={{.19\width} {.20\height} 0 0}, clip, width=\linewidth]{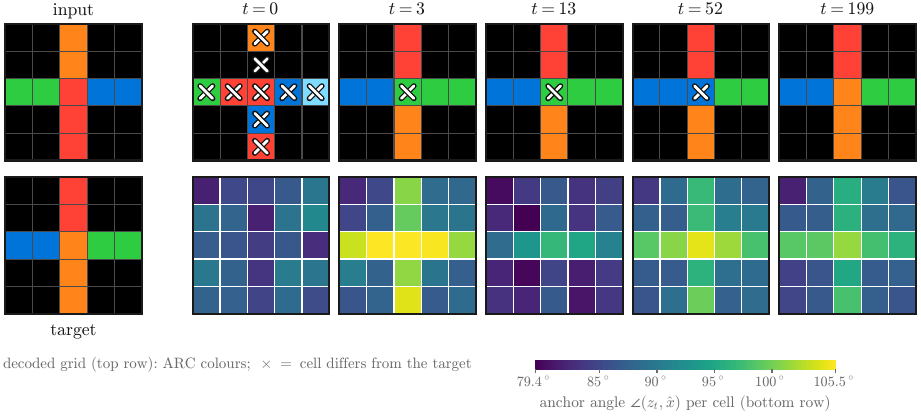}\par\nointerlineskip
      \hfill\adjincludegraphics[trim={{.55\width} 0 0 {.80\height}}, clip, width=.5556\linewidth]{figures/miniarc/filmstrips_more/sfprm/filmstrip_p9.pdf}
    \end{minipage}
  \end{tabular}
\vspace{-2mm}
\caption{\textbf{Rollout frames on Mini-ARC} for FPRM (left) and \methodname (right). Input (top) and target (bottom), then iterations $t=0$ to $199$: decoded grid ($\times$: wrong cell) and angle to the input anchor, $\angle(z_t,\hat x)$. Only \methodname solves this puzzle, and its angle map singles out the cells the task rewrites.}
\label{fig:filmstrips-arc0}
\vspace{-2mm}
\end{figure}

\section{Conclusion}
\label{sec:conclusion}

\begin{wrapfigure}[15]{r}{0.48\linewidth}
  \vspace{-8mm}
  \centering
  \includegraphics[width=\linewidth]{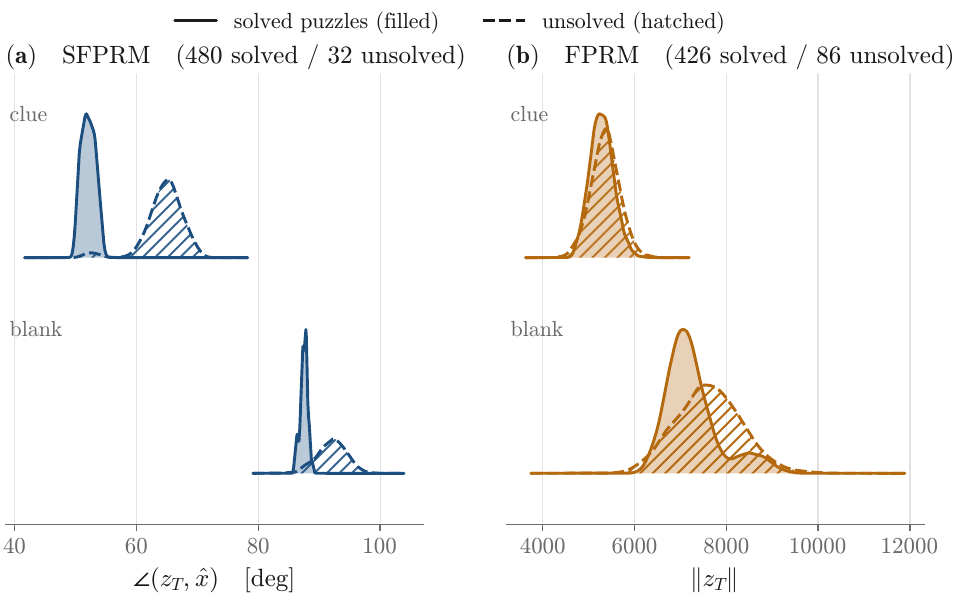}
  \caption{\textbf{Final anchor angle vs.\ correctness} on $512$ Sudoku puzzles: angle (\methodname) and state norm (FPRM) of clue and blank cells, for solved (filled) and unsolved (hatched) puzzles.}
  \label{fig:interpretab_angles}
\end{wrapfigure}
We introduced \methodname, which runs FPRM's damped fixed-point loop on nGPT's hypersphere. This change of geometry assures norm boundedness by construction and eliminates the need for FPRM's coupled residual-rescaling coefficients. We provide theoretical guarantees under weaker assumptions than its Euclidean counterpart.
Empirically, on puzzle and reasoning tasks, \methodname halts more effectively at equal or higher accuracy than FPRM. 
Finally, we showed that the angle each token keeps to its input turns out to be a readable encoding of the task's structure and, on Sudoku, a label-free predictor of correctness.
Natural next steps are to scale \methodname to language modeling and train with better solvers, which we have so far used only at inference.

\newpage
\clearpage

\textbf{Reproducibility.}
We provide the experimental protocol in Appendix \ref{app:exp-details} and in \Cref{tab:slfpr-hparams} a list of the used hyperparameters in our experiments. The inference-time solver description and comparison from \Cref{sec:exp-halting} is provided in Appendix \ref{appsec:solvers}. We implemented \methodname on top of the public FPRM code base \citep{Movahedi2026FPRM}. The code is provided at \url{https://github.com/arberzela/sfprm.git}. All runs were trained using Nvidia A100 GPUs for a single seed each.

\textbf{LLM Usage.}
Running FPRM's damped fixed-point solver on nGPT's unit hypersphere, with geodesic injection, angular damping, and a chordal halting residual, was the authors' idea. For the theory, the authors provided proof strategies and guidance, while AI tools helped develop, write, and check parts of the proofs. The \methodname implementation on the FPRM code base was assisted by Claude and verified by the authors, who also produced the experimental design, analyses, and figures with LLM assistance. Claude helped restructure and edit the manuscript for length and readability.

\paragraph*{Acknowledgments.}
This work was supported by EPFL AI Center \& Swiss AI Initiative Postdoctoral Fellowship Programme. This work was supported as part of the Swiss AI Initiative, through a grant from the ETH Domain and computational resources provided by the Swiss National Supercomputing Centre (CSCS) under the Alps infrastructure.

\bibliography{iclr2027_conference}
\bibliographystyle{iclr2027_conference}

\clearpage
\newpage

\appendix

\section{Proofs for the \methodname guarantees}
\label{app:proofs}

Throughout this appendix, the already token-wise normalized input $\mathbf x\in\mathcal S$ is
fixed and
\[
  \mathcal S=(\sph^{d-1})^T,
  \qquad
  \angle(u,v)
  :=\arccosop(\inner{u}{v}),
  \qquad u,v\in\sph^{d-1}.
\]
The integer $T$ is the number of tokens. Every spherical operation is applied independently to
the $T$ token vectors. In particular, we use the sequence-level convention from
\cref{eq:sequence-slerp}. All operands passed to $\Norm$ are assumed nonzero. Exact
$\Slerp$ means interpolation along the unique shortest arc, so its endpoints are assumed
non-antipodal. Numerical guards require separate care because an arbitrary antipodal fallback
need not be continuous.

\subsection{Basic geometry of spherical interpolation}

\begin{lemma}[Geometry of \textsc{slerp}]
\label{lem:slerp-geometry}
Let $u,v\in\sph^{d-1}$ with $v\ne-u$, and let $s\in[0,1]$. Then
\begin{align*}
  \norm{\Slerp_s(u,v)}_2&=1, \\
  \angle\!\left(u,\Slerp_s(u,v)\right)
    &=s\,\angle(u,v), \\
  \angle\!\left(\Slerp_s(u,v),v\right)
    &=(1-s)\angle(u,v).
\end{align*}
Moreover, if $\Omega:=\angle(u,v)>0$, then
\[
  \Slerp_s(u,v)=\cos(s\Omega)\,u+\sin(s\Omega)\,e,
  \qquad
  e:=\frac{v-\cos\Omega\,u}{\sin\Omega},
\]
where $e$ is a unit vector orthogonal to $u$.
\end{lemma}

\begin{proof}
Set $\Omega=\angle(u,v)$. If $\Omega=0$, then $v=u$ and
$\Slerp_s(u,u)=u$ by definition. Its norm is one and both angles are zero, proving
all three claims in this case.
For $0<\Omega<\pi$, define
\[
  e
  :=\frac{v-\cos\Omega\,u}{\sin\Omega}.
\]
Then $\inner{u}{e}=0$, $\norm{e}_2=1$, and
$v=\cos\Omega\,u+\sin\Omega\,e$. Substituting this identity into the
\textsc{slerp} formula and using the sine addition rule gives
\[
  \Slerp_s(u,v)
  =\cos(s\Omega)u+\sin(s\Omega)e.
\]
The two terms are orthogonal, so the result has unit norm. The same expression places it at angle
$s\Omega$ from $u$ and $(1-s)\Omega$ from $v$.
\end{proof}

\subsection{Guarded interpolation}
\label{sec:app-slerp}

The implementation replaces exact \textsc{slerp} by a guarded step. The sine formula of
\cref{eq:slerp} gives $0/0$ at $\Omega=0$ and is ill conditioned nearby, especially in bf16,
although $\Slerp_\eta(u,u)=u$; at $\Omega=\pi$ the shortest arc is not unique. Write
$c=\inner{u}{v}$ and $\nlerp_\eta(u,v):=\Norm\big((1-\eta)u+\eta v\big)$, and near an antipode fix
an explicit arc through a unit vector $e(u,v)\perp u$, namely $e(u,v)=\Norm(P_u v)$ with
$P_u=I-uu^\top$ when $\norm{P_u v}_2>\epsilon_v$ and a fixed unit vector $e_0(u)\perp u$ otherwise.
The guarded interpolation is
\begin{equation}
  \label{eq:gslerp}
  \gSlerp_\eta(u,v)=
  \begin{cases}
    \nlerp_\eta(u,v), & c>1-\epsilon_\Omega,\\[2pt]
    \operatorname{Arc}_\eta(u,v):=\cos(\eta\Omega)\,u+\sin(\eta\Omega)\,e(u,v), & c<-1+\epsilon_\Omega,\\[2pt]
    \Slerp_\eta(u,v), & \text{otherwise.}
  \end{cases}
\end{equation}
The last branch is the exact step, and the middle one makes the unavoidable choice of arc near an
antipode explicit. The first branch is what the solver executes in the convergence regime, and the
following lemma is what the boundedness argument of \Cref{sec:guarantees} needs from it.

\begin{lemma}[The guarded step stays on the sphere and on the arc]
\label{lem:gslerp}
Let $u,v\in\sph^{d-1}$ and $\eta\in[0,1]$. Every branch of \cref{eq:gslerp} returns a unit vector.
On the first branch, $\nlerp_\eta(u,v)$ lies on the same great circle as $\Slerp_\eta(u,v)$, at
angle $\theta=\eta\Omega+O(\Omega^3)$ from $u$, where $\Omega=\angle(u,v)$.
\end{lemma}

\begin{proof}
The first branch normalizes explicitly, the third has unit norm by \Cref{lem:slerp-geometry}, and
the middle branch is $\cos(\eta\Omega)u+\sin(\eta\Omega)e$ with $e\perp u$ of unit norm. For the
first branch, $(1-\eta)u+\eta v$ lies in $\mathrm{span}\{u,v\}$, so normalizing it lands on the
great circle through $u$ and $v$: the guard can change the position along the arc, not the arc.
Its angle $\theta$ from $u$ satisfies
\[
  \cos\theta
  =\frac{(1-\eta)+\eta\cos\Omega}{\norm{(1-\eta)u+\eta v}_2}
  =\frac{1-\eta\Omega^2/2+O(\Omega^4)}{1-\eta(1-\eta)\Omega^2/2+O(\Omega^4)}
  =1-\frac{\eta^2\Omega^2}{2}+O(\Omega^4),
\]
using $\cos\Omega=1-\Omega^2/2+O(\Omega^4)$ and
$\norm{(1-\eta)u+\eta v}_2^2=1-2\eta(1-\eta)(1-\cos\Omega)$. Hence $\theta=\eta\Omega+O(\Omega^3)$:
the guard's error is a third-order timing error along the correct arc. The guard fires for
$\Omega\lesssim\sqrt{2\epsilon_\Omega}$, where this discrepancy is of order $\epsilon_\Omega^{3/2}$,
while the $0/0$ is removed.
\end{proof}

\subsection{Euclidean mixing follows the same geodesic}
\label{app:proof-euclidean-injection}

This is the identity behind the geodesic injection of \cref{eq:slfpr-map}: FPRM's affine input
mixing, once normalized, moves along the same arc with a state-dependent fraction.

\begin{lemma}[Normalized Euclidean mixing follows the same geodesic]
\label{lem:fprm-slerp}
Let $u,v\in\reals^d\setminus\{0\}$, write $\widehat u=u/\norm{u}_2$ and
$\widehat v=v/\norm{v}_2$, and assume $\widehat v\ne-\widehat u$. For any $\alpha,\beta>0$, if
$\Omega=\angle(\widehat u,\widehat v)\in(0,\pi)$, then
\begin{equation}
  \label{eq:euclidean-slerp-equivalence}
  \begin{aligned}
  \Norm(\alpha u+\beta v)
  &=\Slerp_{\lambda_{\mathrm{eff}}}(\widehat u,\widehat v),\\
  \lambda_{\mathrm{eff}}
  &=\frac{1}{\Omega}\operatorname{atan2}\!\left(
    \beta\norm{v}_2\sin\Omega,\;
    \alpha\norm{u}_2+\beta\norm{v}_2\cos\Omega
  \right)\in(0,1).
  \end{aligned}
\end{equation}
If $\Omega=0$, then $\Norm(\alpha u+\beta v)=\widehat u=\widehat v$.
\end{lemma}

\begin{proof}
If $\Omega=0$, then $\widehat u=\widehat v$ and the positive mixture
$\alpha u+\beta v=(\alpha\norm{u}_2+\beta\norm{v}_2)\widehat u$ normalizes to
$\widehat u$. This is also $\Slerp_s(\widehat u,\widehat u)$ for every $s\in[0,1]$.

For $0<\Omega<\pi$, define
\[
  e
  :=\frac{\widehat{v}-\cos\Omega\,\widehat{u}}{\sin\Omega}.
\]
As in \Cref{lem:slerp-geometry}, $\widehat{u}$ and $e$ are orthonormal and
$\widehat{v}=\cos\Omega\,\widehat{u}+\sin\Omega\,e$. Hence
\[
  \alpha u+\beta v
  =\left(
    \alpha\norm{u}_2+\beta\norm{v}_2\cos\Omega
   \right)\widehat{u}
   +\beta\norm{v}_2\sin\Omega\,e.
\]
By the definition of $\operatorname{atan2}$ and $\lambda_{\mathrm{eff}}$ in the lemma,
the normalized mixture is therefore
\[
  \Norm(\alpha u+\beta v)
  =\cos(\lambda_{\mathrm{eff}}\Omega)\,\widehat{u}
   +\sin(\lambda_{\mathrm{eff}}\Omega)\,e.
\]
The first argument of $\operatorname{atan2}$ is positive, so
$0<\lambda_{\mathrm{eff}}\Omega<\pi$. Moreover,
\[
  \sin\!\left(\Omega-\lambda_{\mathrm{eff}}\Omega\right)
  =\frac{\alpha\norm{u}_2\sin\Omega}{\norm{\alpha u+\beta v}_2}>0.
\]
Since $\Omega-\lambda_{\mathrm{eff}}\Omega\in(-\pi,\pi)$, this implies
$0<\lambda_{\mathrm{eff}}\Omega<\Omega$, or $0<\lambda_{\mathrm{eff}}<1$.
The representation in \Cref{lem:slerp-geometry} now gives
\[
  \Norm(\alpha u+\beta v)
  =\Slerp_{\lambda_{\mathrm{eff}}}\!\left(\widehat{u},\widehat{v}\right),
\]
which is exactly \cref{eq:euclidean-slerp-equivalence}.

At $\Omega=\pi$, the shortest great-circle direction is not unique. If the two opposing weighted
vectors also have equal magnitudes, their Euclidean mixture is zero and cannot be normalized.
This explains why antipodal directions are excluded.
\end{proof}

\subsection{Detailed proof of fixed-point existence}
\label{app:proof-existence}

\begin{proof}[Proof of \Cref{prop:existence}]
Set $R:=(1-\lambda)\pi$. Since $0<\lambda<1$, we have $0<R<\pi$. Recall the cap
\[
  \mathcal C_{\mathbf x}
  =\left\{
    \mathbf y\in\mathcal S:
    \angle(y_t,x_t)\le R
    \text{ for every }t
  \right\}.
\]
The exact recurrent map $F_\theta$ does not generally map this cap into itself, because the network is
applied after the token-wise interpolation in \cref{eq:slfpr-map}. The proof therefore uses the
reversed order only in an auxiliary map:
\[
  H(\mathbf y)
  :=\Slerp_{\lambda}\!\left(f_\theta(\mathbf y),\mathbf x\right).
\]
For every $\mathbf y\in\mathcal C_{\mathbf x}$, the assumptions ensure that
$f_\theta(\mathbf y)_t$ is a unit vector different from $-x_t$. By
\Cref{lem:slerp-geometry},
\begin{align*}
  \angle\!\left(H(\mathbf y)_t,x_t\right)
  &=(1-\lambda)
    \angle\!\left(f_\theta(\mathbf y)_t,x_t\right) \\
  &<(1-\lambda)\pi=R.
\end{align*}
Thus $H(\mathcal C_{\mathbf x})\subseteq\mathcal C_{\mathbf x}$. Continuity of $f_\theta$ and
exclusion of the antipode make $H$ continuous throughout this cap.

We next give the Euclidean coordinates needed for Brouwer's theorem. The spherical logarithm
below is a vector-valued coordinate map, not a scalar logarithm. Given unit vectors
$u$ and $v\ne-u$, let
$\Omega=\angle(u,v)<\pi$ and define
\[
  \Log_{u}(v)
  :=
  \begin{cases}
    0, & v=u,\\[2pt]
    \displaystyle
    \frac{\Omega}{\sin\Omega}
    \left(v-\cos\Omega\,u\right),
      & v\ne u.
  \end{cases}
\]
The vector in parentheses is orthogonal to $u$ and has norm $\sin\Omega$. Therefore
\[
  \inner{\Log_{u}(v)}{u}=0,
  \qquad
  \norm{\Log_{u}(v)}_2=\Omega.
\]
In words, $\Log_{u}(v)$ points in the tangent direction of the shortest arc from $u$ to
$v$, and its length equals that arc's angle. Its inverse is
\[
  \Exp_{u}(w)
  :=
  \begin{cases}
    u, & w=0,\\[2pt]
    \displaystyle
    \cos(\norm{w}_2)u
    +\sin(\norm{w}_2)
      \frac{w}{\norm{w}_2},
      & w\ne0,
  \end{cases}
\]
for $\inner{w}{u}=0$ and $\norm{w}_2<\pi$. The two maps are continuous
inverses on every closed cap of radius strictly smaller than $\pi$.

Applying the spherical logarithm token by token identifies $\mathcal C_{\mathbf x}$ continuously
and bijectively with
\[
  \mathcal B
  :=\prod_{t=1}^{T}
  \left\{
    u_t\in\reals^d:
    \inner{u_t}{x_t}=0,
    \ \norm{u_t}_2\le R
  \right\}.
\]
Each factor is an ordinary closed Euclidean ball inside the hyperplane perpendicular to
$x_t$. Consequently, $\mathcal B$ is nonempty, compact, and convex and can be identified
with a subset of $\reals^{T(d-1)}$.

We now apply \emph{Brouwer's fixed-point theorem} \citep{Brouwer1911}, in the following
standard form: if $K$ is a nonempty, compact, convex subset of a finite-dimensional
Euclidean space and $G:K\to K$ is continuous, then
\[
  \text{there exists }y^\star\in K\text{ such that }G(y^\star)=y^\star.
\]
For a textbook statement, see \citet[Section~6.3, p.~161]{Ortega1970}.
Transforming $H$ into the logarithmic coordinates above produces just such a continuous
self-map of $\mathcal B$. Therefore
there exists $\mathbf y^\star\in\mathcal C_{\mathbf x}$ such that
\[
  \mathbf y^\star
  =H(\mathbf y^\star)
  =\Slerp_{\lambda}\!\left(f_\theta(\mathbf y^\star),\mathbf x\right).
\]
Now define
\[
  \mathbf z^\star:=f_\theta(\mathbf y^\star).
\]
The preceding fixed-point identity gives
$\mathbf y^\star=\Slerp_{\lambda}(\mathbf z^\star,\mathbf x)$, and hence
\begin{align*}
  F_\theta(\mathbf z^\star;\mathbf x)
  &=f_\theta\!\left(\Slerp_{\lambda}(\mathbf z^\star,\mathbf x)\right) \\
  &=f_\theta(\mathbf y^\star)
   =\mathbf z^\star.
\end{align*}
Thus $\mathbf z^\star$ is a fixed point of the exact injection-first recurrent map. Notice that the
proof places $\mathbf y^\star$, not necessarily $\mathbf z^\star$, in the input-centered cap.
\end{proof}

\subsection{The first derivative of \textsc{slerp} at coincident endpoints}

\begin{lemma}[First-order interpolation rule]
\label{lem:slerp-linearization}
Fix $0\le\eta\le1$ and $a\in\sph^{d-1}$. As unit vectors $u,v$ approach $a$,
\[
  \Slerp_\eta(u,v)
  =(1-\eta)u+\eta v+O(\norm{u-v}_2^2),
\]
with the exact value $\Slerp_\eta(u,u)=u$ at coincidence. In particular, the
interpolation is jointly differentiable at $(a,a)$, with tangent derivative
\[
  D\Slerp_\eta(a,a)[h,k]=(1-\eta)h+\eta k,
  \qquad \inner{a}{h}=\inner{a}{k}=0.
\]
\end{lemma}

\begin{proof}
For $\Omega=\angle(u,v)>0$ small, the elementary estimate
$\sin(c\Omega)=c\Omega+O(\Omega^3)$, for fixed $c$, gives
\[
  \frac{\sin((1-\eta)\Omega)}{\sin\Omega}
  =(1-\eta)+O(\Omega^2),
  \qquad
  \frac{\sin(\eta\Omega)}{\sin\Omega}
  =\eta+O(\Omega^2).
\]
Substituting into the interpolation formula gives an error of order
$\Omega^2$, since both endpoints have unit norm. The chord identity
$\norm{u-v}_2=2\sin(\Omega/2)$ implies
$\Omega=O(\norm{u-v}_2)$, proving the stated estimate. If $u=v$, the error
is exactly zero because $\Slerp_\eta(u,u)=u$.

To obtain the joint derivative, express nearby endpoints using tangent
coordinates $h,k\perp a$:
\[
  u=\frac{a+h}{\norm{a+h}_2}=a+h+O(\norm{h}_2^2),
  \qquad
  v=\frac{a+k}{\norm{a+k}_2}=a+k+O(\norm{k}_2^2).
\]
Here $\norm{a+h}_2^2=1+\norm{h}_2^2$, so normalization has no first-order
effect. The uniform interpolation estimate now yields
\[
  \Slerp_\eta(u,v)
  =a+(1-\eta)h+\eta k
   +O(\norm{h}_2^2+\norm{k}_2^2).
\]
Thus the linear response to arbitrary simultaneous endpoint perturbations
is $(1-\eta)h+\eta k$, as claimed.
\end{proof}

\subsection{Proof of local convergence under angular damping}
\label{app:proof-convergence}

\begin{proof}[Proof of \Cref{prop:converge}]
\textbf{Tangent directions and fixed points.}
Let $J=D_{\mathbf z}F_\theta(\mathbf z^\star;\mathbf x)$. Consider a
differentiable spherical curve $\mathbf z(\varepsilon)$ with
$\mathbf z(0)=\mathbf z^\star$ and $\mathbf z'(0)=\mathbf u$.
Both its input and output tokens have unit norm, so differentiation gives
\begin{align*}
  0&=\left.\frac{d}{d\varepsilon}\norm{z_t(\varepsilon)}_2^2
       \right|_{\varepsilon=0}
    =2\inner{z_t^\star}{u_t},\\
  0&=\left.\frac{d}{d\varepsilon}
       \norm{F_\theta(\mathbf z(\varepsilon);\mathbf x)_t}_2^2
       \right|_{\varepsilon=0}
    =2\inner{z_t^\star}{(J\mathbf u)_t}.
\end{align*}
The output derivative uses $F_\theta(\mathbf z^\star;\mathbf x)=\mathbf z^\star$.
Hence $J$ sends tangent perturbations to tangent perturbations at the same
fixed point; it includes all coupling between tokens.

For $\eta>0$, interpolation moves a positive fraction of the shortest arc,
so it returns its first endpoint exactly when the endpoints coincide.
The damped update therefore has exactly the fixed points of $F_\theta$,
wherever these arcs are defined. At $\mathbf z^\star$ the endpoints coincide,
and continuity ensures that nearby endpoint pairs are non-antipodal.

\paragraph{Linearization and the damping range.}
Represent a nearby state by tangent coordinates $\mathbf w=(w_t)_{t=1}^T$:
\[
  z_t=\frac{z_t^\star+w_t}{\norm{z_t^\star+w_t}_2},
  \qquad \inner{z_t^\star}{w_t}=0.
\]
These coordinates measure displacement along the sphere to first order:
$\mathbf z=\mathbf z^\star+\mathbf w+O(\norm{\mathbf w}^2)$.
The inverse is
$w_t=z_t/\inner{z_t^\star}{z_t}-z_t^\star$, defined nearby, and also has
identity tangent derivative. Differentiability of $F_\theta$ gives
\[
  F_\theta(\mathbf z;\mathbf x)
  =\mathbf z^\star+J\mathbf w+o(\norm{\mathbf w}).
\]
In particular, the two interpolation endpoints differ by
$O(\norm{\mathbf w})$. By \Cref{lem:slerp-linearization}, replacing their
spherical interpolation by their weighted average incurs only
$O(\norm{\mathbf w}^2)$ error. Thus, writing $\Phi_\eta$ for the damped
update in these coordinates,
\[
  \Phi_\eta(\mathbf w)=A\mathbf w+o(\norm{\mathbf w}),
  \qquad A:=(1-\eta)\operatorname{Id}+\eta J.
\]

If $Jv=\zeta v$ for a nonzero, possibly complex eigenvector $v$, then
$Av=[1+\eta(\zeta-1)]v$. Conversely, every eigenvector of $A$ is one of
$J=(A-(1-\eta)\operatorname{Id})/\eta$. This identifies all eigenvalues
without assuming diagonalizability. As in Euclidean damping
\citep[Appendix~A.3]{Movahedi2026FPRM},
\[
  |1+\eta(\zeta-1)|^2
  =1+2\eta(\Real(\zeta)-1)+\eta^2|\zeta-1|^2<1
\]
exactly when $0<\eta<2(1-\Real(\zeta))/|\zeta-1|^2$.
The theorem's damping bound therefore places every eigenvalue of $A$
strictly inside the unit disk.

\paragraph{From the linearization to convergence.}
This eigenvalue condition need not make the Euclidean norm shrink at every
step. A standard finite-dimensional stability result instead provides a
vector norm $\norm{\cdot}_*$ with induced matrix norm
$a:=\norm{A}_*<1$ \citep{Ortega1970}. This norm is equivalent to the
Euclidean norm: positive constants $c_1,c_2$ satisfy
\[
  c_1\norm{\mathbf w}\le\norm{\mathbf w}_*
  \le c_2\norm{\mathbf w}.
\]
It therefore measures the same notion of convergence, up to fixed factors.

Choose $a<q<1$. The remainder $o(\norm{\mathbf w})$ means that its norm
divided by $\norm{\mathbf w}$ tends to zero. By norm equivalence, we may
choose a sufficiently small neighborhood where
\[
  \norm{\Phi_\eta(\mathbf w)-A\mathbf w}_*
  \le(q-a)\norm{\mathbf w}_*.
\]
The triangle inequality then gives the full one-step error bound:
\begin{align*}
  \norm{\Phi_\eta(\mathbf w)}_*
  &\le\norm{A\mathbf w}_*
      +\norm{\Phi_\eta(\mathbf w)-A\mathbf w}_*\\
  &\le a\norm{\mathbf w}_*+(q-a)\norm{\mathbf w}_*
   =q\norm{\mathbf w}_*.
\end{align*}
Because $q<1$, a sufficiently small ball in this norm stays within itself
after every update. Starting there, induction gives
$\norm{\mathbf w_i}_*\le q^i\norm{\mathbf w_0}_*$.
Finally, the coordinate map and its inverse are smooth nearby, so they
bound coordinate error and state error by constant multiples of each
other. Together with norm equivalence, this yields
\[
  \norm{\mathbf z_i-\mathbf z^\star}
  \le Cq^i\norm{\mathbf z_0-\mathbf z^\star}
\]
for some $C>0$, proving the asserted local geometric convergence.
\end{proof}

\paragraph{The guarded step.}
Near $\mathbf z^\star$ the endpoints $\mathbf z_i$ and $F_\theta(\mathbf z_i;\mathbf x)$ are
$O(\norm{\mathbf w})$ apart, so on a neighborhood of $\mathbf z^\star$ the solver executes the first
branch of \cref{eq:gslerp}. By \Cref{lem:gslerp} that branch lies on the same arc as the exact step,
at an angle from $u$ that differs by $O(\Omega^3)$, so it differs from $\Slerp_\eta(u,v)$ by
$O(\norm{u-v}_2^3)$ and shares the expansion $(1-\eta)u+\eta v+O(\norm{u-v}_2^2)$ of
\Cref{lem:slerp-linearization}. It also returns $u$ only when $v=u$, because for $\eta>0$ a
normalized nonnegative combination of two non-antipodal unit vectors equals $u$ only if $v=u$. The
proof above therefore applies verbatim, and \Cref{prop:converge} holds for the guarded step the
solver runs.

\begin{remark}[Optimal damping, and when the window is narrow]
\label{rem:damping}
Write $\zeta=a+bi$. The squared modulus $|1+\eta(\zeta-1)|^2=1-2\eta(1-a)+\eta^2|\zeta-1|^2$ is an
upward parabola in $\eta$ with value $1$ at $\eta=0$, so for $a<1$ the window
$0<\eta<2(1-a)/|\zeta-1|^2$ of \cref{eq:eta-stability-window} is exactly the interval on which it
lies below one, the vertex $\eta^\star=(1-a)/|\zeta-1|^2$ is the optimal damping for that mode, and
the smallest attainable squared modulus is $1-(1-a)^2/|\zeta-1|^2$, a contraction factor of
$\sqrt{1-(1-a)^2/|\zeta-1|^2}$ per step. Three consequences. \emph{(i)} Damping
handles rotation but not expansion: for $a\ge1$ every term is nonnegative, so no $\eta>0$ contracts
the mode and only training can remove it. \emph{(ii)} The guarantee is weak when $1-a$ is small compared with $|b|$, that is,
when $\zeta-1$ is nearly imaginary, whatever $|\zeta|$ is: for $\zeta=0.9+1.5i$ the undamped mode grows by $|\zeta|\approx1.75$ per step while
rotating by about $59^\circ$, the window is $0<\eta<0.089$, and even the optimal damping contracts
by only $0.22\%$ per step, so halving the error takes about $313$ iterations. Bounding the
rotation magnitude, not only the expansion, is therefore what makes halting fast. \emph{(iii)} The
patience rule of \Cref{alg:slfpr-simple}, which starts at $\eta=1$ and multiplies it by $\gamma$
whenever the residual stalls, is a derivative-free descent into this window: well-conditioned
modes ($\zeta=0.5$ admits every $\eta<4$) converge at once, and the geometric decay eventually
drops $\eta$ below the threshold of the rotational modes, at which point their tangent error
contracts geometrically and, whenever the residual then improves within $P$ iterations, the decay
pauses, without any spectrum ever being estimated. Because the contraction is in an adapted norm,
the residual need not fall monotonically, so the rule may keep shrinking $\eta$; this slows
convergence but cannot break it, since every smaller $\eta$ stays inside the window.
\end{remark}

\subsection{Global convergence under sufficiently strong input injection}
\label{app:proof-global-convergence}

\begin{proof}[Proof of \Cref{thm:global-convergence}]
We use the token-wise logarithm and exponential from
Appendix \ref{app:proof-existence}. They satisfy
\[
  \Slerp_\lambda(\mathbf z,\mathbf x)
  =\Exp_{\mathbf x}\!\left((1-\lambda)\Log_{\mathbf x}(\mathbf z)\right)
\]
whenever $z_t\ne-x_t$ for every token $t$. Indeed, \cref{eq:slerp} gives
$\Slerp_\lambda(z_t,x_t)=\Slerp_{1-\lambda}(x_t,z_t)$, which by the representation in
\Cref{lem:slerp-geometry} equals $\cos((1-\lambda)\Omega)\,x_t+\sin((1-\lambda)\Omega)\,e$ with
$\Omega=\angle(x_t,z_t)$ and $\Log_{x_t}(z_t)=\Omega e$, and this is
$\Exp_{x_t}((1-\lambda)\Log_{x_t}(z_t))$ by definition, also when $\Omega=0$.
Thus input injection multiplies the logarithmic coordinates by exactly
$1-\lambda$: it preserves the direction toward the current state while
reducing its angular distance from the input.

The geometric difficulty is that shrinking angles to the input does not
give a global contraction bound between states. Two states approaching an
input antipode along different arcs can be arbitrarily close while their
injected states remain separated for fixed $0<\lambda<1$. Thus the direct
global Lipschitz argument for Euclidean input mixing in FPRM
\citep[Appendix~A.2]{Movahedi2026FPRM} does not apply here.
Instead, every injected token has angle less than $(1-\lambda)\pi$ to its
input. We use this uniform bound to place all injected states in one fixed
neighborhood, prove contraction in its logarithmic coordinates, and then
transfer convergence back to the original states.

These are smooth inverse coordinate maps for $\norm{w_t}_2<\pi$. The
spherical exponential is smooth, and its differential has singular values
$1$ in the radial direction and $\sin r/r$ in the transverse directions at
radius $0<r<\pi$, with identity differential at the origin. Its differential
is therefore invertible throughout this domain, so the inverse function
theorem gives smoothness of the logarithm.

\paragraph{A fixed neighborhood of the input.}
Since $f_\theta$ is $C^1$ near $\mathbf x$ and
$f_\theta(\mathbf x)_t\ne-x_t$ for every $t$, continuity allows us to choose
$0<r<\pi/2$ such that $f_\theta\circ\Exp_{\mathbf x}$ is $C^1$ and none of
its output tokens equals $-x_t$ on an open neighborhood of
\[
  \mathcal B_r
  :=\left\{\mathbf w\in(\reals^d)^T:
       \inner{w_t}{x_t}=0,\quad \norm{w_t}_2\le r
       \text{ for every }t\right\}.
\]
Here $\mathcal B_r$ is a compact, convex product of ordinary tangent-space
balls. Define, on this neighborhood,
\[
  A(\mathbf w)
  :=\Log_{\mathbf x}\!\left(f_\theta(\Exp_{\mathbf x}(\mathbf w))\right).
\]
Here $\mathbf w$ represents an injected state, and $A(\mathbf w)$ represents
the network's output in the same coordinates. Multiplication by
$1-\lambda$ then performs the next input injection.
Both $A$ and $f_\theta\circ\Exp_{\mathbf x}$ are $C^1$, so their derivatives
are bounded on the compact set $\mathcal B_r$. Integrating these derivatives
along the line segment between any $\mathbf w,\mathbf v\in\mathcal B_r$
gives a common finite constant $L\ge1$ such that
\begin{align*}
  \norm{A(\mathbf w)-A(\mathbf v)}
  &\le L\norm{\mathbf w-\mathbf v},\\
  \norm{f_\theta(\Exp_{\mathbf x}(\mathbf w))
        -f_\theta(\Exp_{\mathbf x}(\mathbf v))}
  &\le L\norm{\mathbf w-\mathbf v}.
\end{align*}
All norms here are the product Euclidean norm defined in the main text.
Moreover, $\norm{A(\mathbf w)_t}_2<\pi$ for every token, because the
logarithm's length is the angle to $x_t$ and the antipode is excluded.

\paragraph{Strong injection gives a contraction.}
Choose
\[
  \lambda_0
  :=1-\min\!\left\{\frac r\pi,\frac1{2L}\right\}\in(0,1).
\]
Fix any $\lambda\in(\lambda_0,1)$, and set
$a:=1-\lambda$ and $q:=aL$. Then $a\pi<r$ and $0<q<1/2$.
The two terms in the threshold serve separate purposes: $a\pi<r$ keeps
all injected states in the chosen neighborhood, while $aL<1$ makes their
distances contract. The factor $1/2$ simply leaves a strict margin below one.
The map $\mathbf w\mapsto aA(\mathbf w)$ sends $\mathcal B_r$ into itself,
since
\[
  \norm{aA(\mathbf w)_t}_2<a\pi<r.
\]
It is also a contraction:
\[
  \norm{aA(\mathbf w)-aA(\mathbf v)}
  \le q\norm{\mathbf w-\mathbf v}.
\]
The set $\mathcal B_r$ is closed in a finite-dimensional Euclidean space,
hence complete. Banach's contraction theorem therefore gives a unique
$\mathbf w^\star\in\mathcal B_r$ satisfying
$\mathbf w^\star=aA(\mathbf w^\star)$, and every iteration
$\mathbf w_{i+1}=aA(\mathbf w_i)$ starting in $\mathcal B_r$ satisfies
\[
  \norm{\mathbf w_i-\mathbf w^\star}
  \le q^i\norm{\mathbf w_0-\mathbf w^\star}.
\]

\paragraph{Return to the original recurrence.}
Let $\mathbf z_0$ be any admissible initialization, with
$z_{0,t}\ne-x_t$ for every $t$, and set
\[
  \mathbf w_0:=a\Log_{\mathbf x}(\mathbf z_0).
\]
Because $\norm{w_{0,t}}_2<a\pi<r$, this initial coordinate lies in
$\mathcal B_r$, regardless of the distance from $\mathbf z_0$ to a fixed
point. The original recurrence and the coordinate iteration are related by
\[
  \mathbf z_{i+1}
  =f_\theta(\Exp_{\mathbf x}(\mathbf w_i)),
  \qquad
  \mathbf w_{i+1}
  =a\Log_{\mathbf x}(\mathbf z_{i+1})
  =aA(\mathbf w_i).
\]
These identities hold inductively for every $i\ge0$: each
$\mathbf w_i$ remains in $\mathcal B_r$, where the network's outputs avoid
the antipodes. Hence every subsequent interpolation is well-defined.

Define $\mathbf z^\star:=f_\theta(\Exp_{\mathbf x}(\mathbf w^\star))$.
The identity $\mathbf w^\star=aA(\mathbf w^\star)$ implies
$\mathbf w^\star=a\Log_{\mathbf x}(\mathbf z^\star)$, so
\[
  F_\theta(\mathbf z^\star;\mathbf x)
  =f_\theta(\Exp_{\mathbf x}(\mathbf w^\star))
  =\mathbf z^\star.
\]
This fixed point is admissible by the choice of $r$. Conversely, any
admissible fixed point $\mathbf z$ of $F_\theta$ gives
$\mathbf w=a\Log_{\mathbf x}(\mathbf z)\in\mathcal B_r$ with
$\mathbf w=aA(\mathbf w)$. Uniqueness of $\mathbf w^\star$ then implies
$\mathbf z=\mathbf z^\star$.

Finally, the Lipschitz bound above transfers convergence to the original
states:
\begin{align*}
  \norm{\mathbf z_{i+1}-\mathbf z^\star}
  &\le L\norm{\mathbf w_i-\mathbf w^\star}\\
  &\le Lq^i\norm{\mathbf w_0-\mathbf w^\star}
   \le 2Lr\sqrt T\,q^i,
  \qquad i\ge0.
\end{align*}
Thus the theorem holds with $C=2Lr\sqrt T$, uniformly over all admissible
initializations. The neighborhood, derivative bound, and injection threshold
depend on the fixed network and input.
\end{proof}

\section{Details on Experimental Setup}
\label{app:exp-details}

This appendix section provides the details on the experimental setup and experiments in \Cref{sec:experiments}. Details on tasks and data are in
Appendix \ref{app:exp-tasks}, on models in Appendix \ref{app:exp-models}, on training and solver settings in Appendix \ref{app:exp-training}. \Cref{tab:slfpr-hparams} reports all hyperparameters. The evaluation protocol is in Appendix \ref{app:exp-eval}.

\subsection{Tasks and data}
\label{app:exp-tasks}

\textbf{Sudoku-Extreme.}
This dataset consists of $9{\times}9$ Sudoku puzzles selected for difficulty, with a test split of $422{,}786$ puzzles. The standard training build has $1{,}000$ base puzzles, each augmented $1000\times$ by validity-preserving transformations (relabeling of the digits and permutations of rows, columns, bands, and stacks), giving $10^6$ training examples. We train instead on a $5$k$\times200$ build, $5{,}000$ base puzzles with $200$ augmentations each, which has the same number of training examples. Unless stated otherwise, every \methodname Sudoku run and every rollout of \Cref{sec:exp-halting,sec:exp-geometry} uses this build. A puzzle is a
sequence of $81$ cell tokens (blank or a digit) preceded by a learned per-puzzle embedding of $16$
tokens, and the model predicts the digit of every cell.

\textbf{Maze-Hard.}
The $30{\times}30$ maze benchmark of the same suite: $1{,}000$ training and $1{,}000$ test mazes, used without augmentation. A maze is a sequence of $900$ cell tokens (wall, free, start, goal) plus the puzzle embedding, and the model predicts for every cell whether it lies on the shortest path. 

\textbf{Mini-ARC.}
A $5{\times}5$-grid ARC variant \citep{kim2022playgrounds} with $149$ tasks. We train on the $528$
demonstration pairs augmented $100\times$ and report task-level Pass@1 on the $149$ test pairs,
voting over augmentations. A puzzle is $25$ cell tokens over a $12$-token vocabulary plus the
puzzle embedding.

\paragraph{ARC-AGI-1 and ARC-AGI-2.}
The abstraction-and-reasoning benchmarks of \citet{chollet2025arcprize2024technical}: few-shot tasks
in which a grid transformation must be inferred from demonstration pairs and applied to a test grid.
Following the HRM/TRM protocol, we train on $960$ (ARC-1) and $1{,}280$ (ARC-2) base puzzles, each
augmented $1000\times$ with the transformations of \citet{Wang2025HRM}, and report Pass@2 on the $400$
and $120$ evaluation tasks. Grids are padded to $30{\times}30$, so a puzzle is a sequence of $900$
cell tokens over a $12$-token vocabulary plus the puzzle embedding.

\subsection{Models}
\label{app:exp-models}

\methodname is implemented in the FPRM code base and shares an identical backbone with $6.83$M trainable parameters: hidden size $512$, $8$ attention heads, MLP expansion $4$, rotary position embeddings, and the configuration fields \texttt{H\_cycles}/\texttt{L\_cycles}/\texttt{L\_layers} $=3/6/2$ of the code base, run as a single weight-tied loop without hierarchy. The number of iterations backpropagated through, is the BPTT window $K=6$ of \Cref{sec:background}.

\textbf{FPRM.} FPRM uses a pre-norm block with the coupled residual-rescaling coefficients $(\alpha_1,\beta_1,\alpha_2,\beta_2)$ of \Cref{sec:background}, linear damping, and the relative residual $r_{\mathrm E}$ with tolerance $\tau_{\mathrm E}=0.1$.

\textbf{\methodname.} The spherical model replaces the block by the nGPT-style block of
\Cref{sec:slfpr-core}: QK-normalized attention and a SwiGLU MLP, each folded into the unit-norm
state by an eigen-learning-rate retraction, with all weight vectors renormalized after every
optimizer step. Three further choices are specific to \methodname. \emph{(i)} The input is injected by
the geodesic step of \cref{eq:slfpr-map} with a learned fraction $\lambda$, initialized at $0.1$.
\emph{(ii)} The readout is \emph{unnormalized}: the rows of the output head are not constrained to
unit norm, since with a unit-norm state that constraint would cap every logit at $\pm s_z$, and with
it the confidence of the predictions. \emph{(iii)} The solver is the
geodesic one of \Cref{alg:slfpr-simple}, with the chordal residual of \cref{eq:slfpr-residual} and
tolerance $\tau_c=0.01$, which corresponds to an angular tolerance of
$\tau_{\mathrm{ang}}=\arccos(0.99)\approx8.1^\circ$ per token.

\subsection{Training and solver settings}
\label{app:exp-training}

\Cref{tab:slfpr-hparams} lists every hyperparameter. Unless stated otherwise, \methodname trains with a constant learning rate after the warmup steps and an exponential moving average of the weights with rate $0.999$. We optimize \methodname's $2$-D weight matrices with Muon \citep{jordan2024muon,pethick2025training} at learning rate $5\times10^{-3}$ and momentum $0.95$, and its embeddings, normalization scales, and scalar parameters with AdamW at $10^{-3}$. Weight decay is $0.1$ for \methodname, and the per-puzzle embeddings have their own learning rate. \methodname additionally uses weight dropout and variational dropout of $0.1$ on the main build (variational dropout only on the $1$k$\times1000$ Sudoku run).
Following the HRM/FPRM code base, one \emph{epoch} draws one augmented example for every base puzzle; hence $12{,}000$ epochs on the $5$k$\times200$ build and $60{,}000$ epochs on the $1$k$\times1000$ build and on Maze-Hard all amount to the same $78{,}125$ optimizer steps at a global batch of $768$. For the final Sudoku results, we train for $15{,}000$ epochs.


\textbf{Deep supervision.} Training runs one segment of $K=6$ solver iterations per optimizer step and backpropagates through all $K$ geodesic steps of the segment; the recurrent state is carried across optimizer steps and detached between segments, and a sample is re-initialized once it halts or once it has been unrolled for \texttt{max\_iter} $=12$ iterations. The halting residual, patience counter, step-size schedule, and iterate bookkeeping are evaluated on detached states, so they steer the solver without entering the gradient.

\textbf{Solver.} Both geodesic steps of \methodname use the guarded interpolation $\gSlerp$ of \cref{eq:gslerp}. Per sample, the solver keeps the current step size $\eta$, the patience counter, the residual, and the residual state. As in FPRM, $\eta$ starts at $1$ and is multiplied by the decay factor $\gamma$ whenever the residual has not improved for $P=5$ consecutive iterations. The loop halts when the residual drops below tolerance, when $\eta$ falls below the floor $\eta_{\min}$, or when the iteration budget is exhausted.

\begin{table*}[t]
\centering
\small
\caption{\textbf{Hyperparameters and learned scalars of the \methodname runs in \Cref{sec:experiments}}. Sudoku-Extreme values are for the epoch-$15{,}000$ run, while values in parentheses are for the $12{,}000$-epoch constant-schedule run. All convolutions are depthwise and bias-free. For the per-channel step sizes $\alpha_A,\alpha_M$ and the per-dimension QK scale $s_{qk}$, cells give the mean for layer 1 / layer 2, and the range row gives the minimum and maximum over all channels of both layers. On ARC, a linear Q-head on a detached state copy only breaks ties in Pass@$K$.}
\setlength{\tabcolsep}{4pt}
\resizebox{\linewidth}{!}{
\begin{tabular}{@{}lcccc@{}}
\toprule
& Maze-Hard & Sudoku-Extreme & ARC-AGI-1 & Mini-ARC \\
\midrule
\multicolumn{5}{@{}l}{\emph{Data and schedule}} \\
base puzzles $\times$ augmentations & $1{,}000 \times 1$ & $5{,}000 \times 200$ & $960 \times 1{,}000$ & $149 \times 100$ \\
sequence length / vocabulary & $900$ / $6$ & $81$ / $11$ & $900$ / $12$ & $25$ / $12$ \\
test set & $1{,}000$ mazes & $422{,}786$ puzzles & $400$ tasks & $149$ tasks \\
epochs & $60{,}000$ & \shortstack{$15{,}000$ ($12{,}000$)} & $100{,}000$ & $100{,}000$ \\
optimizer steps & $78{,}125$ & \shortstack{$97{,}656$ ($78{,}125$)} & $518{,}226$ & $68{,}993$ \\
global batch size & \multicolumn{4}{c}{$768$} \\
\midrule
\multicolumn{5}{@{}l}{\emph{Optimization}} \\
optimizer & \multicolumn{4}{c}{Muon $+$ AdamW} \\
Muon lr / momentum & \multicolumn{3}{c}{$5\times10^{-3}$ / $0.95$} & $2.5\times10^{-3}$ / $0.95$ \\
AdamW lr & \multicolumn{4}{c}{$1\times10^{-3}$} \\
LR schedule & constant & \shortstack{cosine to $0.87\times$ (const.)} & constant & constant \\
warmup steps & $200$ & $200$ ($2{,}000$) & $0$ & $0$ \\
weight decay & \multicolumn{4}{c}{$0.1$} \\
puzzle-embedding lr / wd & $10^{-2}$ / $0$ & $10^{-3}$ / $0$ & $10^{-2}$ / $0$ & $10^{-2}$ / $0$ \\
EMA rate & \multicolumn{4}{c}{$0.999$} \\
\midrule
\multicolumn{5}{@{}l}{\emph{Backbone}} \\
hidden / heads / MLP expansion & \multicolumn{4}{c}{$512$ / $8$ / $4$} \\
\texttt{H\_cycles} / \texttt{L\_cycles} / \texttt{L\_layers} & \multicolumn{4}{c}{$3$ / $6$ / $2$} \\
BPTT window $K$ (\texttt{n\_back}) & \multicolumn{4}{c}{$6$} \\
puzzle-embedding length & $16$ & $16$ & $16$ & $1$ \\
positional encoding & \multicolumn{4}{c}{RoPE} \\
conv branch & 1-D causal, $k=4$ & none ($2$-D, $k=3$) & 1-D causal, $k=4$ & 1-D causal, $k=4$ \\
trainable parameters & \multicolumn{4}{c}{$6.83$M} \\
\midrule
\multicolumn{5}{@{}l}{\emph{Solver and regularization}} \\
halting tolerance $\tau$ & \multicolumn{4}{c}{$0.01$} \\
patience $P$ / floor $\eta_{\min}$ & \multicolumn{4}{c}{$5$ / $0.001$} \\
readout & \multicolumn{4}{c}{unnormalized head, scale $s_z$} \\
\texttt{max\_iter} (train / eval) & $24$ / $1{,}000$ & $12$ / $35{,}000$ & $8$ / $48$ & $8$ / $200$ \\
step-size decay $\gamma$ (train / eval) & $0.9$ / $0.99$ & $0.9$ / $0.99$ & $0.9$ / $0.9$ & $0.9$ / $0.9$ \\
Q-head loss weight & $0$ & $0$ & $0.5$ & $0.5$ \\
weight dropout & $0.1$ & $0.1$ & $0.05$ & $0.05$ \\
variational dropout & $0.1$ & $0.1$ & $0.1$ & $0.2$ \\
\midrule
\multicolumn{5}{@{}l}{\emph{Learned scalars at the evaluated checkpoint}} \\
input injection $\lambda$, init.\ $0.1$ & $0.575$ & $0.1$, fixed ($0.470$) & $0.334$ & $0.436$ \\
logit scale $s_z$, init.\ $16$ & $75.8$ & $88.9$ ($75.8$) & $206.0$ & $79.0$ \\
attention step size $\alpha_A$, init.\ $0.5$ & $0.50$ / $0.54$ & \shortstack{$0.41$ / $0.50$ \\ ($0.49$ / $0.42$)} & $0.44$ / $0.46$ & $0.41$ / $0.47$ \\
MLP step size $\alpha_M$, init.\ $0.5$ & $0.48$ / $0.58$ & \shortstack{$0.47$ / $0.47$ \\ ($0.48$ / $0.64$)} & $0.49$ / $0.67$ & $0.54$ / $0.71$ \\
QK scale $s_{qk}$, init.\ $1$ & $1.56$ / $1.40$ & \shortstack{$1.11$ / $0.95$ \\ ($0.86$ / $1.00$)} & $1.31$ / $1.53$ & $1.07$ / $1.06$ \\
\bottomrule
\end{tabular}
}
\label{tab:slfpr-hparams}
\vspace{-6mm}
\end{table*}

\subsection{Evaluation protocol}
\label{app:exp-eval}

All tasks are scored by exact accuracy (Pass@1 for Maze-Hard and Sudoku-Extreme and Pass@2 for ARC-1 and ARC-2): a puzzle counts only if every cell is correct, i.e.\ the whole board or the whole path. 
We report the final exact test accuracy, except that the \methodname run on the $1$k$\times1000$ Sudoku build is reported at both its best and its final evaluation because it collapses between them (\Cref{tab:settings}). The published rows of \Cref{tab:main-results} are taken from the respective papers.

\section{Additional Results, Ablations, and Analysis}
\label{secapp:ablations}

This section provides the results that complement \Cref{sec:experiments}: the Sudoku-Extreme training dataset ablation (Appendix \ref{app:build-ablation}), the Sudoku-Extreme and Maze-Hard counterparts of the convergence analyses of \Cref{sec:exp-halting} (Appendix \ref{app:conv-halting}), the learned scalars over training (Appendix \ref{secapp:lambda_scales}), additional results on the geometry of the computation (Appendix \ref{app:exp-rollouts}), and the study of alternative fixed-point solvers at inference (Appendix \ref{appsec:solvers}).

\subsection{Training-build ablation}
\label{app:build-ablation}

\Cref{tab:settings} compares \methodname with an FPRM control trained under the same budget on the two Sudoku-Extreme builds described in Appendix \ref{app:exp-tasks}. The $5$k$\times200$ does not seem to provide benefits for FPRM, however for \methodname it removes the collapse happening when training with the $1$k$\times1000$ build between its best and its final checkpoint.

\subsection{Convergence and halting}
\label{app:conv-halting}

\begin{wraptable}{R}{.56\linewidth}
\centering
\vspace{-4mm}
\caption{Exact accuracy on Sudoku-Extreme for \methodname vs. FPRM at matched backbone ($6.83$M trainable parameters), data, global batch ($768$), and optimizer steps/epochs ($78{,}125 / 12{,}000$). ``$1\text{k}\!\times\!1000$'' and ``$5\text{k}\!\times\!200$'' are the
number of distinct base puzzles times the augmentations per puzzle. The \methodname $1\text{k}$ run is given at both its best and
checkpoint because it collapses between them, which the $5\text{k}$ build fixes. For the $5\text{k}$ run the final checkpoint is also the best. The FPRM rows fall well below the $94.2\%$ reported for FPRM in \Cref{tab:main-results}.}
\label{tab:settings}
\small
\setlength{\tabcolsep}{6pt}
\resizebox{\linewidth}{!}{
\begin{tabular}{@{}llc@{}}
\toprule
Model & Setting & \shortstack{Sudoku-Ext.} \\
\midrule
FPRM  & Sudoku $1\text{k}\!\times\!1000$                 & 78.5\%  \\
FPRM  & Sudoku $5\text{k}\!\times\!200$                  & 78.4\% \\
\midrule
\methodname & Sudoku $1\text{k}\!\times\!1000$ (best checkpoint)  & 87.0\%  \\
\methodname & Sudoku $1\text{k}\!\times\!1000$ (final checkpoint) & 12.2\% \\
\methodname & Sudoku $5\text{k}\!\times\!200$          & \textbf{92.8\%} \\
\bottomrule
\end{tabular}
}
\end{wraptable}
For \Cref{sec:exp-halting,sec:exp-geometry} each trained model is rolled with halting disabled and without segment resets, in fp32, for a fixed budget of $N_{\max}=400$ iterations on $n=512$ uniformly sampled test sequences from Sudoku-Extreme and $N_{\max}=1000$ on the full Maze-Hard test data. At every iterate $\mathbf z_i$ the model's readout head decodes the board, and we record exact and per-token accuracy together
with the residual each solver tests: the worst-token chordal residual
$r_{\mathrm c,i}=\max_t\bigl(1-\inner{z_{i,t}}{\widetilde z_{i,t}}\bigr)$ with $\tau_c=0.01$ for \methodname, and the worst-token relative residual $r_{\mathrm E,i}=\max_t\norm{z_{i,t}-\widetilde z_{i,t}}_\infty/\norm{\widetilde z_{i,t}}_\infty$ with $\tau_{\mathrm E}=0.1$ for FPRM, where $\widetilde{\mathbf z}_i=F_\theta(\mathbf z_i;\mathbf x)$ is each model's proposal. A sequence counts as halted at the first iteration from which its residual stays below its own threshold. 
Halt iterations are reported per sequence (mean over the sequences that halt), whereas the deployed evaluator halts a whole batch on its worst sequence, so the per-sequence counts are idealized costs. Every FPRM rollout is generated using the released checkpoints from \citet{Movahedi2026FPRM}. 

\begin{figure}[t]
\centering
\begin{subfigure}[b]{.655\linewidth}
  \centering
  \includegraphics[width=.49\linewidth]{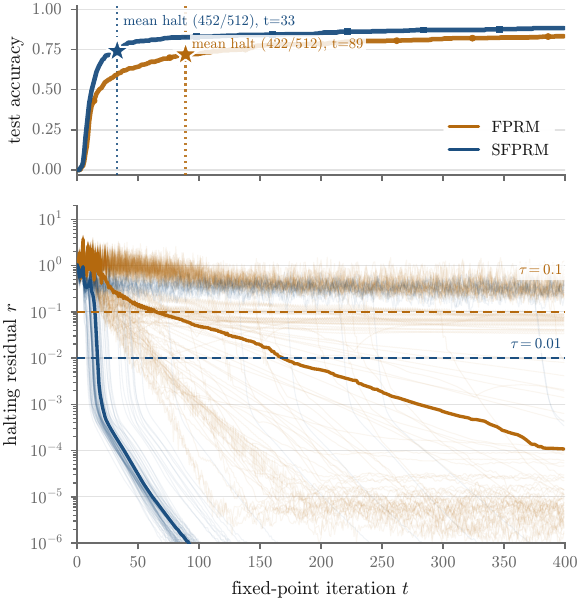}\hfill
  \includegraphics[width=.49\linewidth]{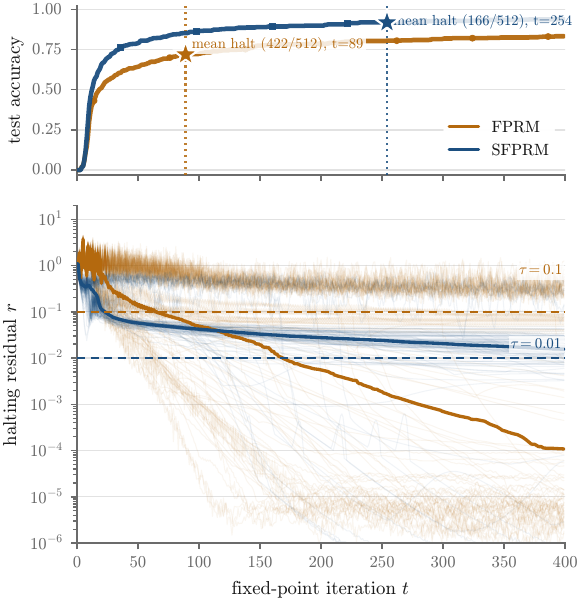}
  \caption{Sudoku-Extreme, learned $\lambda=0.47$ (left) and fixed $\lambda=0.1$ (right)}
  \label{fig:halting-sudoku-a}
\end{subfigure}\hfill
\begin{subfigure}[b]{.321\linewidth}
  \centering
  \includegraphics[width=\linewidth]{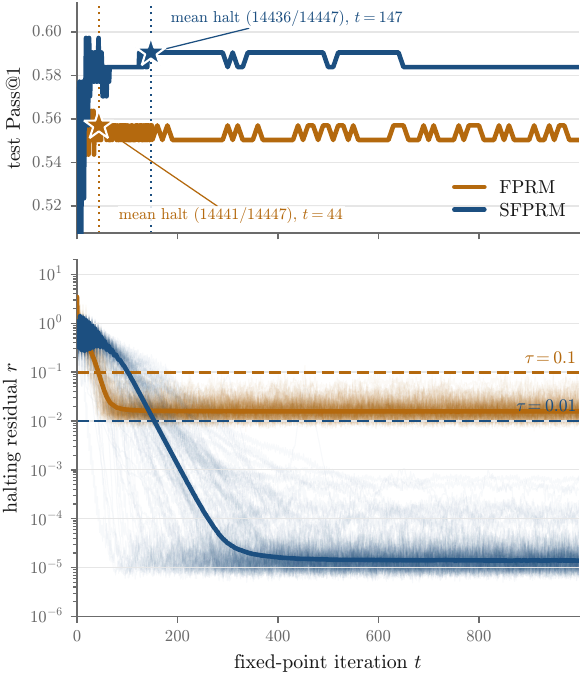}
  \caption{Mini-ARC}
  \label{fig:halting-miniarc}
\end{subfigure}
\caption{\textbf{Accuracy and halting on Sudoku-Extreme and Mini-ARC.} Open-loop rollouts of FPRM (amber) and \methodname (blue). (a) $512$ Sudoku test sequences of the released FPRM checkpoint and of \methodname with the learned $\lambda=0.47$ (left) and with $\lambda$ fixed at $0.1$ (right, the run in \Cref{tab:main-results}). (b) All $14{,}447$ Mini-ARC test rows ($149$ tasks, $100$ augmentations each) over $1000$ iterations, for the runs in \Cref{tab:main-results}. \emph{Top:} exact accuracy of the decoded state (a) and task-level Pass@1 (b). Dotted lines and stars mark each model's mean halt iteration, labelled with how many sequences halt. \emph{Bottom:} every sequence's halting residual (thin) and its median (heavy) against its own threshold ($\tau_c{=}0.01$ chordal for \methodname, $\tau_{\mathrm E}{=}0.1$ for FPRM).}
\label{fig:halting-sudoku}
\vspace{-6mm}
\end{figure}

\textbf{Halting.} \Cref{fig:halting-sudoku} shows the two Sudoku runs of \Cref{tab:slfpr-hparams} (one that reached $92.8\%$ and one $96.6\%$ exact accuracy) against FPRM. With the learned $\lambda=0.47$ (left), the median residual of \methodname drops below $\tau_c$ within a few dozen iterations, whereas FPRM's decays roughly geometrically over the whole window. $452$ of the $512$ sequences halt for \methodname, at a mean of $33$ iterations, against $422$ at a mean of $89$ for FPRM. Exact accuracy
is still rising after the mean halt for both models, so on this task the threshold fires before the
population has finished improving. With $\lambda$ fixed at $0.1$ (right) the residual of \methodname decays far more slowly and hovers near $\tau_c$, so only $166$ sequences halt within $400$ iterations, at a mean of $254$, although the decoded accuracy is higher. In both runs the puzzles a solver halts on are solved in $99.7$--$100\%$ of cases, so on Sudoku a halt certifies the answer, whereas a missing halt does not certify a failure.

\textbf{Fixed-point uniqueness.} \Cref{fig:global-conv-maze} repeats the global-convergence test of \Cref{fig:global-conv} on the Maze model (learned $\lambda=0.575$), running the raw map without damping for $1000$ steps. The contractive regime of \Cref{thm:global-convergence} exists here too, but it is narrower: all eight initializations reach the same fixed point only for $\lambda$ between about $0.6$ and $0.8$. The trained $\lambda$ sits at the lower edge of this band, where some mazes have a single end point and others do not, and below $\lambda\approx0.5$ no maze does. For $\lambda\ge0.9$ the end states spread again, while the token accuracy of the fixed point falls to $0.7$--$0.8$ as the input pulls the state onto itself. 
On the maze of panel (b) the raw map does not reach the tolerance at the trained $\lambda$ within $1000$ steps, although it decodes the maze exactly throughout; it is the damped solver that halts this model. At $\lambda=0.7$ the raw map converges inside the exact-match region, and at $\lambda=0.9$ the eight paths run straight to one point that no longer solves the maze.

\begin{figure}[t]
\centering
\includegraphics[width=0.23\linewidth]{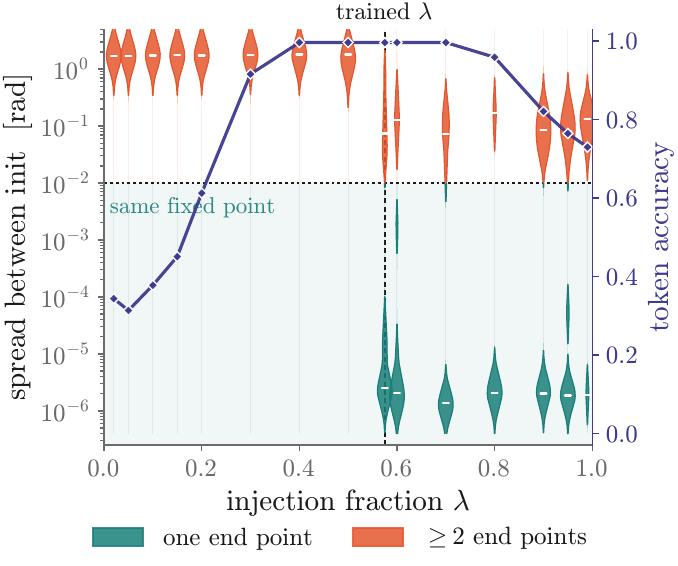}\hfill
\includegraphics[width=0.76\linewidth]{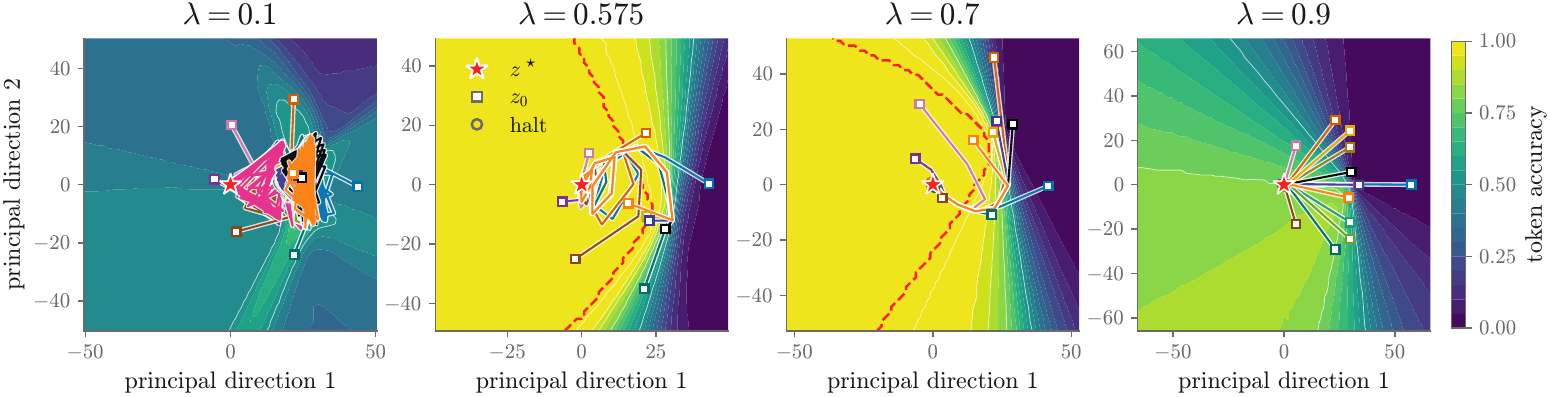}
\caption{\textbf{Global convergence under input injection on Maze-Hard} (\methodname, learned $\lambda=0.575$), analogue to \Cref{fig:global-conv}. \textit{(left)} The largest angle between the end states of eight initializations of the raw map (teal: all coincide, coral: $\ge2$ end states), and the token accuracy of the end state (indigo, right axis), as a function of $\lambda$. \textit{(right)} Trajectories at four values of $\lambda$ for a single Maze example: the top-2 principal plane of the trajectories from the eight initializations $\mathbf z_0$ (squares) through $\mathbf z^\star$ (star; circles: first iteration with $r_{\mathrm c,i}<\tau_c$). The color shows the token accuracy of the state decoded at each position (dashed red: exact-match region).}
\label{fig:global-conv-maze}
\end{figure}

\begin{figure}[t]
\centering
\begin{subfigure}{.49\linewidth}
  \centering
  \includegraphics[width=\linewidth]{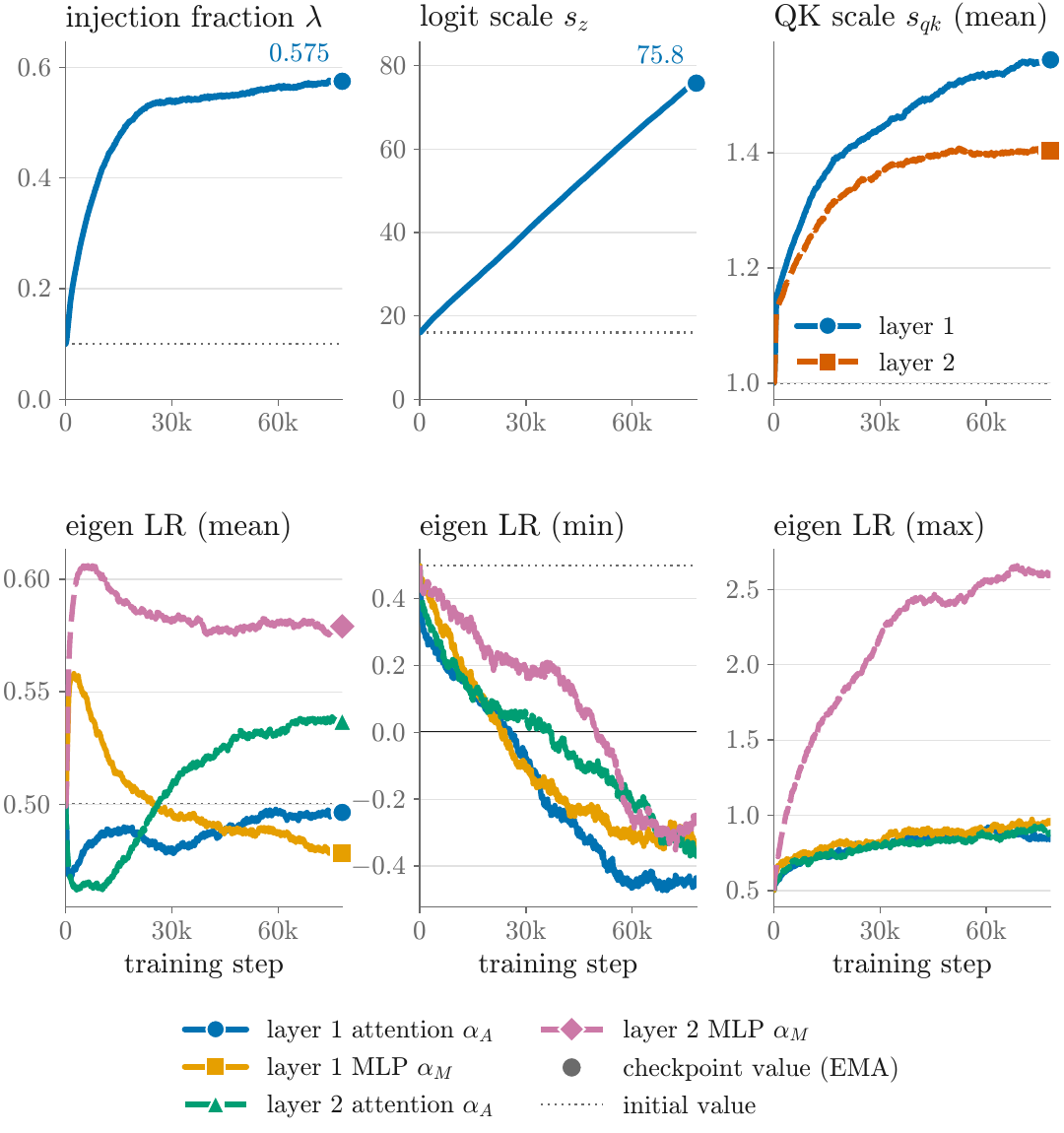}
  \caption{Maze-Hard}
\end{subfigure}\hfill 
\begin{subfigure}{.49\linewidth}
  \centering
  \includegraphics[width=\linewidth]{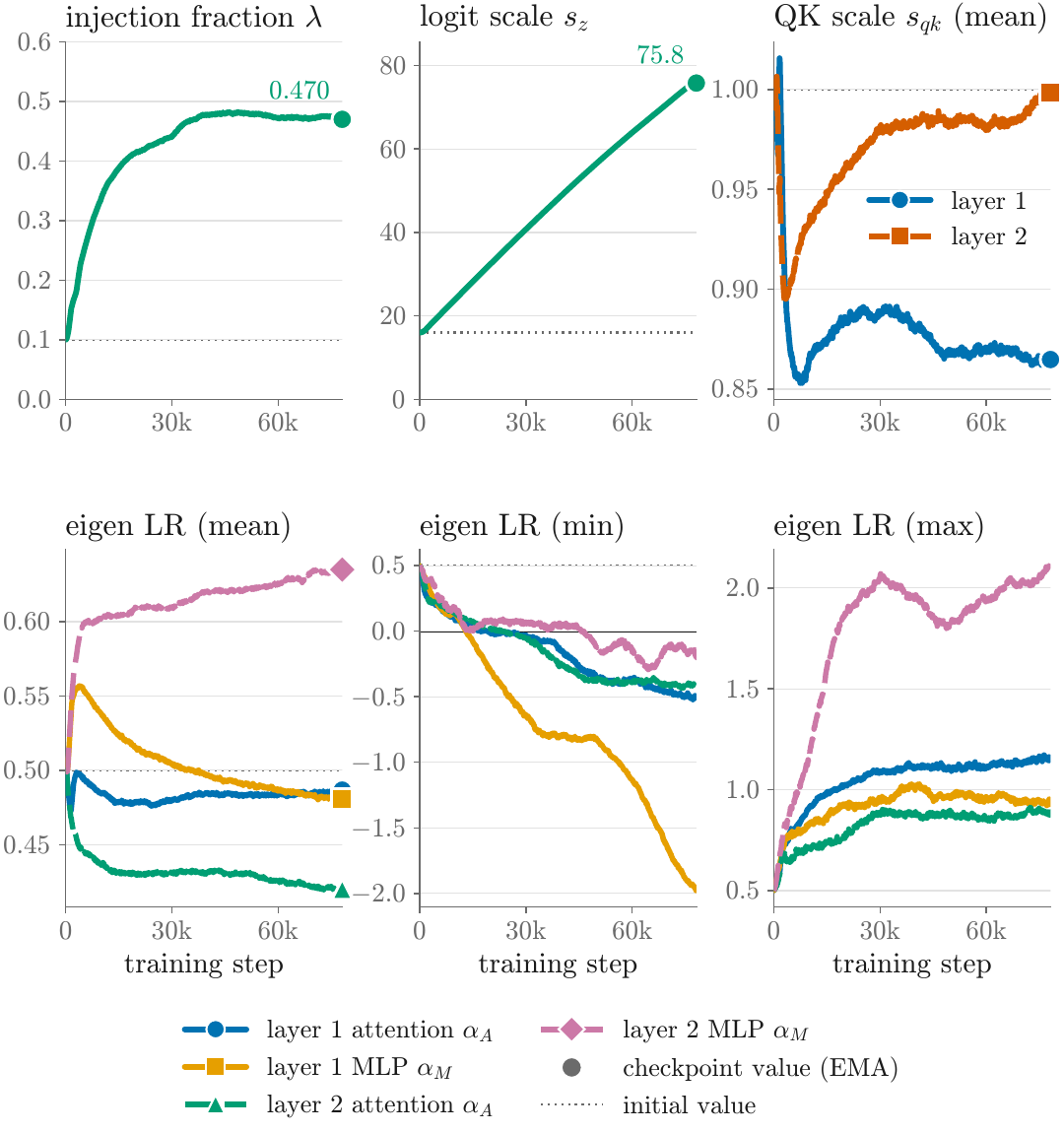}
  \caption{Sudoku-Extreme}
\end{subfigure}
\caption{\textbf{Learned parameters of SFPRM during training.}
Per task, top row: injection fraction $\lambda$, logit scale $s_z$, and QK scale $s_{qk}$ (mean over its 64 dimensions, per layer). Bottom row: the eigen learning rates $\alpha_A$ (attention) and $\alpha_M$ (MLP) of both layers, as the mean, minimum and maximum over the 512 hidden dimensions.}
\label{fig:learned-params}
\end{figure}

\subsection{Eigen LR and other learnable parameters during training}
\label{secapp:lambda_scales} 

\Cref{fig:learned-params} tracks the learned scalars of SFPRM during training on Maze and Sudoku. Both models raise the injection fraction from $0.1$ to $\lambda=0.575$ on Maze and $\lambda=0.470$ on Sudoku, with most of the increase in the first 30k of 78k training steps. This moves them toward the strong-injection regime of \Cref{thm:global-convergence}. On Sudoku, all eight initializations already reach the same fixed point on 72\% of puzzles at the learned $\lambda$ (\Cref{fig:global-conv}). The logit scale $s_z$ grows almost linearly from 16 to 76, at nearly the same rate on both tasks, and is still rising when training ends. The eigen learning rates keep means near their initial $0.5$, similarly as noted by \citet{Loshchilov2024nGPT}. However, individual dimensions spread widely, from negative values, which move the state away from the sub-layer output, to values up to $2.6$.

\subsection{The geometry of the computation}
\label{app:exp-rollouts}

\textbf{Qualitative results (\Cref{fig:filmstrip,fig:filmstrips-sudoku,fig:filmstrips-maze,fig:filmstrips-arc,fig:filmstrips-arc2}).} We unroll several test puzzles from each task open loop, for $40$ iterations on Maze and Sudoku (rendered at $t=0,2,11,39$) and $200$ on Mini-ARC (rendered at $t=0,3,13,52,199$). The decoded board is the per-cell argmax of the readout. The anchor map shows, per cell, the angle $\angle(z_{i,t},\hat x_i)$ between the iterate and its unit-normalized input embedding $\hat x_i$ (token plus position), on a colour scale normalized to the minimum and maximum angle over that puzzle's whole rollout, so that a colour means the same angle in every frame of a strip. The same quantity is computed for FPRM from its unnormalized state. \methodname's anchor angles separate the roles of the cells far more cleanly than FPRM's: on Sudoku, clue cells end $35^\circ$ closer to their anchors than blank cells, against $6^\circ$ for FPRM, and on Mini-ARC the cells a task must rewrite end furthest from their anchors (AUC $0.81$ for changed vs.\ unchanged cells, averaged over tasks), whereas FPRM's ordering flips from task to task (AUC $0.51$).

\begin{wrapfigure}[24]{R}{0.47\linewidth}
\vspace{-6mm}
\centering
\includegraphics[width=\linewidth]{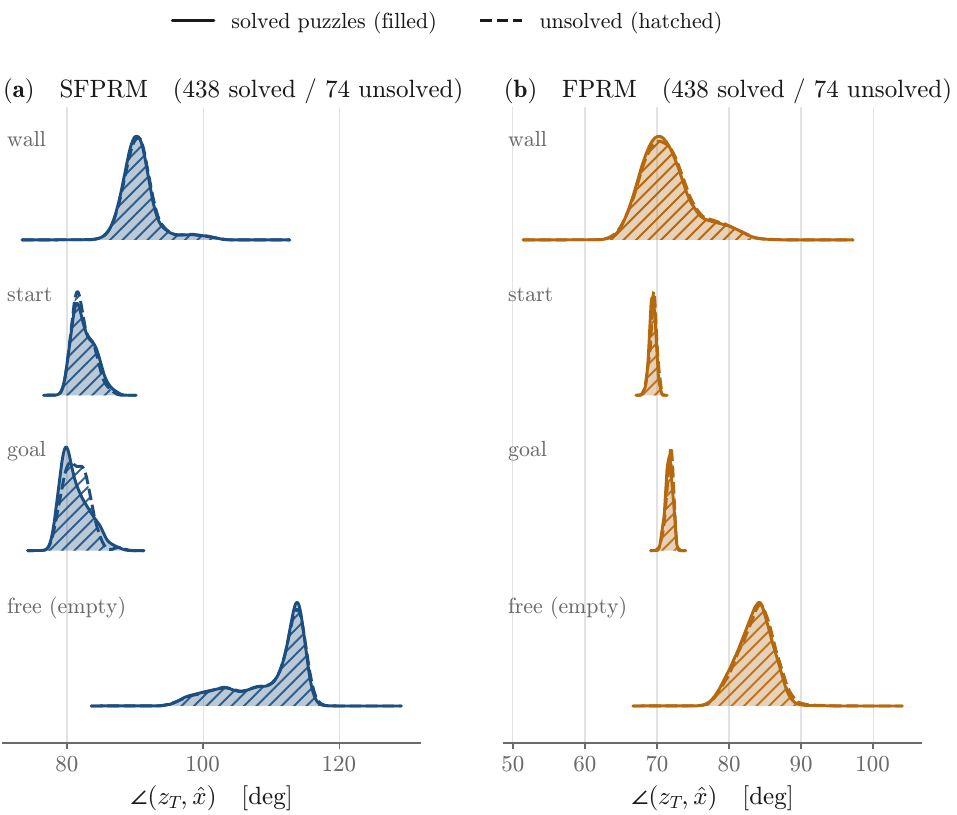}
\caption{\textbf{Final anchor angle on solved and unsolved mazes} (Maze-Hard analogue of \Cref{fig:interpretab_angles}). Per-class angle to the anchor at the end of the rollout over $512$ test mazes ($438$ solved, $74$ unsolved), for (a) \methodname and (b) FPRM's normalized state. Filled: solved; hatched: unsolved. Unlike on Sudoku, the distributions overlap for every class, so the final angle does not predict correctness.}
\label{fig:interpretab-angles-maze}
\end{wrapfigure}
\textbf{Anchor angles by cell class (\Cref{fig:anchor-angles,fig:anchor-angles-maze,fig:anchor-angles-maze-gallery}).} 
For four Sudoku and Maze test puzzles, we track the anchor angle of every cell throughout the unrolling, grouped by cell class (clues and blanks on Sudoku; walls, free cells that stay empty, free cells that become path, start, and goal on Maze). In the figures, the heavy lines are class means, whilst the thin lines show the individual cells' anchor angles.
The histograms on the right-hand side of each figure show the distribution at the end of the unrolling ($i=399$ on Sudoku, $999$ on Maze). For FPRM the angle is that of the normalized state.

\textbf{Trajectories on the sphere (\Cref{fig:sphere_traj}).} We draw each token's path in the frame spanned by the inference unrolling trajectory's two leading singular directions and its input anchor $\mathbf x$. This way the anchor is exact rather than a near-zero projection. We only show the cell trajectories that capture the most trajectory variance. Over $98\%$ of each trajectory's variance lies in that frame, so every token effectively moves on one great circle, as the geodesic steps of \Cref{alg:slfpr-simple} prescribe, and the anchor sits nearly orthogonal to that circle ($|\inner{x_t}{v_{1,2}}|\approx0.05$). The input injection pulls from far off the plane in which the
computation happens, explaining how a small $\lambda$ can anchor the hidden state without dictating it.

\textbf{Solved versus unsolved (\Cref{fig:interpretab_angles}).} For $128$ Sudoku test puzzles we compare the distribution of the final anchor angle (\methodname) or final norm (FPRM) of clue and blank cells between solved and unsolved puzzles. The same comparison on $512$ test mazes
(\Cref{fig:interpretab-angles-maze}) shows no separation for either model.


\begin{figure}[t]
\centering
\includegraphics[width=.99\linewidth]{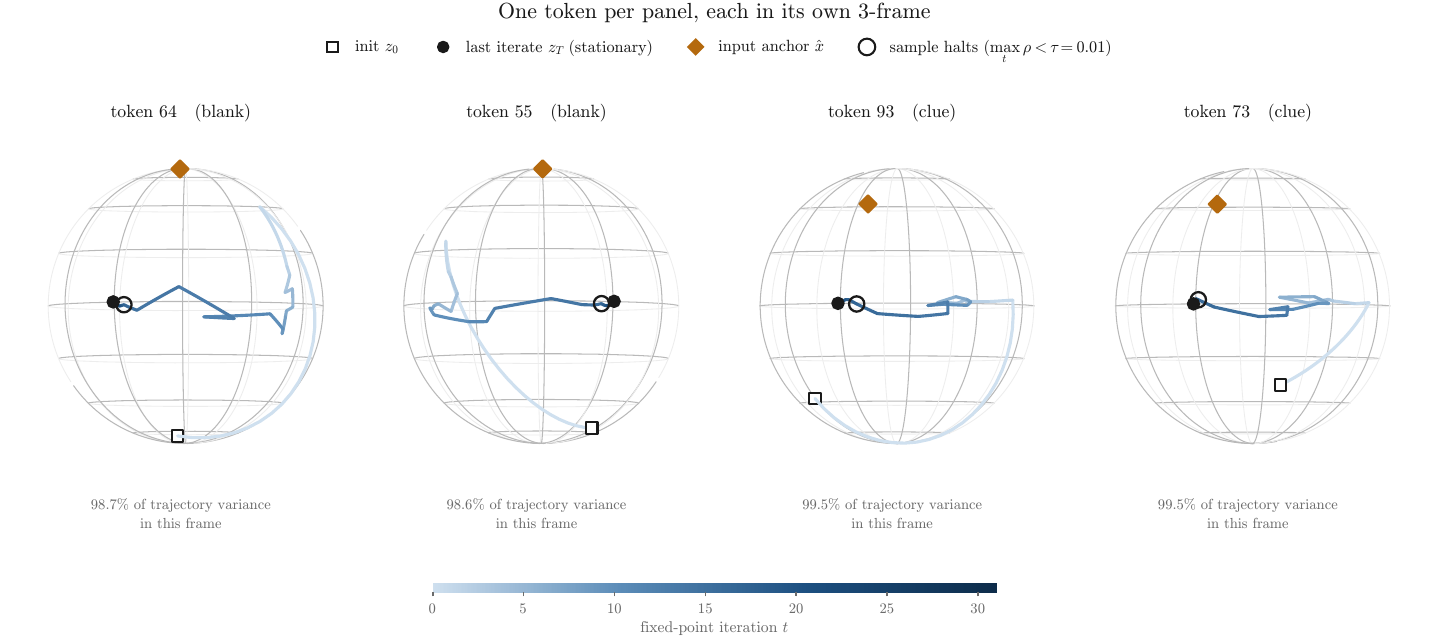}\\[10pt]
\includegraphics[width=.99\linewidth]{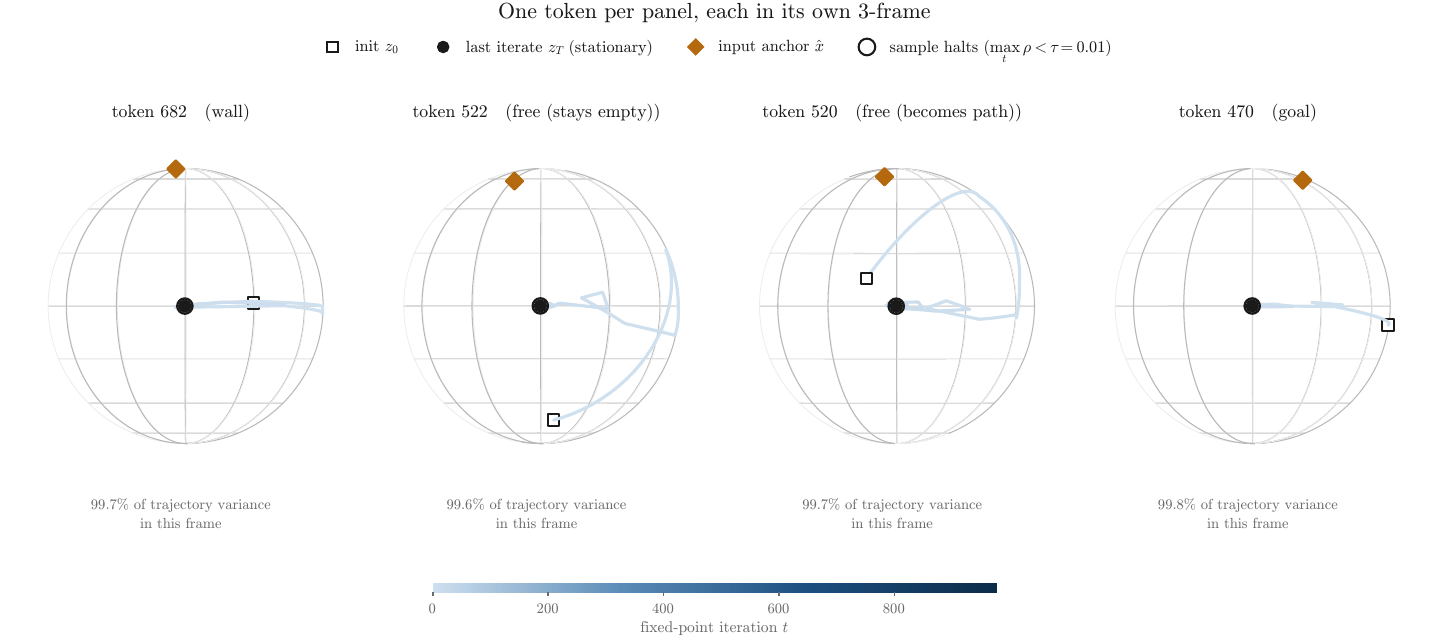}
\caption{\textbf{One \methodname trajectory per cell class, on the sphere} (Sudoku left, Maze right). For every cell class we draw the token whose $3$-frame captures the most trajectory variance. The frame is spanned by the trajectory's two leading singular directions and the anchor $x_t$, which sits at the pole. $x_t$ is nearly orthogonal to the plane of the trajectory ($|\inner{x_t}{v_{1,2}}|\approx0.05$). Arcs are great circles between projected iterates. The colour shows the iteration index.}
\label{fig:sphere_traj}
\end{figure}

\subsection{Fixed-point solvers at inference}
\label{appsec:solvers}

Both \methodname and FPRM are trained with a damped fixed-point iteration, but at inference any method that finds a fixed point of $F_\theta(\cdot\,;\mathbf x)$ can take its place. We run four solvers (including the default one) on the trained models, without retraining and without changing the architecture (see \Cref{sec:exp-halting}).

\subsubsection{Solvers}
\label{sec:solvers-methods}

All four solvers seek a fixed point of the same trained map.  The model weights are unchanged, and only the default solver (adaptive) was used during training. At test time the input embedding $\mathbf x$ is computed once and every solver runs the loop
\begin{equation*}
\widetilde{\mathbf z}_i=F_\theta(\mathbf z_i;\mathbf x),\qquad
\text{stop if } r_i<\tau,\qquad
\mathbf z_{i+1}=\mathrm{update}\bigl(\mathbf z_i,\widetilde{\mathbf z}_i,\text{history}\bigr),
\end{equation*}
starting from a random $\mathbf z_0$. Across solvers, only the update function differs. Each step evaluates $F_\theta$ once, and the same evaluation gives the halting residual, so the number of steps is the compute spent. \methodname halts on the
chordal residual $r_i=r_{\mathrm c,i}$ of \cref{eq:slfpr-residual} with $\tau_c=0.01$, and FPRM on the per-token
relative residual $r_i=r_{\mathrm E,i}=\max_t\norm{z_{i,t}-\widetilde z_{i,t}}_\infty/\norm{\widetilde z_{i,t}}_\infty$ (\cref{eq:fprm-halt}) with $\tau_{\mathrm E}=0.1$. We write $\mathbf g_i:=\widetilde{\mathbf z}_i-\mathbf z_i$
for the residual vector. Two operators adapt the updates to each model's state space: the interpolation
$\mathcal I_\eta(\mathbf a,\mathbf b)$ moves $\mathbf a$ a fraction $\eta$ of the way to $\mathbf b$, along the great
circle token by token for \methodname, $\mathcal I_\eta=\Slerp_\eta$ as in \cref{eq:sequence-slerp}, and along the
straight line for FPRM, $\mathcal I_\eta(\mathbf a,\mathbf b)=(1-\eta)\mathbf a+\eta\mathbf b$ as in
\cref{eq:fprm-damping}; the projection $\Pi$ applies $\Norm$ to every token for \methodname and is the identity for
FPRM.

\begin{figure}[t]
  \centering
  \begin{subfigure}[t]{0.48\linewidth}
    \centering
    \includegraphics[width=\linewidth]{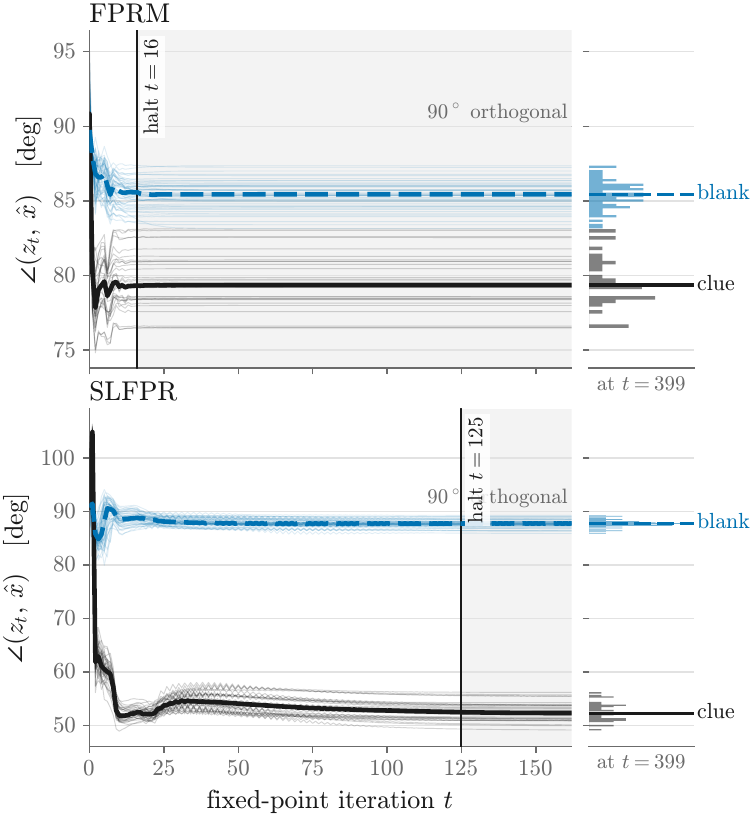}
    \caption{Sudoku test puzzle $0$}
  \end{subfigure}\hfill
  \begin{subfigure}[t]{0.48\linewidth}
    \centering
    \includegraphics[width=\linewidth]{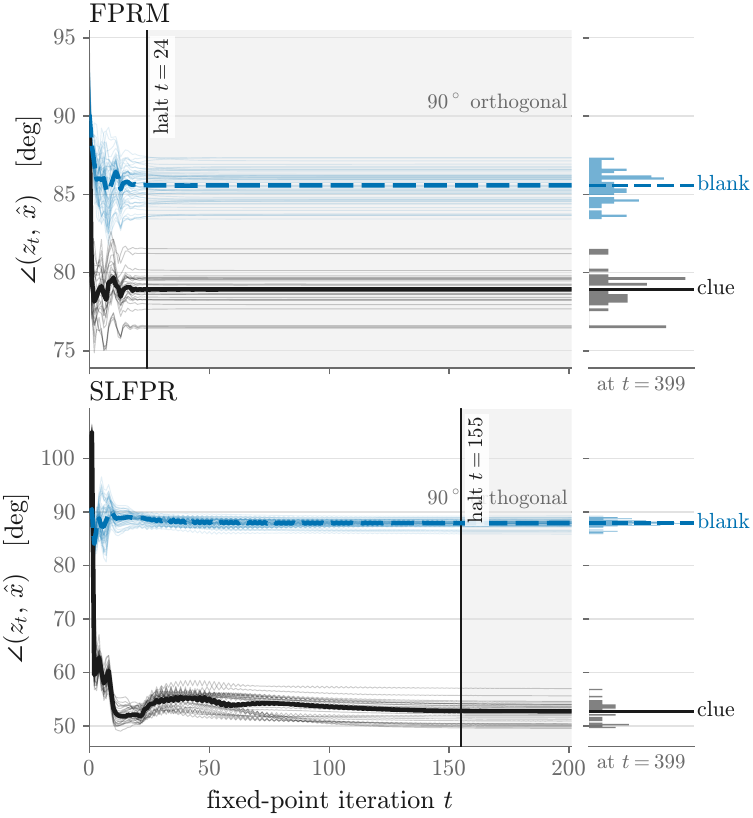}
    \caption{Sudoku test puzzle $1$}
  \end{subfigure}\\[4pt]
  \begin{subfigure}[t]{0.48\linewidth}
    \centering
    \includegraphics[width=\linewidth]{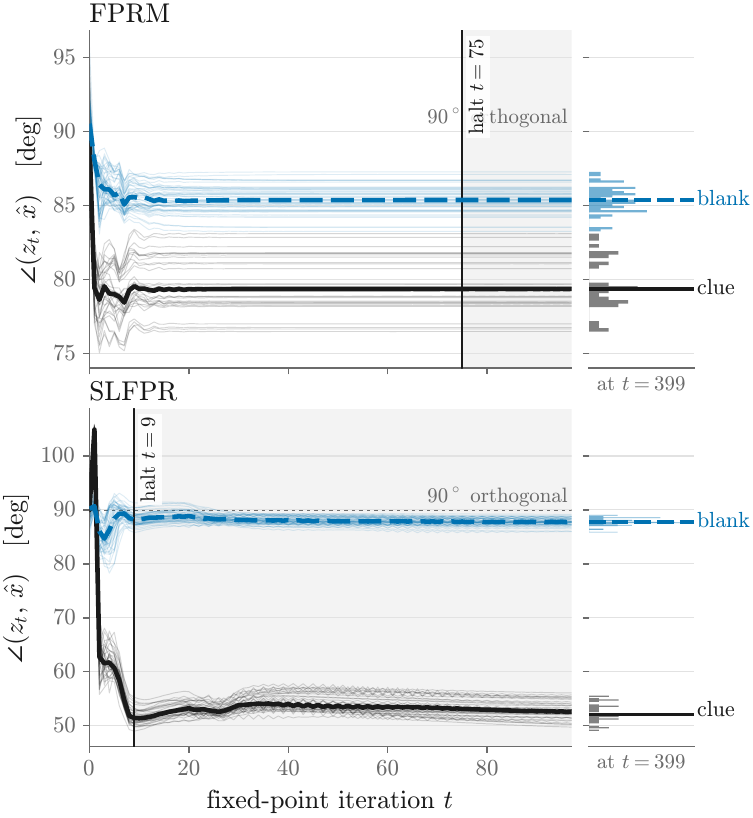}
    \caption{Sudoku test puzzle $2$}
  \end{subfigure}\hfill
  \begin{subfigure}[t]{0.48\linewidth}
    \centering
    \includegraphics[width=\linewidth]{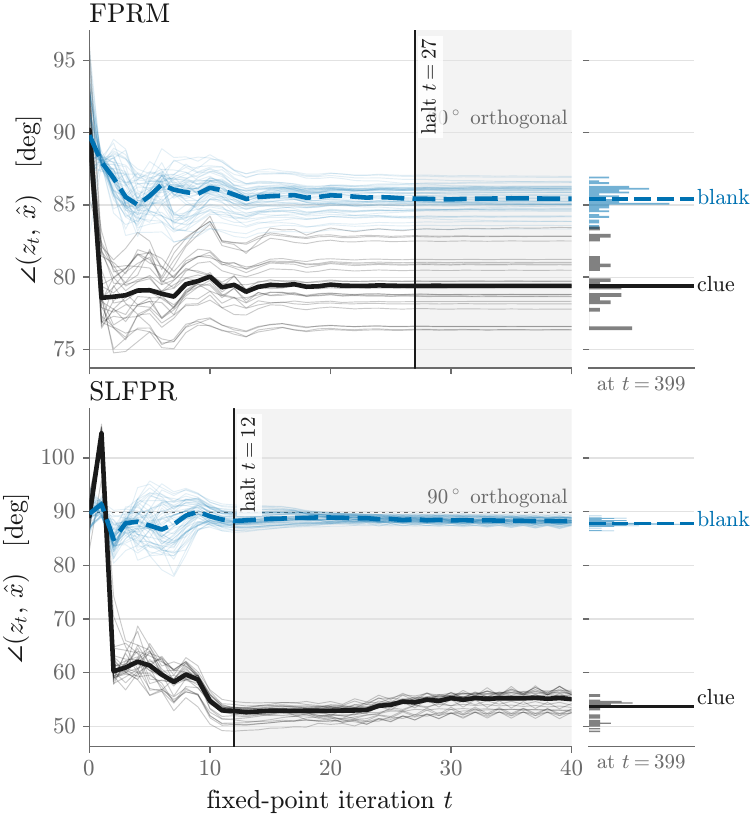}
    \caption{Sudoku test puzzle $3$}
  \end{subfigure}
  \caption{\textbf{Anchor angles by cell class on Sudoku-Extreme}. Four test puzzles (Maze in \Cref{fig:anchor-angles-maze,fig:anchor-angles-maze-gallery}). In each pair FPRM is above and \methodname below. Thin lines are individual cells, thick lines class means, the vertical line the halt iteration, and the margin the distribution at the end of the rollout. \methodname rotates the clues to ${\approx}52^\circ$ from their anchors within about $20$ iterations and leaves the blanks near
  orthogonal, whereas FPRM's normalized states settle within the first few iterations with clues and blanks only ${\approx}6^\circ$ apart. 
  }
  \label{fig:anchor-angles}
\end{figure}

\begin{figure}[t]
  \centering
  \begin{subfigure}[t]{0.48\linewidth}
    \centering
    \includegraphics[width=\linewidth]{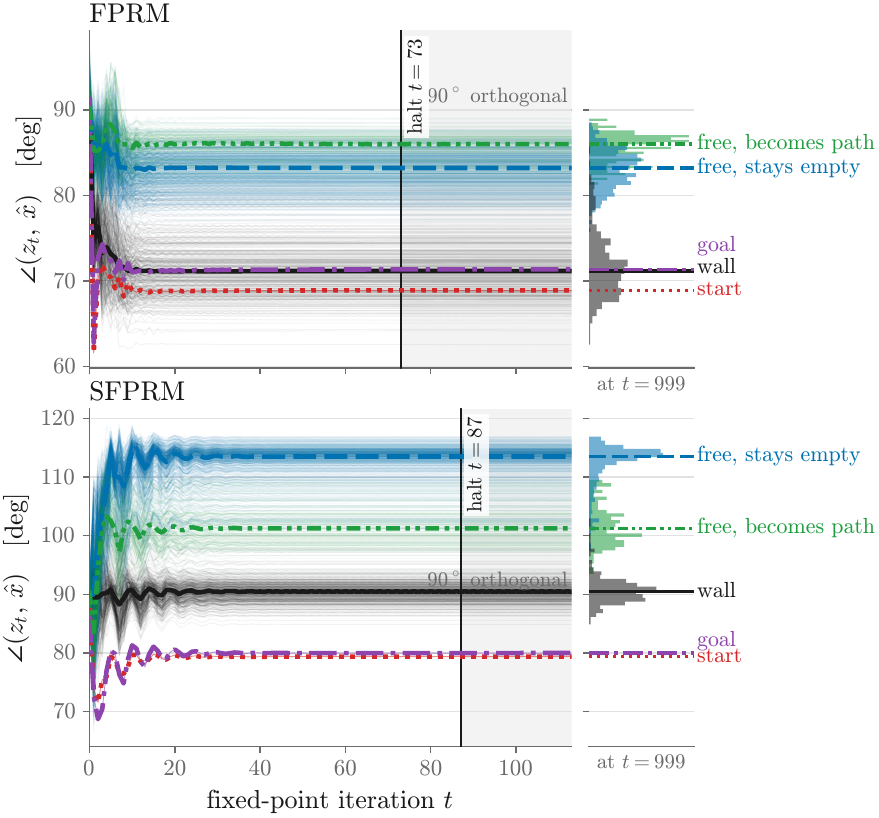}
    \caption{Maze test maze $0$}
  \end{subfigure}\hfill
  \begin{subfigure}[t]{0.48\linewidth}
    \centering
    \includegraphics[width=\linewidth]{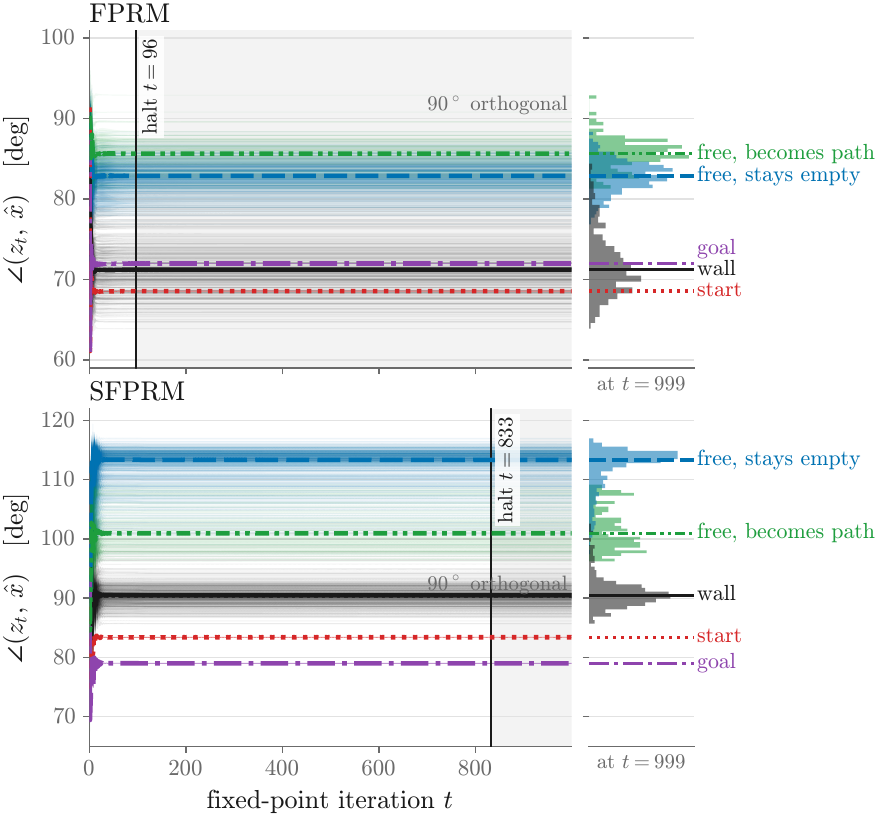}
    \caption{Maze test maze $1$}
  \end{subfigure}\\[4pt]
  \begin{subfigure}[t]{0.48\linewidth}
    \centering
    \includegraphics[width=\linewidth]{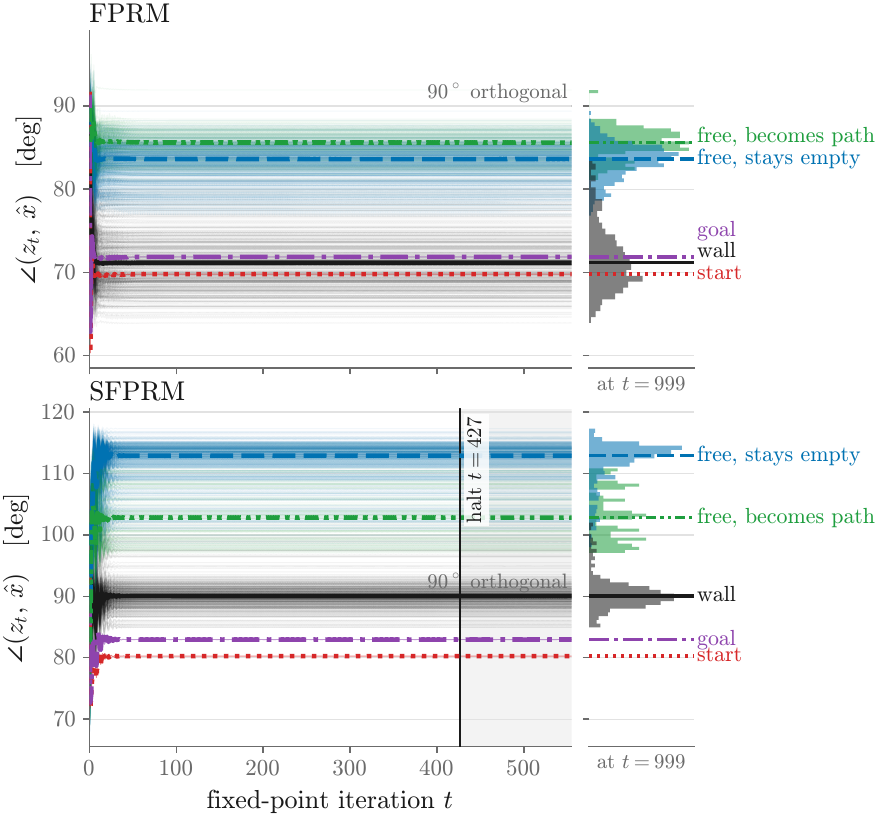}
    \caption{Maze test maze $2$}
  \end{subfigure}\hfill
  \begin{subfigure}[t]{0.48\linewidth}
    \centering
    \includegraphics[width=\linewidth]{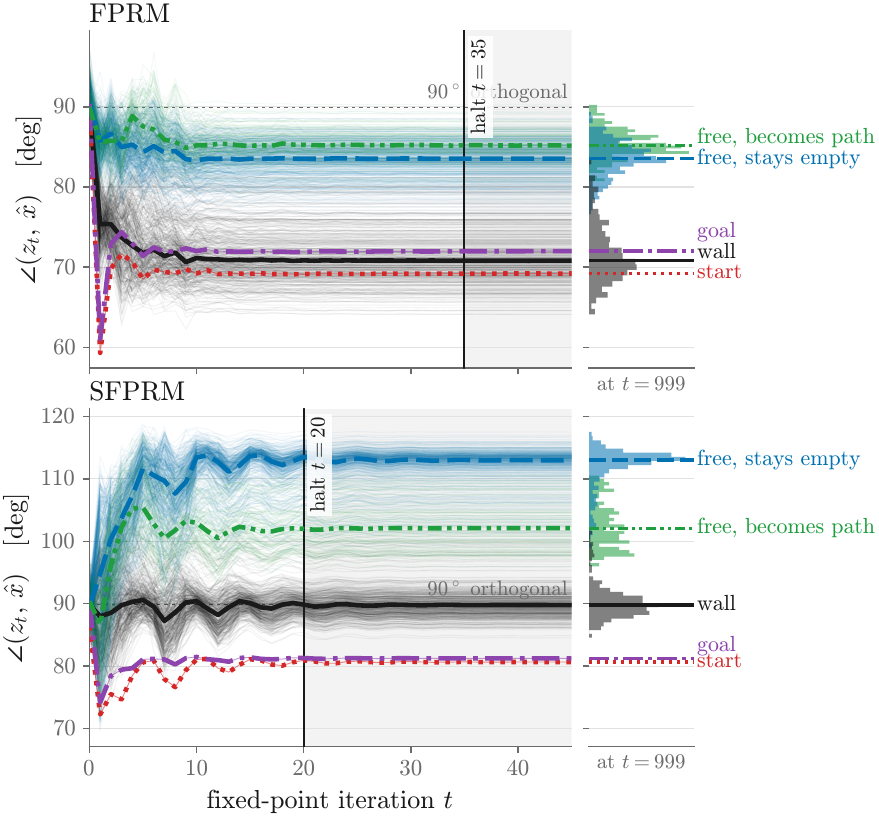}
    \caption{Maze test maze $3$ (\cref{fig:anchor-angles-maze})}
  \end{subfigure}
  \caption{\textbf{Anchor angles by cell class on Maze-Hard}. Four test puzzles, FPRM above and \methodname below in each pair, with the conventions of \Cref{fig:anchor-angles}. On every maze \methodname settles walls at ${\approx}90^\circ$, free cells that stay empty at ${\approx}113^\circ$, free cells that become path at ${\approx}102^\circ$, and start and goal at ${\approx}80^\circ$, whereas FPRM keeps all classes within $70$--$87^\circ$ and separates the two free-cell classes by only a few degrees. 
  }
  \label{fig:anchor-angles-maze-gallery}
\end{figure}

\textbf{Adaptive (default).} $\mathbf z_{i+1}=\mathcal I_{\eta_i}(\mathbf z_i,\widetilde{\mathbf z}_i)$ with
$\eta_0=1$. The step size is multiplied by $\gamma=0.99$ whenever the residual has not improved for $P=5$ consecutive
steps while it is still above $\tau$. Both models are trained and evaluated with this solver.

\textbf{Fixed step.} The same update with $\eta_i=1$ throughout, that is, plain fixed-point iteration
$\mathbf z_{i+1}=\widetilde{\mathbf z}_i$.

\textbf{Anderson acceleration} \citep{anderson1965iterative,walker2011anderson} replaces the latest state by the combination of the last $m+1$
states whose residuals cancel best. With the columns of $\Delta\mathbf G_i$ and $\Delta\widetilde{\mathbf Z}_i$ the $m$
differences of consecutive stored residuals and images,
\begin{equation*}
\boldsymbol\omega_i=\arg\min_{\boldsymbol\omega\in\reals^m}\,\norm{\mathbf g_i-\Delta\mathbf G_i\boldsymbol\omega}_2^2
+\delta\norm{\boldsymbol\omega}_2^2,\qquad
\mathbf z_{i+1}=\Pi\bigl(\widetilde{\mathbf z}_i-\Delta\widetilde{\mathbf Z}_i\boldsymbol\omega_i\bigr),
\end{equation*}
with a small ridge $\delta>0$.
The least-squares problem learns, from the recent iterates alone, how the residual responds to a change of state, 
and the combination removes the components that the plain step overshoots, so an oscillation across the fixed point, for example, cancels between consecutive states. We use memory $m=3$.

\textbf{Broyden's method} \citep{broyden1965class}, the solver commonly used for deep equilibrium models
\citep{Bai2019DEQ}, takes quasi-Newton steps on $\mathbf g(\mathbf z)=F_\theta(\mathbf z;\mathbf x)-\mathbf z$ with an
estimate $\mathbf B_i$ of the inverse Jacobian of $\mathbf g$:
\begin{equation*}
\mathbf z_{i+1}=\Pi\bigl(\mathbf z_i-\mathbf B_i\mathbf g_i\bigr),\qquad
\mathbf B_{i}=\mathbf B_{i-1}+\frac{(\mathbf s_i-\mathbf B_{i-1}\mathbf y_i)\,\mathbf s_i^{\top}\mathbf B_{i-1}}
{\mathbf s_i^{\top}\mathbf B_{i-1}\mathbf y_i},
\end{equation*}
where $\mathbf s_i=\mathbf z_i-\mathbf z_{i-1}$ and $\mathbf y_i=\mathbf g_i-\mathbf g_{i-1}$. Starting from $\mathbf B_0=-\mathbf I$, the estimate is stored as the last $m=5$ rank-one corrections.

Neither Anderson's nor Broyden's method computes a Jacobian or any gradient. Both are quasi-Newton methods that infer curvature from secants, the observed changes of state and residual between steps, so each step still costs one evaluation of $F_\theta$. With an empty memory both take exactly the fixed step. Both act on the whole state of an example at once, all tokens flattened into one vector, with one history per example in a batch.
Their extrapolated points leave the sphere and are projected back by $\Pi$. The extra work per step is a few vector
operations on the state and an $m\times m$ linear solve, and the extra memory is $m$ stored copies of the state, both
small next to an evaluation of $F_\theta$. As safeguards, a non-finite extrapolation falls back to the fixed step, and
Broyden skips an update whose denominator is near zero. Both methods use the step size $\eta=1$ (no damping).

\subsubsection{Results}
\label{sec:solvers-results}

We run every solver on all 1000 Maze test examples and on 1000 Sudoku test puzzles spaced evenly over the
422{,}786-puzzle test set, in fp32, with \methodname at its trained injection fraction ($\lambda=0.575$ on Maze,
$\lambda=0.1$ on Sudoku) and FPRM from its released checkpoints. Each example starts from its own random state (uniform
on the sphere for \methodname, the truncated normal of the FPRM solver for FPRM), shared by the four solvers, and is iterated
for 1000 steps. We then apply the halting rule to each example separately: the answer is the first iterate whose
residual is below $\tau$, and an example that never halts is answered by its iterate at step 1000.
\Cref{tab:solvers} reports the results, and \Cref{fig:solvers,fig:solver-pop} show accuracy and
residual over the iterations in the four settings.

\begin{table}[t]
\centering
\definecolor{accblue}{HTML}{0072B2}\definecolor{haltgreen}{HTML}{009E73}\definecolor{fastorange}{HTML}{D55E00}
\newcommand{\bestacc}[1]{\textbf{\textcolor{accblue}{#1}}}
\newcommand{\besthalt}[1]{\textbf{\textcolor{haltgreen}{#1}}}
\newcommand{\bestfast}[1]{\textbf{\textcolor{fastorange}{#1}}}
\caption{\textbf{Inference-time fixed-point solvers on the trained models.} Every example runs $1000$
steps from its own random start, shared by the four solvers; each step is one evaluation of $F_\theta$,
and an example halts at the first iterate with residual below $\tau$. \emph{Halting}: share of examples
that halt (\%), median and mean halting step over those, and mean $F$ evaluations over all examples.
\emph{Exact at halt}: accuracy over all examples, exact binomial test against the default solver ($p$),
and accuracy among halted examples. \emph{Open loop}: accuracy with halting disabled, at its best step and
at $t=1000$. Maze: all $1000$ test mazes; Sudoku: $1000$ puzzles evenly spaced over the test set.
Best within each block: \bestacc{exact accuracy}, \besthalt{share halted}, \bestfast{fewest halting steps}.}
\label{tab:solvers}
\scriptsize
\setlength{\tabcolsep}{3pt}
\resizebox{\linewidth}{!}{
\begin{tabular}{@{}l rrrr rrr rr@{}}
\toprule
 & \multicolumn{4}{c}{Halting} & \multicolumn{3}{c}{Exact at halt (\%)} & \multicolumn{2}{c}{Open loop (\%)} \\
\cmidrule(lr){2-5}\cmidrule(lr){6-8}\cmidrule(l){9-10}
Solver & \% & med.\ $t$ & mean $t$ & $F$ evals & all & $p$ & halted & peak ($t$) & $t{=}1000$ \\
\midrule
\multicolumn{10}{@{}l}{\textit{\methodname, Maze} \enspace ($\lambda{=}0.575$, $\tau_c{=}0.01$)} \\
\hspace{0.6em}Adaptive (default) & 99.2 & 27 & 96.3 & 104.5 & \bestacc{87.8} & -- & 87.9 & 87.8 (20) & 87.8 \\
\hspace{0.6em}Fixed, $\eta{=}1$ & 63.1 & 23 & 25.8 & 386.3 & \bestacc{87.8} & 1.00 & 87.3 & 87.8 (20) & 87.8 \\
\hspace{0.6em}Anderson, $m{=}3$ & \besthalt{100.0} & \bestfast{13} & \bestfast{13.8} & 14.8 & 86.8 & 0.04 & 86.8 & 87.4 (19) & 86.7 \\
\hspace{0.6em}Broyden, $m{=}5$ & 99.4 & 14 & 14.0 & 20.9 & 87.4 & 0.39 & 87.9 & 87.9 (16) & 51.5 \\
\midrule
\multicolumn{10}{@{}l}{\textit{\methodname, Sudoku} \enspace ($\lambda{=}0.1$, $\tau_c{=}0.01$)} \\
\hspace{0.6em}Adaptive (default) & 75.0 & 435 & 469.9 & 603.5 & \bestacc{95.9} & -- & 100.0 & 95.9 (933) & 95.9 \\
\hspace{0.6em}Fixed, $\eta{=}1$ & 0.0 & -- & -- & 1001.0 & 93.2 & $<$0.001 & -- & 93.2 (973) & 93.2 \\
\hspace{0.6em}Anderson, $m{=}3$ & \besthalt{95.4} & 37 & 71.3 & 115.1 & 95.3 & 0.34 & 99.9 & 95.3 (909) & 95.3 \\
\hspace{0.6em}Broyden, $m{=}5$ & 89.8 & \bestfast{24} & \bestfast{49.9} & 147.8 & 90.0 & $<$0.001 & 99.7 & 69.6 (27) & 1.5 \\
\midrule
\multicolumn{10}{@{}l}{\textit{FPRM, Maze} \enspace ($\tau_{\mathrm E}{=}0.1$)} \\
\hspace{0.6em}Adaptive (default) & 95.9 & 55 & 129.3 & 166.0 & 86.0 & -- & 86.8 & 86.1 (102) & 86.0 \\
\hspace{0.6em}Fixed, $\eta{=}1$ & 51.4 & 32 & 51.9 & 513.7 & 85.4 & 0.07 & 92.0 & 85.8 (24) & 85.4 \\
\hspace{0.6em}Anderson, $m{=}3$ & \besthalt{100.0} & 17 & \bestfast{18.0} & 19.0 & \bestacc{86.1} & 1.00 & 86.1 & 86.0 (64) & 86.0 \\
\hspace{0.6em}Broyden, $m{=}5$ & 98.6 & \bestfast{16} & 21.1 & 35.8 & 86.0 & 1.00 & 87.1 & 86.1 (19) & 32.8 \\
\midrule
\multicolumn{10}{@{}l}{\textit{FPRM, Sudoku} \enspace ($\tau_{\mathrm E}{=}0.1$)} \\
\hspace{0.6em}Adaptive (default) & \besthalt{91.5} & 50 & 125.0 & 200.3 & 88.0 & -- & 96.2 & 88.1 (972) & 88.1 \\
\hspace{0.6em}Fixed, $\eta{=}1$ & 43.8 & 39 & \bestfast{82.6} & 599.2 & 80.9 & $<$0.001 & 100.0 & 80.9 (932) & 80.9 \\
\hspace{0.6em}Anderson, $m{=}3$ & 89.2 & \bestfast{34} & 91.1 & 190.2 & \bestacc{88.8} & 0.38 & 99.6 & 89.0 (984) & 89.0 \\
\hspace{0.6em}Broyden, $m{=}5$ & 18.3 & 344 & 389.7 & 889.3 & 14.7 & $<$0.001 & 79.2 & 13.6 (8) & 0.6 \\
\bottomrule
\end{tabular}
}
\end{table}

\begin{figure}[t]
\centering
\includegraphics[width=0.325\linewidth]{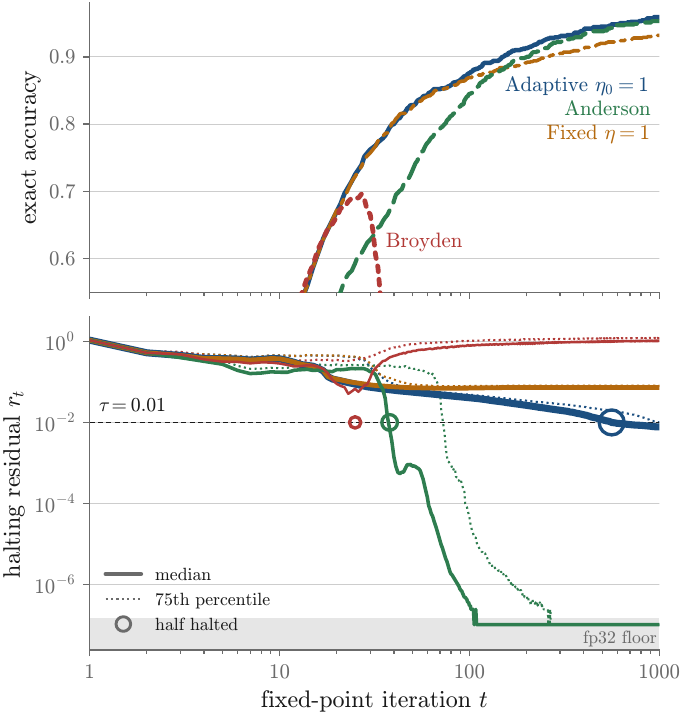}\hfill
\includegraphics[width=0.325\linewidth]{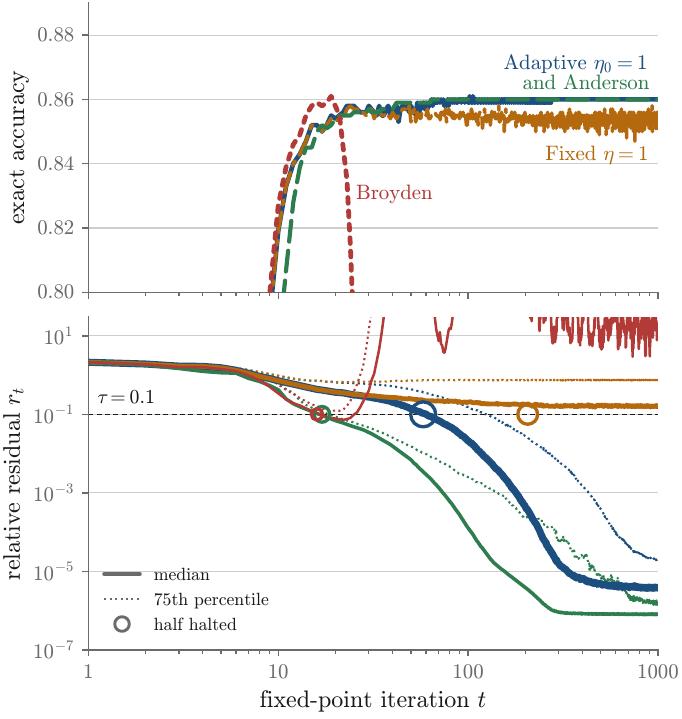}\hfill
\includegraphics[width=0.325\linewidth]{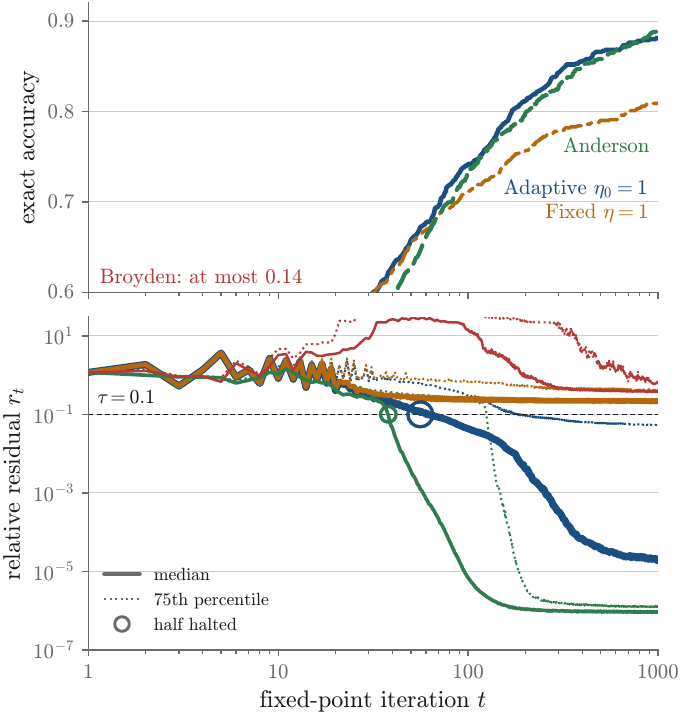}
\caption{\textbf{Accuracy and halting residual over the iterations for the remaining settings of
\Cref{tab:solvers}}. \methodname on Sudoku ($\lambda=0.1$, left), FPRM on Maze (middle) and FPRM on
Sudoku (right); \methodname on Maze is \Cref{fig:solvers}. Same $1000$ examples and solvers as the
table. \emph{Top:} exact accuracy of the iterate with halting disabled. \emph{Bottom:} each model's
halting residual, median (solid) and 75th percentile (dotted) over the examples, its threshold
(dashed), and the iteration by which half of the examples have halted (circle).}
\label{fig:solver-pop}
\end{figure}

\textbf{The undamped step does not halt reliably.}
The fixed step never reaches $\tau$ on 37\%, 100\%, 49\% and 56\% of the examples in the four settings. For \methodname on
Maze its answers are identical to the default solver's, but on Sudoku it is less accurate for both models. For \methodname on Sudoku a quarter of the puzzles do not halt within 1000 steps.

\textbf{Anderson acceleration matches the default solver at a fraction of the cost.}
Its accuracy at the halt is within one point of the default solver's in all four settings. Anderson halts earlier in every setting, with a striking 37 against 435 for \methodname on Sudoku. Averaged over all examples, this is 9$\times$ fewer evaluations of $F_\theta$. For FPRM on Sudoku the saving disappears (190 against 200 evaluations): 10.8\% of the puzzles never reach $\tau$ under Anderson, and all of them are answered incorrectly.

\textbf{Broyden's method reaches the tolerance threshold quickly but does not stay there.}
On Maze it halts as early as Anderson (14 and 16 median steps), and its halted answers are as accurate as the default solver's. By continuing to iterate beyond the tolerance, however, it leaves the fixed point and its accuracy degrades, making it usable only with halting.

\textbf{On Sudoku, halting certifies the answer.}
For \methodname on Sudoku, 99.7--100\% of the puzzles a solver halts on are solved exactly, so the accuracy differences come entirely from the puzzles it fails to halt on. The same holds less strictly for FPRM (96.2\% of the default solver's halted puzzles are solved, 99.6\% of Anderson's). On Maze, halting does not certify the answer: incorrect mazes converge too, and for the solvers that halt on almost every example the accuracy among halted examples is about the same as the overall accuracy.


\textbf{Trajectories in the principal plane.}
\Cref{fig:solver-planes} shows the four solvers at the trained $\lambda$ in the plane construction of \Cref{fig:global-conv}b, on a single test example from each task, from eight initializations. The reference $\mathbf z^\star$ is Anderson's end point, a fixed point to fp32 precision that decodes both examples exactly. On Maze, Anderson halts at iteration $13$ from all eight initializations, the default solver crosses the tolerance at a median of $140$
iterations but ends $0.06$~rad away, and fixed $\eta=1$ never crosses it. On the Sudoku puzzle only Anderson halts (median $40$), while the default and fixed solvers end $0.16$ and $0.33$~rad from $\mathbf z^\star$ and still decode the puzzle exactly. Broyden reaches the tolerance as early as Anderson and then leaves from most starts. 

\textbf{Trajectories in polar coordinates.}
\Cref{fig:solver-traj-slfpr,fig:solver-traj-sudoku} follow one token per cell class under each solver in the polar view of \Cref{fig:solvers}b, for the maze of that panel and for a Sudoku puzzle. The radius is the geodesic angle to a common reference, Anderson's end point after $2000$ iterations, a fixed point to fp32 precision that decodes both
examples exactly. Inward is convergence toward $\mathbf z^\star$, a held radius is orbiting, and outward is leaving. Anderson ends within
$10^{-4}$ degrees of the reference. The default solver ends $0.005$--$0.14^\circ$ away on the mazes and up to $2.1^\circ$ on the Sudoku puzzles, after orbiting for hundreds of iterations; fixed $\eta=1$ stays
up to $0.64^\circ$ and $4$--$17^\circ$ away and never halts; and Broyden reaches the tolerance first and then leaves to $60$--$103^\circ$.

\begin{figure}[t]
\centering
\includegraphics[width=\linewidth]{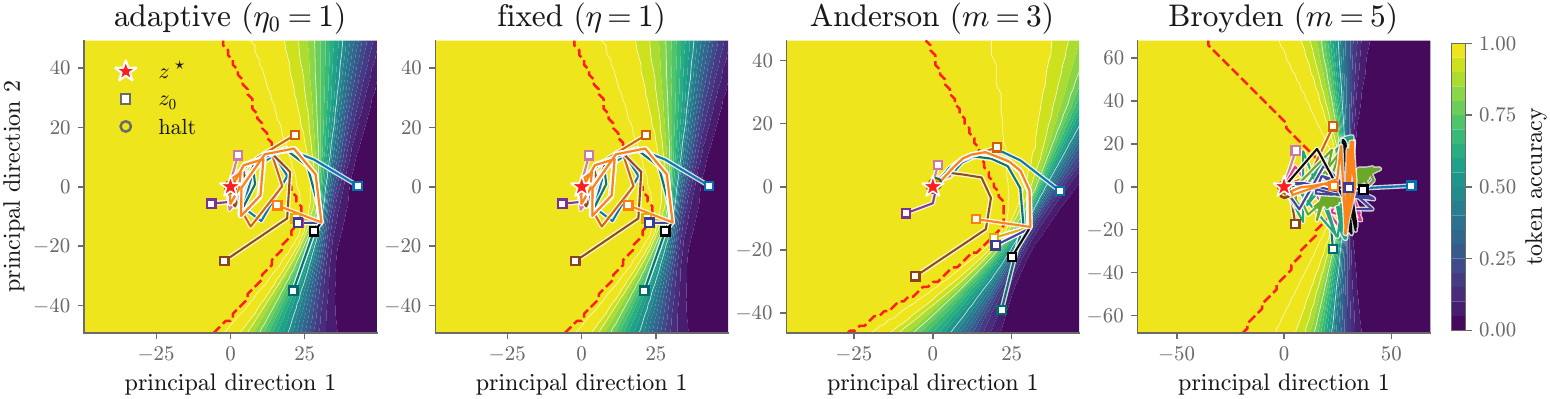}\\[2pt]
\includegraphics[width=\linewidth]{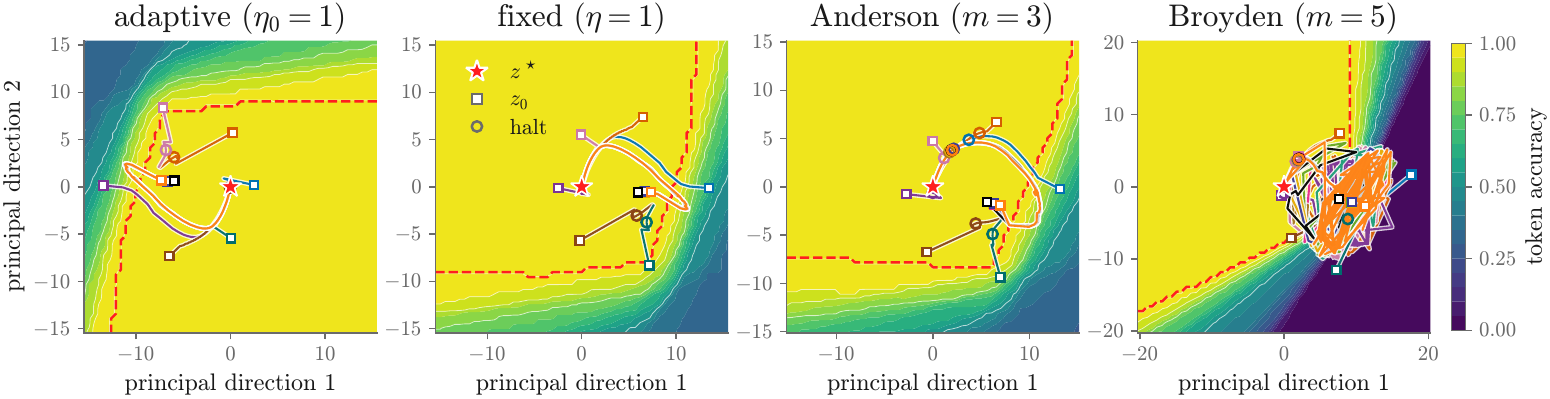}
\caption{\textbf{The four solvers in the principal plane through the fixed point}, for test example
$0$ of Maze-Hard (top, $\lambda=0.575$) and of Sudoku-Extreme (bottom, $\lambda=0.1$). Axes and markers
as in \Cref{fig:global-conv}b, with $\mathbf z^\star$ the end point of Anderson acceleration; every
solver runs from the same eight initializations (squares) and six further starts placed on its plane.
Each panel uses the top-2 principal directions of that solver's own iterates, so the background
landscapes differ between panels.}
\label{fig:solver-planes}
\end{figure}

\begin{figure}[p]
\centering
\includegraphics[width=0.99\linewidth]{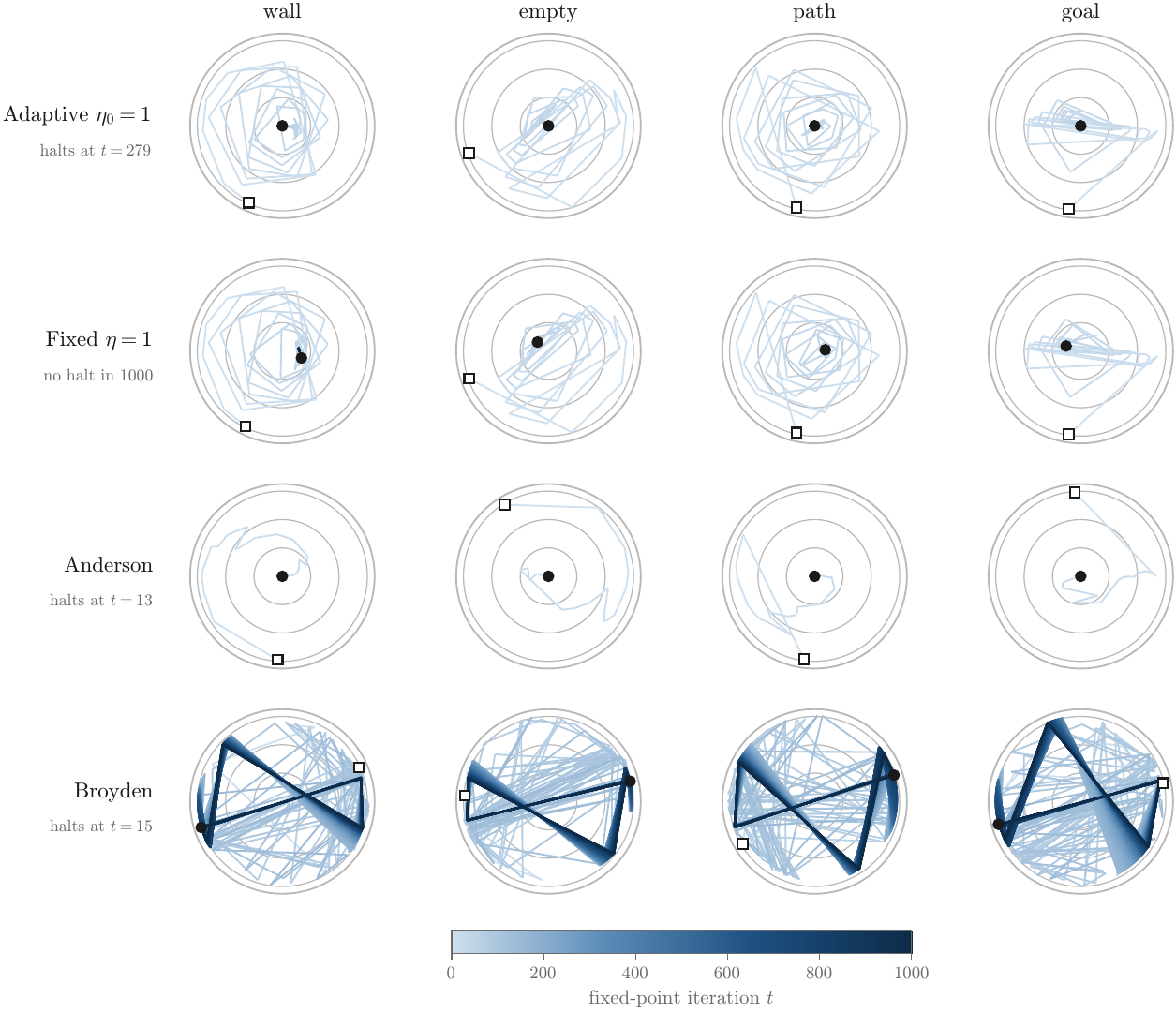}
\caption{\textbf{\methodname on Maze-Hard: State trajectories in polar coordinates.} The four solvers
(rows, with each one's halting step) on the maze of \Cref{fig:solvers}b, one token per cell class
(columns). Radius: geodesic angle to the reference fixed point, log-scaled from $0.1^\circ$ at the
centre to $180^\circ$ at the rim, rings at $1$, $10$ and $100^\circ$; azimuth: the top-2 principal
plane of that token's component orthogonal to the reference; square: start, dot: end of the
$1000$-iteration run; colour: iteration. Anderson ends at the centre, the default and fixed steps
orbit at a small radius, and Broyden leaves for the rim.}
\label{fig:solver-traj-slfpr}
\end{figure}

\begin{figure}[p]
\centering
\includegraphics[width=0.99\linewidth]{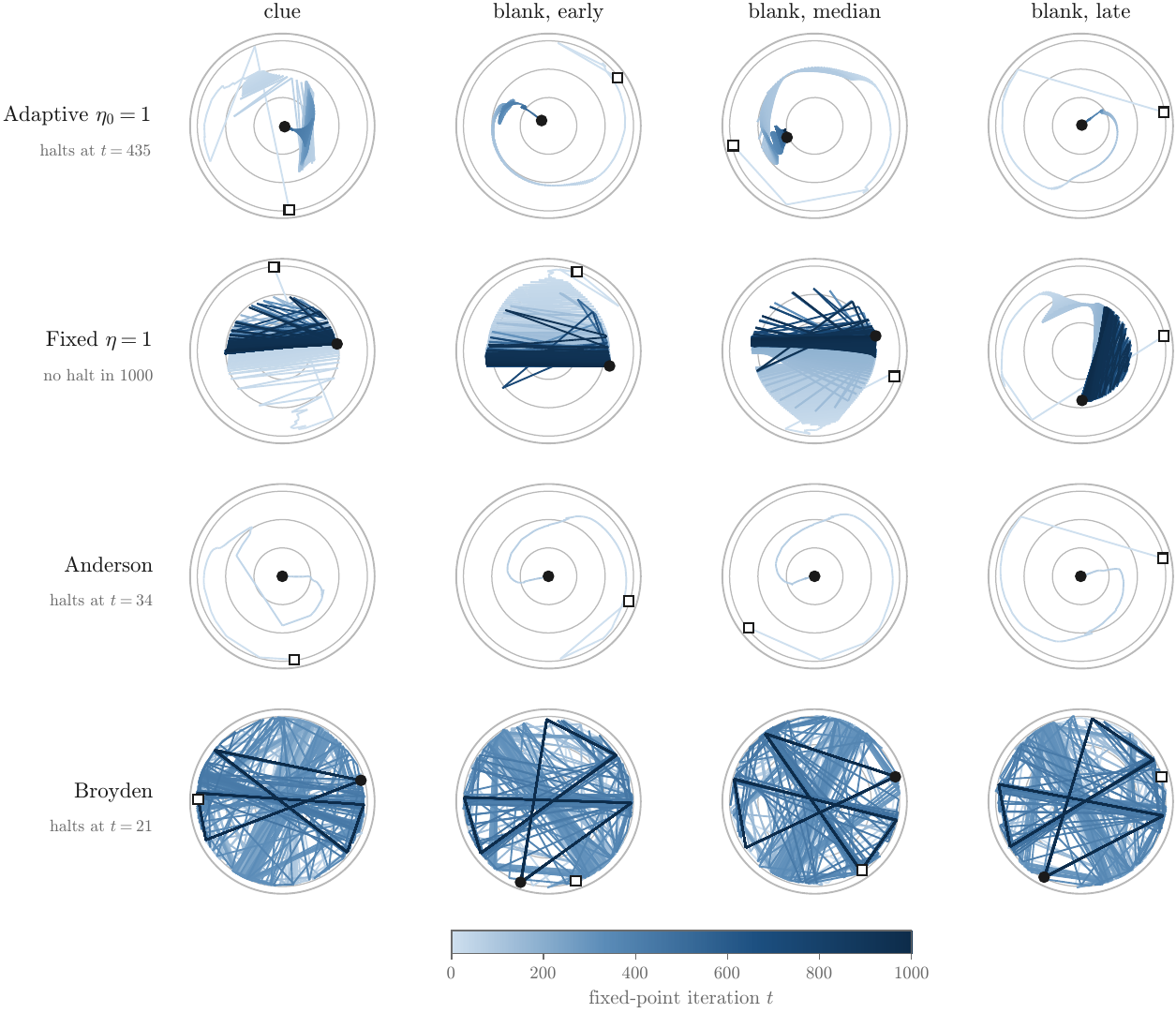}
\caption{\textbf{\methodname on Sudoku-Extreme: State trajectories in polar coordinates.} A puzzle on which the default solver halts at its median step ($435$); \textit{columns}: the middle clue cell and the blank cells whose decode settles earliest, at the median and last. Anderson and Broyden reach the tolerance early, the default solver orbits for hundreds of iterations before halting, fixed $\eta=1$ never halts, and Broyden fills the disc after leaving.}
\label{fig:solver-traj-sudoku}
\end{figure}


\subsection{ARC-AGI-2 Results}
\label{sec:arc2}

\begin{wraptable}[11]{R}{0.49\linewidth}
\vspace{-11mm}
\centering
\caption{\textbf{\methodname ablations on Mini-ARC.} Pass@1 accuracy at an iteration cap of 200. Each row changes one component of the full model. $\Delta$: change relative to the full model.}
\vspace{-2mm}
\label{tab:miniarc-ablation}
\resizebox{\linewidth}{!}{
\begin{tabular}{lcc}
\toprule
Variant & Pass@1 & $\Delta$ \\
\midrule
\textbf{\methodname (full)} & \textbf{59.1} & -- \\
\midrule
w/o QK normalization & 26.2 & $-32.9$ \\
w/o damping ($\eta = 1$) & 47.7 & $-11.4$ \\
Fixed injection fraction ($\lambda = 0.1$) & 49.7 & $-9.4$ \\
\midrule
ACT halting instead of chordal residual & 54.4 & $-4.7$ \\
Fixed eigen learning rates ($\alpha_A = \alpha_M = 0.5$) & 55.0 & $-4.0$ \\
w/o logit scale ($s_z = 1$) & 55.7 & $-3.4$ \\
w/o weight normalization & 56.4 & $-2.7$ \\
\bottomrule
\end{tabular}
}
\end{wraptable}

We compute the Pass@2 on the $120$ evaluation tasks of ARC-AGI-2. All models remain far from solving the benchmark. The best recurrent reasoner, TRM, reaches $7.8\%$, and \methodname reaches $3.8\%$. The gaps nonetheless amount to only a few tasks, so we do not draw strong conclusions from them. At $7$M parameters, \methodname still outperforms reasoning LLMs, including o3-mini-high ($3.0\%$), DeepSeek-R1 ($1.3\%$) and Claude~3.7 ($0.7\%$).


\subsection{Algorithmic Ablations}
Table~\ref{tab:miniarc-ablation} retrains the Mini-ARC model with one component of \methodname changed at a time. QK normalization is essential, and removing the step damping ($\eta = 1$) or fixing the input-injection fraction costs 11.4 and 9.4 points, respectively. The remaining changes, including ACT halting in place of the chordal residual, cost at most 4.7 points.

\section{Limitations}
\label{app:limitations}
On the theory side, \Cref{prop:converge} certifies a step-size window but not that the patience rule finds it, and \Cref{thm:global-convergence} covers only strong injection. On the experimental side, every run uses a
single seed and neither model was tuned beyond FPRM's recipe: a broader hyperparameter search, stronger regularization, and optimizers that
exploit the spherical geometry further would likely improve both. Two directions remain open. Attention-free MLP-Mixer blocks have been reported to beat Transformer blocks on some of these puzzles
\citep{JolicoeurMartineau2025TRM,sghaier2026ptrm,suleymanzade2026loopedflows}; the spherical loop is agnostic to the block and could be combined with them. And we have not applied \methodname to language modeling, where looped Transformers are gaining ground \citep{geiping2025recurrent,prairie2026parcaescalinglawsstable} and where halting on a fixed angular tolerance may be most valuable.

\clearpage


\begin{figure}[p]
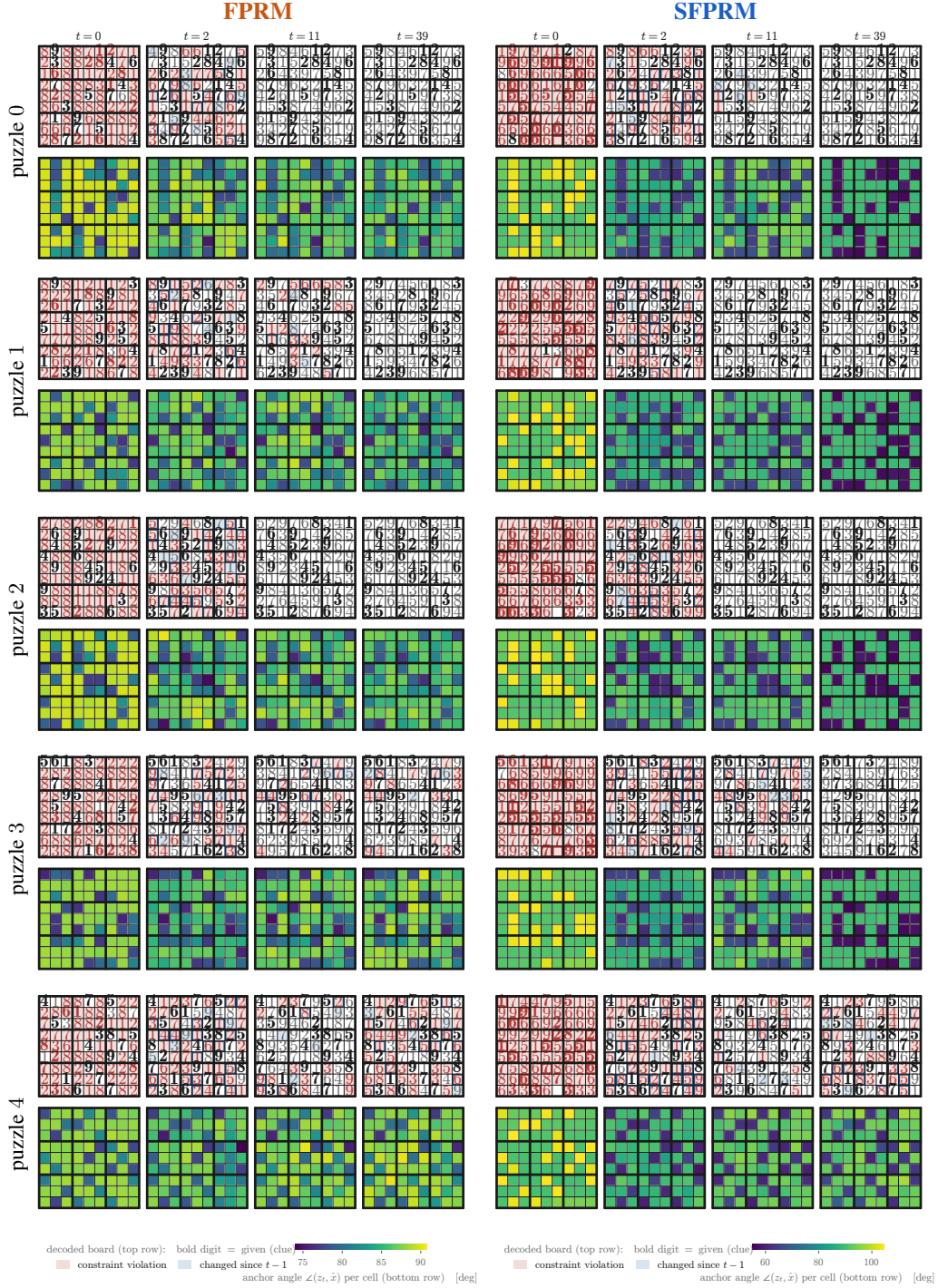

  \centering
  \begin{tabular}{@{}c@{\hspace{3pt}}c@{\hspace{5pt}}c@{}}
    & \dfp{\textbf{FPRM}} & \dsl{\textbf{\methodname}} \\[1pt]
    \sudokutop{0}
    \sudokumid{1}
    \sudokumid{2}
    \sudokumid{3}
    \sudokubot{4}
  \end{tabular}
  \caption{\textbf{Filmstrip gallery on Sudoku-Extreme.} Test puzzles $0$--$4$. One test puzzle per row, rolled open-loop for $40$ iterations and rendered at $i=0,2,11,39$ (protocol in Appendix \ref{app:exp-rollouts}), for FPRM (left) and \methodname (right): decoded board on top, each cell's angle to its input anchor below.}
  \label{fig:filmstrips-sudoku}
\end{figure}


\begin{figure}[p]
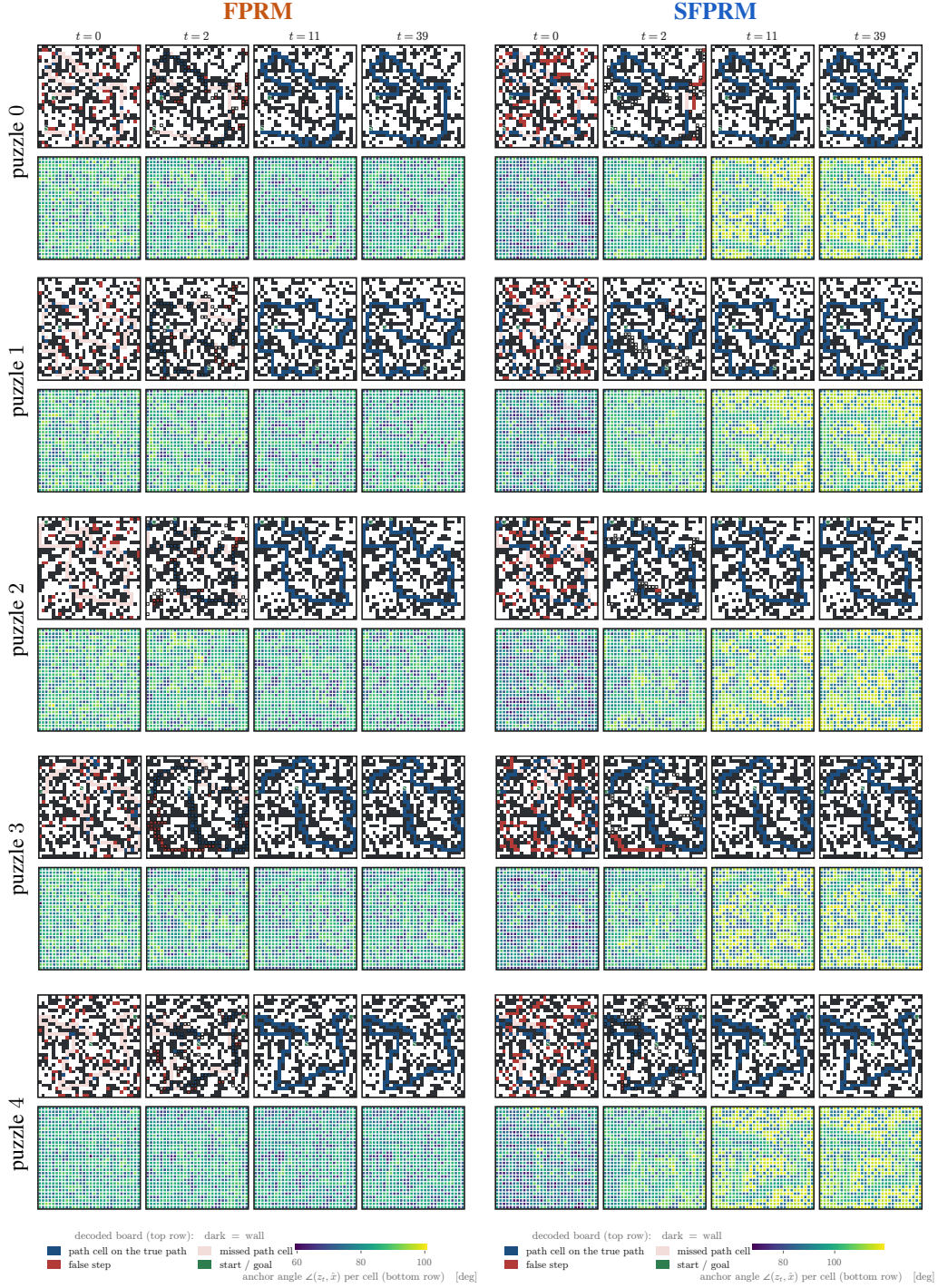

  \centering
  \begin{tabular}{@{}c@{\hspace{3pt}}c@{\hspace{5pt}}c@{}}
    & \dfp{\textbf{FPRM}} & \dsl{\textbf{\methodname}} \\[1pt]
    \mazetop{0}
    \mazemid{1}
    \mazemid{2}
    \mazemid{3}
    \mazebot{4}
  \end{tabular}
  \caption{\textbf{Filmstrip gallery on Maze-Hard.} Test puzzles $0$--$4$. One $30{\times}30$ test maze per row, rolled open-loop for $40$ iterations and rendered at $i=0,2,11,39$ (protocol in Appendix \ref{app:exp-rollouts}), for FPRM (left) and \methodname (right):
  decoded path on top, each cell's angle to its input anchor below.}
  \label{fig:filmstrips-maze}
\end{figure}


\begin{figure}[p]
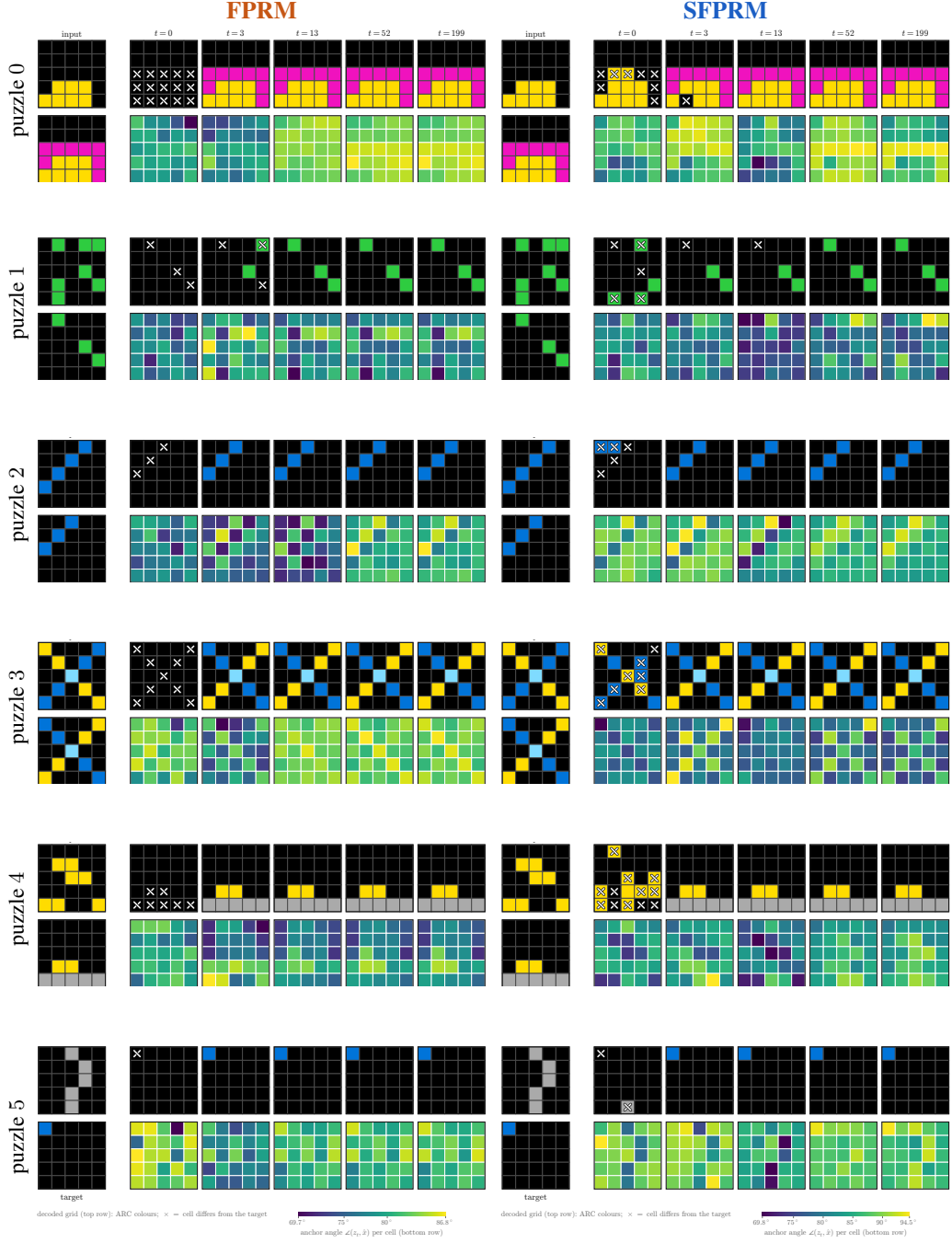

  \centering
  \begin{tabular}{@{}c@{\hspace{3pt}}c@{\hspace{5pt}}c@{}}
    & \dfp{\textbf{FPRM}} & \dsl{\textbf{\methodname}} \\[1pt]
    \arctop{0}
    \arcmid{1}
    \arcmid{2}
    \arcmid{3}
    \arcmid{4}
    \arcbot{5}
  \end{tabular}
  \caption{\textbf{Filmstrip gallery on Mini-ARC.} One test task per row, rolled out open-loop for $200$ iterations. Left: FPRM; right: \methodname. Each row starts with the test input (top) and its target (bottom), followed by five iterations from $t=0$ to $t=199$. At each iteration the decoded grid is shown on top, with $\times$ marking cells that differ from the target, and each cell's angle to its input anchor, $\angle(z_t,\hat x)$, below.}
  \label{fig:filmstrips-arc}
\end{figure}

\begin{figure}[p]
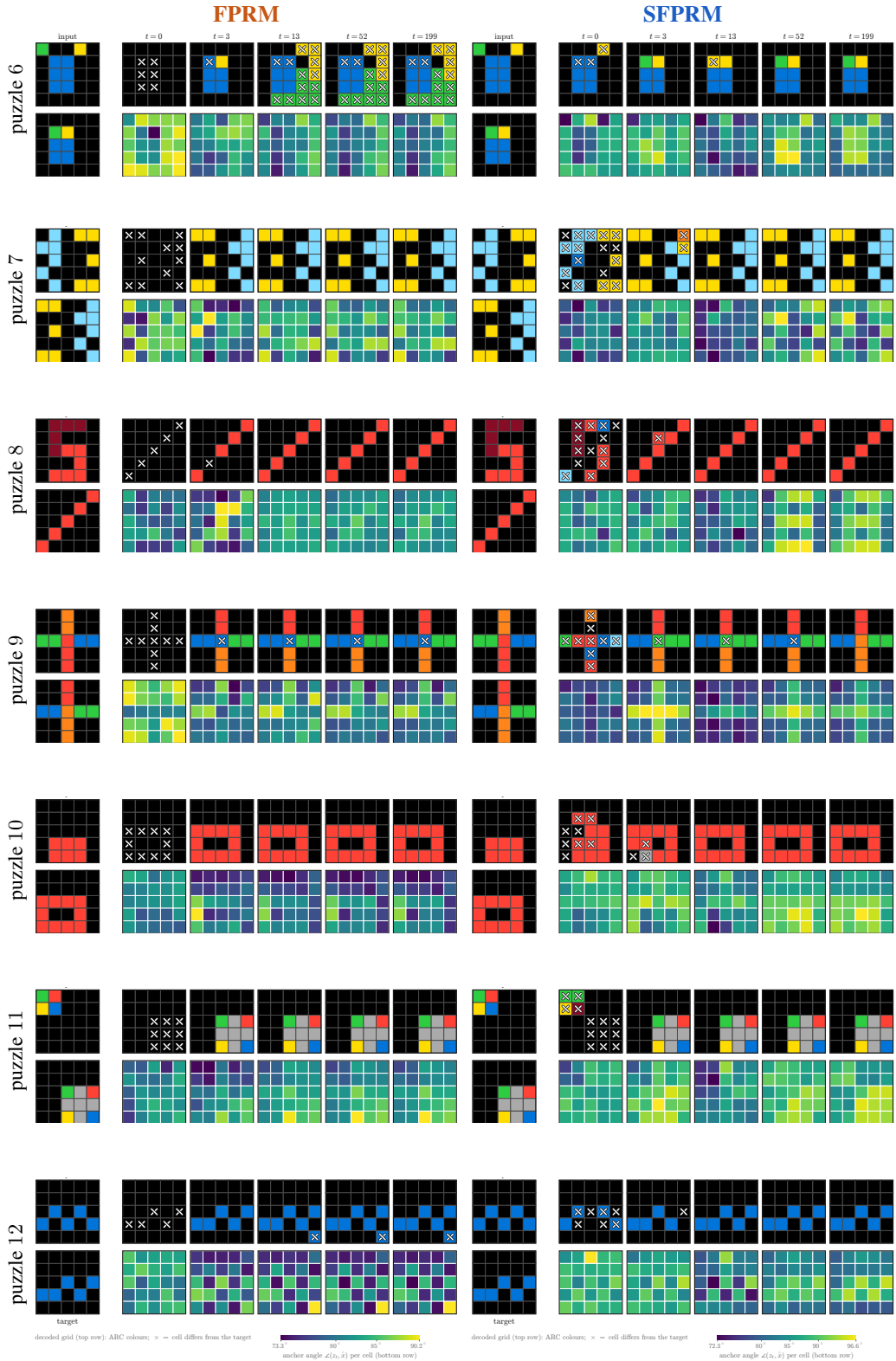

  \centering
  \begin{tabular}{@{}c@{\hspace{3pt}}c@{\hspace{5pt}}c@{}}
    & \dfp{\textbf{FPRM}} & \dsl{\textbf{\methodname}} \\[1pt]
    \arcetop{6}
    \arcemid{7}
    \arcemid{8}
    \arcemid{9}
    \arcemid{10}
    \arcemid{11}
    \arcebot{12}
  \end{tabular}
  \caption{\textbf{Filmstrip gallery on Mini-ARC (continued).} Same caption as in \Cref{fig:filmstrips-arc}.}
  \label{fig:filmstrips-arc2}
\end{figure}

\end{document}